\documentclass[letterpaper]{article} %
\usepackage{aaai25}  %
\usepackage{times}  %
\usepackage{helvet}  %
\usepackage{courier}  %
\usepackage[hyphens]{url}  %
\usepackage{graphicx} %
\usepackage{natbib}  %
\usepackage{caption} %
\usepackage{algorithm}

\usepackage{newfloat}
\usepackage{listings}
\DeclareCaptionStyle{ruled}{labelfont=normalfont,labelsep=colon,strut=off} %
\floatstyle{ruled}
\newfloat{listing}{tb}{lst}{}
\floatname{listing}{Listing}
\usepackage{nicefrac}       %
\usepackage{microtype}      %
\usepackage[dvipsnames]{xcolor}
\usepackage{subcaption}
\usepackage{amsmath,amssymb,amsthm,amsfonts}
\usepackage{enumitem}
\usepackage{causality}
\usepackage{algorithmicx}
\usepackage{algpseudocode}
\usepackage{booktabs}       %
\usepackage{multirow}
\usepackage{pifont}
\usepackage{comment}

\newtheorem{theorem}{Theorem}
\newtheorem{corollary}{Corollary}
\newtheorem{lemma}{Lemma}
\theoremstyle{definition}
\newtheorem{prop}{Proposition}
\newtheorem{definition}{Definition}
\newtheorem{example}{Example}

\usepackage{apptools}
\AtAppendix{\counterwithin{adxprop}{subsection}}
\AtAppendix{\counterwithin{adxtheorem}{subsection}}
\AtAppendix{\counterwithin{algorithm}{subsection}}
\AtAppendix{\counterwithin{adxexample}{subsection}}
\AtAppendix{\counterwithin{adxcorollary}{subsection}}
\AtAppendix{\counterwithin{adxlemma}{subsection}}
\AtAppendix{\counterwithin{adxdefinition}{subsection}}
\AtAppendix{\counterwithin{figure}{subsection}}

\AtAppendix{\counterwithin{adxprop}{subsection}}\AtAppendix{\counterwithin{table}{subsection}}

\usepackage{ifthen}

\usepackage{tikz}
\usetikzlibrary{shapes,decorations,arrows,calc,arrows.meta,fit,positioning}
\tikzstyle{mybox} = [draw=gray, fill=gray!20, very thick,
rectangle, rounded corners, inner xsep=3pt, inner ysep=7pt]

\newcommand{\G}{\mathcal{G}}

\newcommand{\C}{\mathcal{C}}

\newcommand{\cmark}{\ding{51}}%
\newcommand{\xmark}{\ding{55}}%
\newcommand{\chkmark}{\color{Green}{\cmark}}
\newcommand{\crsmark}{\color{Red}{\xmark}}
\newcommand{\trimark}{\color{Black}{\ding{115}}}
\newcommand{\trimarkblue}{\color{ProcessBlue}{\ding{115}}}

\algdef{SE}[SUBALG]{Indent}{EndIndent}{}{\algorithmicend\ }%
\algtext*{Indent}
\algtext*{EndIndent}

\title{Testing Causal Models with Hidden Variables in Polynomial Delay \protect\\via Conditional Independencies}

\title{Testing Causal Models with Hidden Variables in Polynomial Delay \protect\\via Conditional Independencies}
\author {
    Hyunchai Jeong\equalcontrib\textsuperscript{\rm 1},
    Adiba Ejaz\equalcontrib\textsuperscript{\rm 2},
    Jin Tian\textsuperscript{\rm 3},
    Elias Bareinboim\textsuperscript{\rm 2}
}
\affiliations {
    \textsuperscript{\rm 1}Purdue University\\
    \textsuperscript{\rm 2}Columbia University\\
    \textsuperscript{\rm 3}Mohamed bin Zayed University of Artificial Intelligence\\
    jeong3@purdue.edu,
    adiba.ejaz@cs.columbia.edu,
    jin.tian@mbzuai.ac.ae,
    eb@cs.columbia.edu
}

\begin{document}

\maketitle

\begin{abstract}
    Testing a hypothesized causal model against observational data is a key prerequisite for many causal inference tasks. A natural approach is to test whether the conditional independence relations (CIs) assumed in the model hold in the data. While a model can assume exponentially many CIs (with respect to the number of variables), testing all of them is both impractical and unnecessary. Causal graphs, which encode these CIs in polynomial space, give rise to local Markov properties that enable model testing with a significantly smaller subset of CIs.
    Model testing based on local properties requires an algorithm to list the relevant CIs. However, existing algorithms for realistic settings with hidden variables and non-parametric distributions can take exponential time to produce even a single CI constraint.
    In this paper, we introduce the c-component local Markov property (C-LMP) for causal graphs with hidden variables. 
    Since C-LMP can still invoke an exponential number of CIs, we develop a \emph{polynomial delay} algorithm to list these CIs in poly-time intervals. To our knowledge, this is the first algorithm that enables poly-delay testing of CIs in causal graphs with hidden variables against arbitrary data distributions. Experiments on real-world and synthetic data demonstrate the practicality of our algorithm.
\end{abstract}

\begin{links}
    \link{Code}{https://github.com/CausalAILab/ListConditionalIndependencies}
\end{links}

\section{Introduction}
\label{section:intro}

Causal models are the daily bread of many fields of research \citep{pearl:2k,Spirtes2001}, but tools for testing them are lacking.
In various studies, researchers posit a causal model and use it to compute causal effects from data \cite{tennant:etal2021,hoover1990logic,king2004functional,sverchkov2017,robins2000marginal,rotmensch2017learning}.
The model imposes testable constraints on the statistics of the data collected.
Before using the model for causal inference, it's crucial to test if these constraints are met, and adjust the model as needed \citep{pearl:95a,pearl:2k,bareinboim:pea16, malinsky:2024, ankan:2022}.

Causal directed acyclic graphs (DAGs) are one popular model for causal assumptions \citep{pearl:2k,Spirtes2001}.
Conditional independencies (CIs) are the most basic constraint that a causal DAG imposes on observational data.
The study of CIs in the context of graphical models dates back to at least the 1980's \citep{pearl:88a,dawid:79,spirtes:etal98,pearl:98a,pearl:mes99,pearl:2k}.
A classic problem in this line of research is: \emph{given observational data and a hypothesized causal graph, do all the CIs implied by this graph hold in the data?}
If the answer is no, the DAG may be revised.

A key idea in the early literature of graphical models was to use a DAG to represent the constraints of probability distributions.
A multivariate probability distribution may encode exponentially many CIs with respect to the number of variables.
A DAG can encode these CIs in polynomial space.
The \emph{d-separation} criterion allows us to derive the CIs encoded in a DAG \cite{pearl:88a}.
The \emph{global Markov property} of a DAG is the set of all CIs encoded in it \cite{pearl:88a}.
There is also a well-known \emph{local Markov property} for DAGs \cite{pearl:88a,lauritzen:etal90}, which states that each variable must be conditionally independent of its non-descendants given its parents.
Since the CI relation is a semi-graphoid, the linearly many CIs of the local Markov property together imply the exponentially many CIs of the global Markov property.
This means that to test a DAG against observational data, it suffices to perform a linear number of CI tests as given by the local Markov property.
For concreteness, consider the DAG \(\G^1\) in Fig.~\ref{fig:intro_local_nonmarkov} and assume all variables \(\{A,B,\dots,H,U_1,U_2, U_3\}\) are observed.
Though \(\G^1\) encodes 35787 CIs, only 11 need testing by the local Markov property. For example, if we test that \(F \indep \{A,B\} \mid \{C\}\), we do not need to test that \(F \indep \{A\} \mid \{B,C\}\), since the former implies the latter by the weak union axiom.

Unobserved confounding is a widespread phenomenon in real-world settings \cite{fisher:36}. 
It occurs when a hidden variable causally affects two or more observed variables.
The local Markov property can be used to test \emph{Markovian} causal DAGs, which represent models without unobserved confounding.
However, it cannot be used to test \emph{non-Markovian} DAGs, which represent models with unobserved confounding.
This is because if the parents of a variable are partially unobserved, we cannot test CIs that require conditioning on these parents (Fig.~\ref{fig:intro_local_nonmarkov}).
Since the assumption of no unobserved confounding rarely holds in practice, alternative ways to test non-Markovian DAGs have been developed \citep{tian:pea02testable-implications,kang2009markov,geiger:98,geiger:99, richardson2003markov, richardson:2009, huandevans:2023}.
Despite their power, these works either (a) make strong assumptions on the DAG or probability distribution, or (b) do not provide an algorithm to query their required CI tests in poly-time intervals, with naive algorithms taking exponential time to output a single CI constraint.

\vspace{+0.05in}
\noindent  \textbf{Summary of contributions.} 
We give the first efficient algorithm for testing causal DAGs with hidden variables via conditional independencies.
This enables researchers to test their causal assumptions using observational data prior to inference.
Importantly, our approach extends to arbitrary data distributions and networks of unobserved confounding.

This result builds on a newer, fine-grained characterization of CIs in graphs based on a new construct called ancestral c-components (i.e., connected components in the bidirected skeleton).
In particular, we show that \(O(n2^s)\) CI tests (Prop.~\ref{prop:lmpsize}) are required to test a DAG on \(n\) variables whose largest c-component  has size \(s\).
This is an exponential improvement over naively testing all \(\Theta(4^n)\) CI constraints encoded in the DAG.
The upshot is largest for DAGs with many variables but small c-components.
For instance, the DAG \(\G^2\) in Fig.~\ref{fig:intro_local_semimarkov} implies 753 CIs, but only 5 really need testing.
More specifically, our contributions are as follows: 
\begin{enumerate}
    \item We introduce the c-component local Markov property, or C-LMP (Def.~\ref{def:lmpplus}).
    We show that C-LMP and the global Markov property are equivalent, admitting the same set of probability distributions for a given DAG. We then show an important property of C-LMP: a one-to-one mapping between the CI constraints it invokes and \emph{ancestral c-components} (Thm.~\ref{thm:equivalence:clmpplus}).
    \item Building on this characterization, we develop the first algorithm (\textsc{ListCI}) capable of listing all testable CI constraints of C-LMP in \emph{polynomial delay} (Thm.~\ref{thm:listci}).
    On a DAG with \(n\) nodes and \(m\) edges, \textsc{ListCI} takes \(O(n^2(n+m))\) time  to return each new CI constraint, if one exists, or exit when it has exhausted all CI constraints.
\end{enumerate}
Experiments with synthetic data and a real-world protein signaling dataset \cite{sachs2005causal} corroborate the theoretical findings. For the sake of space, proofs are provided in Appendix C. 

\section{Preliminaries} \label{sec:prelim}

\begin{figure}[t]
    \centering
    \begin{subfigure}{0.35\textwidth}
        \centering
        \includegraphics[width=\textwidth]{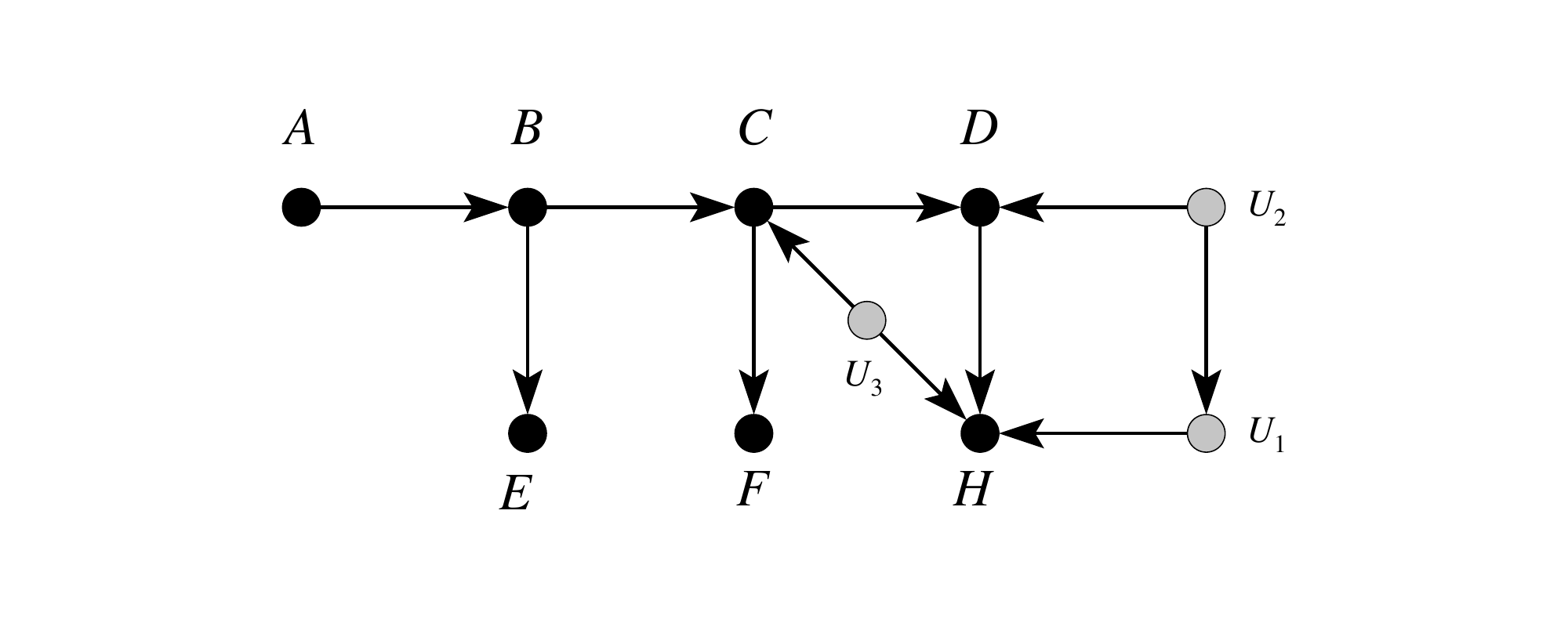}
        \caption{\(\G^1\)}
        \label{fig:intro_local_nonmarkov}
    \end{subfigure}
    \hfill
    \begin{subfigure}{0.29\textwidth}
        \centering
        \includegraphics[width=\textwidth]{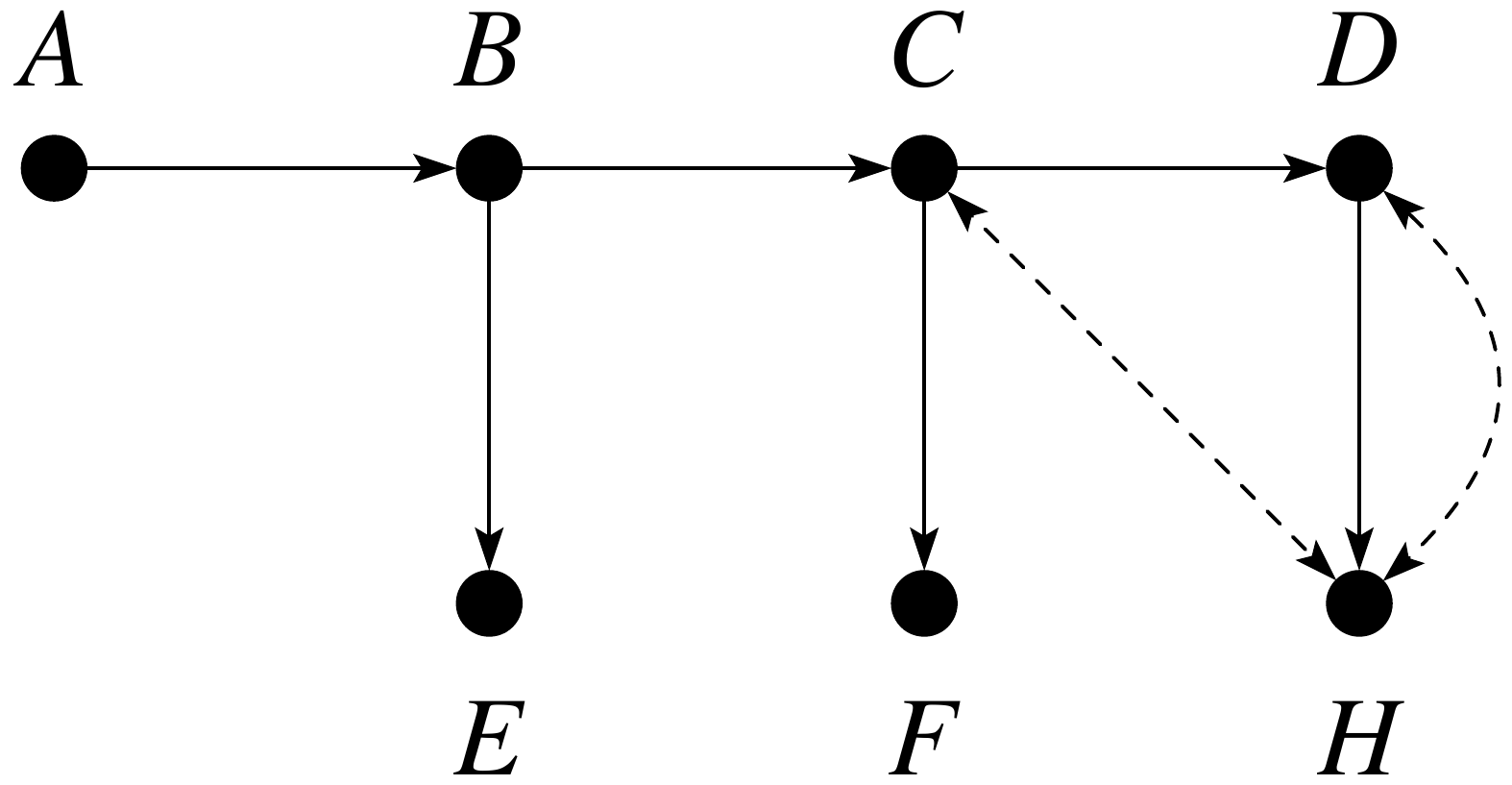}
        \caption{\(\G^2\)}
        \label{fig:intro_local_semimarkov}
    \end{subfigure}
\caption{
(a) A causal DAG \(\G^1\) in which the local Markov property implies the CI: \(H \indep \{A,B,C,E,F\} \mid \{D, U_1, U_3\}\).
If \(U_1\) and \(U_3\) are unobserved, we cannot test this CI.
(b) We project \(\G^1\) onto its observed variables to get \ \(\G^2\). In \(\G^2\), the c-component local Markov property invokes the testable CI: \(H \indep \{A,E,F\} \mid \{B,C,D\}\).
}
\label{fig:intro_local}
\end{figure}

\textbf{Notation.} We use capital letters to denote variables $(X)$, small letters for their values $(x)$, and bold letters for sets of variables $(\*{X})$ and their values $(\*{x})$. The probability distribution over a set of variables $\*{X}$ is denoted by $P(\*{X})$. We consistently use $P(\*x)$ as abbreviations for probabilities $P(\*X = \*x)$.  For disjoint sets of variables \(\*X,\*Y, \*Z\), we use \(\*X \indep \*Y \mid \*Z\) to denote that \(\*X\) and \(\*Y\) are conditionally independent given \(\*Z\).

\textbf{Structural causal models.} The basic framework of our analysis rests on \emph{structural causal models} (SCMs) \citep[Def.~7.1.1]{pearl:2k}. An SCM $\&M$ is a quadruple $\&M = \langle \*V, \*U, \&F, P(\*u)\rangle$ where $\*V$ and $\*U$ are sets of endogeneous and exogeneous variables, respectively. $\mathcal{F}$ is a set of functions: each \(V \in \*V\) is a function \(f_V(\*{PA_V, U_V})\) of its endogeneous and exogeneous parents, $\*{PA_V} \subseteq \*V$ and $\*{U_V} \subseteq \*U$ respectively.
\(P(\*u)\) is a joint distribution over \(\*U\).
Each SCM $\&M$ induces an observed distribution \(P(\*v)\) over $\*V$.
An SCM is said to be \emph{Markovian} if \(\*{U}_V, \*{U}_W\) are independent for every distinct \(V,W \in \*{V}\), and \emph{non-Markovian} otherwise.
For a more detailed survey on SCMs, we refer to \cite{pearl:2k,bareinboim:etal20}.

\textbf{Causal graphs.}
The causal graph \(\G\) for an SCM  $\&M = \langle \*V, \*U, \&F, P(\*u)\rangle$ is constructed as follows: (1) add a vertex for every \(V \in \*V\) (2) add an edge \(V_i \to V_j\) for every \(V_i, V_j \in \*V\) if \(V_i \in \*{PA_{V_j}}\) (3) add a dashed bidirected edge between \(V_i, V_j\) if \(\*U_i, \*U_j\) are correlated or \(\*U_i \cap \*U_j \neq \emptyset\). \(\G\) is said to be Markovian if it contains only directed edges, and semi-Markovian otherwise.

We denote the sets of parents, ancestors, and descendants of \(\*X\) (including \(\*X\) itself) in \(\G\) as \(\Pa{\*X}, \An{\*X}\), and \(\De{\*X}\), respectively.
The set of non-descendants of \(\*X\) in \(\G\) is denoted \(\Nd{\*X} = \*{V}\setminus \De{\*X}\), which does not include \(\*X\) itself.
The set of spouses of \(\*X\) in \(\G\) is \(\Spo{\*X} = \bigcup_{X \in \*X} \{Y \mid Y \leftrightarrow X\}\).
\(\*X\) is said to be an \textit{ancestral set} if it contains its own ancestors, i.e., \(\*X = \An{\*X}\).
We use \(\G_{\*X}\) to denote the induced subgraph of \(\G\) on \(\*X \subseteq \*V\). A subscript \(\G'\), e.g., \(\An{\*X}_{\G'}\) indicates that the set is computed from the subgraph \(\G'\).
We omit the subscript when clear from context.
An ordering \(\*V^{\prec}\) on variables \(\*V\) is said to be consistent with \(\G\) (i.e., a topological ordering) if for any \(X,Y \in \*V\), \(X \prec Y\) implies \(Y \notin \An{X}_{\G}\). Let \(\*V^{\leq X} = \{Y \mid Y \prec X \text{ or } Y = X\}\).

\emph{Semi-Markovianity vs Non-Markovianity.} A non-Markovian causal DAG \(\G\) can be constructed for a non-Markovian SCM by making the exogenous variables \(\*U\) explicit. A non-Markovian DAG with arbitrary hidden variables can be `projected' onto a semi-Markovian causal DAG \(\G'\) which imposes exactly the same CI constraints over the observed variables \cite{tian:pea02testable-implications}. In \(\G'\), each unobserved variable is (i) a parent of at most two observed variables and (ii) made implicit by adding a dashed bidirected edge between its two children. The complexity of the latent structure is irrelevant to the CIs over observed variables. Therefore, we work with semi-Markovian graphs for model testing.

\textbf{$d$-separation.}
A node \(W\) on a path \(\pi\) is said to be a collider on \(\pi\) if \(W\) has converging arrows into \(W\) in \(\pi\), e.g., \(\rightarrow W \leftarrow\) or \(\leftrightarrow W \leftarrow\).
\(\pi\) is said to be blocked by a set \(\*Z\) if there exists a node \(W\) on \(\pi\) satisfying one of the following two conditions: 1) \(W\) is a collider, and neither \(W\) nor any of its descendants are in \(\*Z\), or 2) \(W\) is not a collider, and \(W\) is in \(\*Z\) \citep{pearl:88a}.
Given disjoint sets \(\*X,\*Y\), and \(\*Z\) in \(\G\), \(\*Z\) is said to \textit{$d$-separate} \(\*X\) from \(\*Y\) in \(\G\) if and only if \(\*Z\) blocks every path from a node in \(\*X\) to a node in \(\*Y\) according to the $d$-separation criterion \citep{pearl:88a}.
If \(\*Z\) $d$-separates \(\*X\) from \(\*Y\) in \(\G\) (written \(\*X \perp_d \*Y \mid \*Z\)), then \(\*X\) is conditionally independent of \(\*Y\) given \(\*Z\) in any observational distribution consistent with \(\G\) \citep{pearl:88a, richardson2003markov}.

\begin{definition}{(C-component) \cite{tian:pea02-general-id}}
\label{def:ccomponent}
    A set of variables \(\*C \subseteq \*V\) in a causal graph \(\G\) is said to be a confounded component (c-component, for short) if there is a path of only bidirected edges connecting any \(V_i, V_j \in \*C\), and \(\*C\) is maximal.
\end{definition} 
For a variable \(X \in \*V\), \(\&{C}(X)_{\G}\) denotes the c-component containing \(X\) in \(\G\). 

Previously, we have referred to the set of all CIs encoded in a DAG. We define this formally.
\begin{definition}{(Global Markov Property (GMP)) \cite{pearl:88a,Geiger1989}}
\label{def:gmp}
    A probability distribution \(P(\*v)\) over a set of variables \(\*V\) is said to satisfy the global Markov property for a causal graph \(\G\) if, for arbitrary disjoint sets \(\*X,\*Y,\*Z \subset \*V\) with \(\*X, \*Y \neq \emptyset\),
    \[
    \*X \perp_d \*Y | \*Z \implies \*X \indep \*Y | \*Z \text{ in } P(\*v).
    \]
\end{definition} 

Various local Markov properties have been developed which identify a subset of the CIs invoked by GMP that imply all others. A prominent example is the local Markov property for Markovian DAGs.

\begin{definition}[The Local Markov Property (LMP) \cite{pearl:88a,lauritzen:etal90,lauritzen:96}\footnote{Note that this property is referred to as the \emph{directed local Markov property} in \cite{lauritzen:etal90}.}] \label{def:markovlmp}
    A probability distribution \(P(\*v)\) over a set of variables \(\*V\) is said to satisfy the local Markov property for a given Markovian DAG \(\G\) if, for any variable \(X \in \*V\),
    \begin{equation*}
    \label{eq:markovlmp}
        X \indep \Nd{\{X\}} \setminus \Pa{\{X\}} \mid \Pa{\{X\}} \setminus \{X\} \text{ in } P(\*v).
    \end{equation*}
\end{definition}

\begin{example}
    Consider Fig.~\ref{fig:intro_local_semimarkov}.
    \(\{C,D,H\}\) is a c-component, and \(\&{C}(H)_{\G^2} = \{C,D,H\}\).
    Since \(\{B,C,D\}\) $d$-separates \(H\) from \(\{A,E,F\}\) in \(\G^2\), \(\G^2\) implies the CI: \(H \indep \{A,E,F\} \mid \{B,C,D\}\).
\end{example}

\section{The C-component Local Markov Property}
\label{sec:reformulation}

In this section, we motivate and introduce the c-component local Markov property for causal DAGs with unobserved confounders. 
In Sec.~\ref{sec:naivemarkov}, we demonstrate the limitations of the traditional local Markov property (LMP) when applied to non-Markovian DAGs.
In Sec.~\ref{sec:clmp:def}, to solve this problem, we present the c-component local Markov property (C-LMP) for semi-Markovian DAGs and establish its equivalence with GMP.
In Sec.~\ref{sec:clmp:properties}, we provide a useful property of C-LMP that makes its CIs amenable to listing.

\subsection{A Naive Approach to Testing Non-Markovian Graphs}\label{sec:naivemarkov}

First, we show the limitations of the well-known LMP (Def.~\ref{def:markovlmp}) in testing non-Markovian DAGs. 
For each variable \(X\) in a given graph, LMP states that \(X\) is independent of its non-descendants conditioning on its parents. Intuitively, the parents of \(X\) form a minimal set separating \(X\) from its non-descendants.

\begin{example}
\label{example:markovlmp}
    Consider Fig.~\ref{fig:intro_local_nonmarkov}. The DAG \(\G^1\) contains only directed edges; assuming all variables are observed, \(\G^1\) is Markovian. LMP invokes 11 CIs for \(\G^1\): \(A \indep \{U_1,U_2, U_3\}\), \(B \indep \{U_1,U_2,U_3\} \mid \{A\}\), \(C \indep \{A,E,U_1,U_2\} \mid \{B,U_3\}\), \(D \indep \{A,B,E,F,U_1,U_3\} \mid \{C,U_2\}\),  \(E \indep \{A,C,D,F,H,U_1,U_2,U_3\} \mid \{B\}\),  \(F \indep \{A,B,E,D,H,U_1,U_2,U_3\} \mid \{C\}\), \(H \indep \{A,B,C,E,F, U_2\} \mid \{D, U_1,U_3\}\), \(U_1 \indep \{A,B,C,D,E,F,U_3\} \mid \{U_2\}\), \(U_2 \indep \{A,B,C,E,F,U_3\}\), \(U_3 \indep \{A,B,E,U_1,U_2\}\).  
    All 11 CIs are testable using samples from the distribution \(P(a,b,c,d,e,f,h,u_1,u_2, u_3)\).
    
\qed
\end{example}

LMP fails trivially for semi-Markovian DAGs since, for example, a variable may be connected to a non-descendant by a bidirected edge. One could think to instead apply LMP to the `unprojected' non-Markovian DAG underlying the given semi-Markovian DAG. The non-Markovian DAG would contain no bidireced edges since the unobserved parents are made explicit. However, LMP does not extend to non-Markovian DAGs either, as we show in the following example.

\begin{example}
\label{example:unproj}
    Continuing Ex.~\ref{example:markovlmp}.
    Assume we are given the non-Markovian DAG \(\G^1\) shown in Fig.~\ref{fig:intro_local_nonmarkov}. 
    If \(U_1, U_2\) and \(U_3\) are unobserved, only samples from \(P(\*v) = \int_{u_1,u_2,u_3}P(a,b,c,d,e,f,h,u_1,u_2,u_3) \: du_1 du_2 du_3\) are available, where \(\*V = \{A,B,C,D,E,F,H\}\) denotes the observed variables.
    All 11 CIs invoked by LMP for \(\G^1\), listed in Ex.~\ref{example:markovlmp}, require samples from $P(a,b,c,d,e,f,h,u_1,u_2,u_3)$.
    Hence, none of these CIs can be tested using \(P(\*v)\).
    
    One approach to try salvaging these 11 CIs is to consider only those CIs in which  \(\{U_1,U_2,U_3\}\) appear before the conditioning bar.
    In such CIs, \(\{U_1,U_2, U_3\}\) can be removed using the decomposition axiom. 
    However, only two of the 11 CIs can be modified in this way, i.e.,
    \begin{align}
        E \indep \{A,C,D,F,H\} \mid \{B\},    \label{eq:ci1}  \\ 
        F \indep \{A,B,E,D,H\} \mid \{C\}.    \label{eq:ci2}
    \end{align}
    These two CIs do not suffice to derive the GMP for \(\G^1\).
    To witness, consider a graph \(\G'\) over the same variables as \(\G^1\) but with only one edge \(H \to A\). 
    Say we have an observational distribution \(P(\*v)\) faithfully induced by \(\G'\).
    Then, the CIs in Eqs.~(\ref{eq:ci1},\ref{eq:ci2}) both hold in \(P(\*v)\).
    However, \(\G^1\) implies that 
    \begin{align}
      H \indep \{A,E,F\} \mid \{B,C,D\}    
    \end{align}
    which does not hold in \(P(\*v)\) since \(\G'\) contains an edge \(H \to A\).
    Only testing the two CIs in Eqs.~(\ref{eq:ci1},\ref{eq:ci2}) would lead to the false conclusion that \(P(\*v)\) is consistent with \(\G^1\).
    As a result, it is insufficient to use only those CIs which invoke  \(\{U_1, U_2, U_3\}\) outside the conditioning set.

\qed
\end{example}

As in the example, to test a non-Markovian DAG, one can not simply `filter out' CIs that require conditioning on unobserved variables.
This is because such CIs can entail testable CIs over the observed variables.
The remaining option is to derive all these entailed CIs using the semi-graphoid axioms, and test those which invoke only observed variables.
This is the GMP (Def.~\ref{def:gmp}) of the non-Markovian DAG, which can invoke \(\Theta(4^n)\) CIs for a DAG with \(n\) observed variables (Prop. C.3.1).
This approach fails to exploit any locality in the graph, and requires a prohibitive number of CI tests, many of which are redundant.
This suggests the need for alternative compatibility properties for semi-Markovian (equivalently, non-Markovian) DAGs. We next introduce our contribution, the c-component local Markov property.

\subsection{C-LMP: A Local Markov Property for Semi-Markovian DAGs}\label{sec:clmp:def}

In a semi-Markovian graph, the observed parents of a variable do not suffice to separate it from its non-descendants.
Therefore, a surrogate of the parents is needed to restore locality.
The construct of a c-component (Def.~\ref{def:ccomponent}) was introduced for this purpose \cite{bareinboim:etal20}, which we explain via an example.

\begin{figure*}[t]
    \centering
    \begin{subfigure}{0.25\textwidth}
        \includegraphics[width=\textwidth]{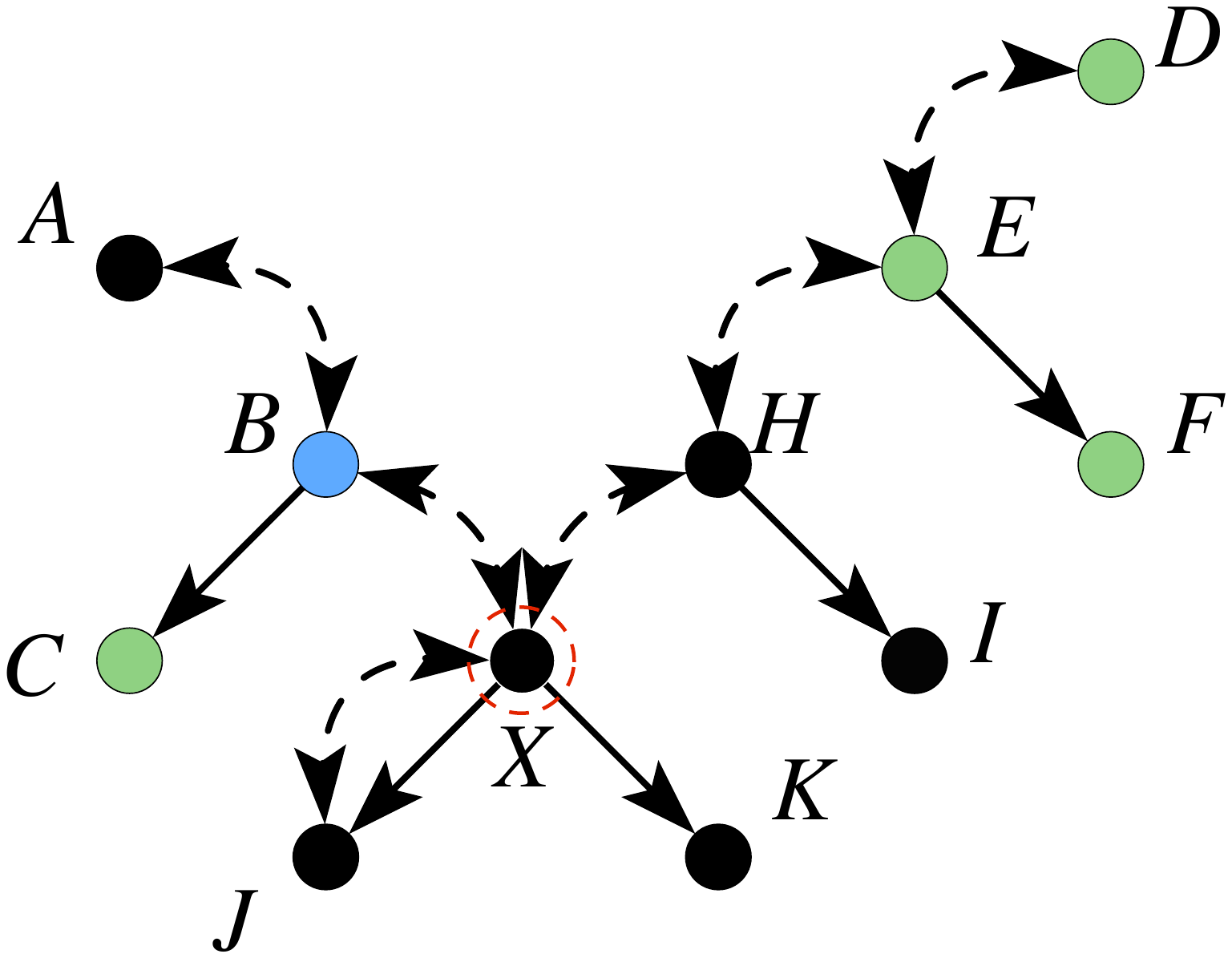}
        \caption{\(X\) is separated from \(C\) but not \(A\) when conditioning on \(B\).}
        \label{fig:clmp_mb1}
    \end{subfigure}
    \hfill
    \begin{subfigure}{0.25\textwidth}
        \includegraphics[width=\textwidth]{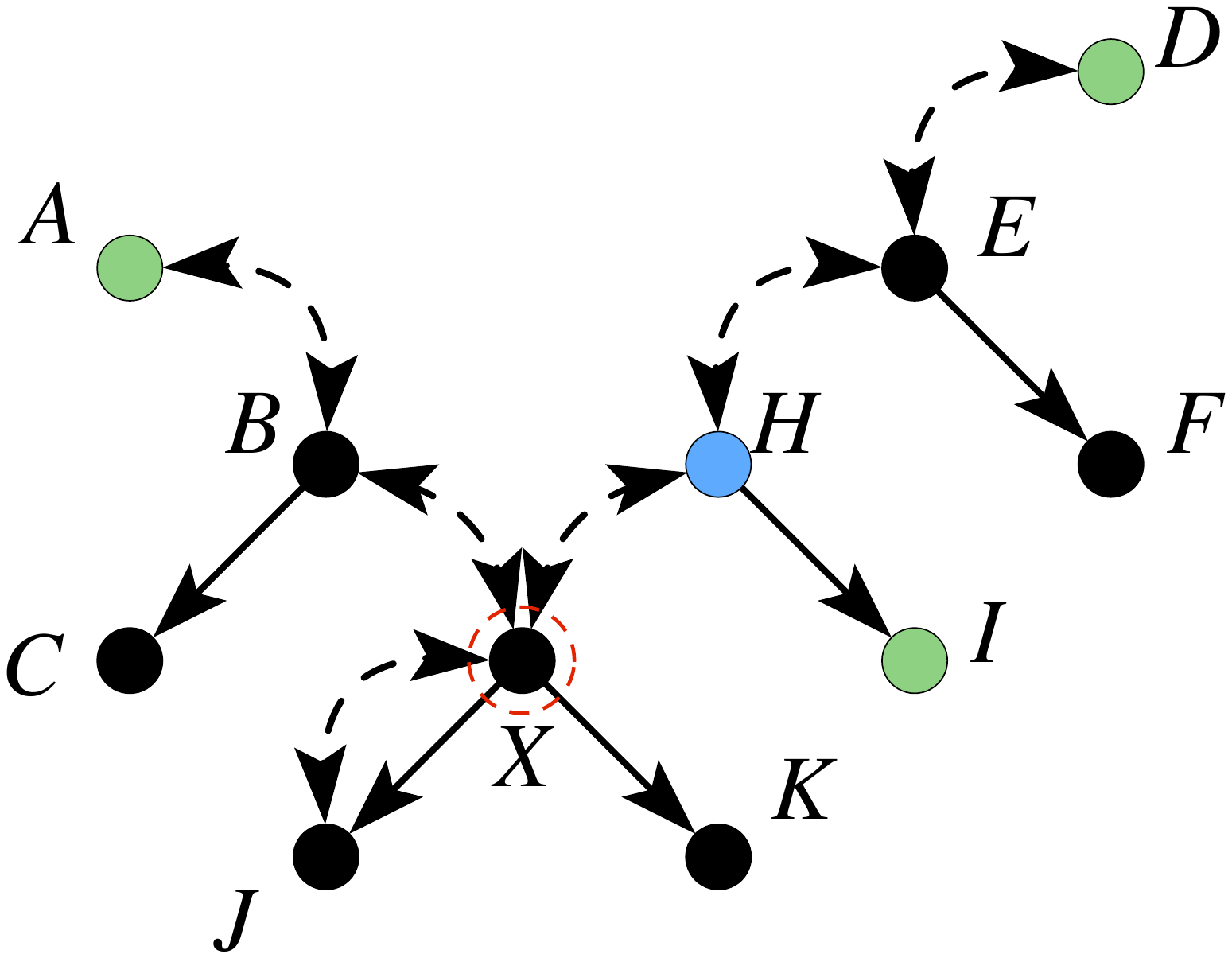}
        \caption{\(X\) is separated from \(A\) but not \(C\) when not conditioning on \(B\).}
        \label{fig:clmp_mb2}
    \end{subfigure}
    \hfill
    \begin{subfigure}{0.25\textwidth}
        \includegraphics[width=\textwidth]{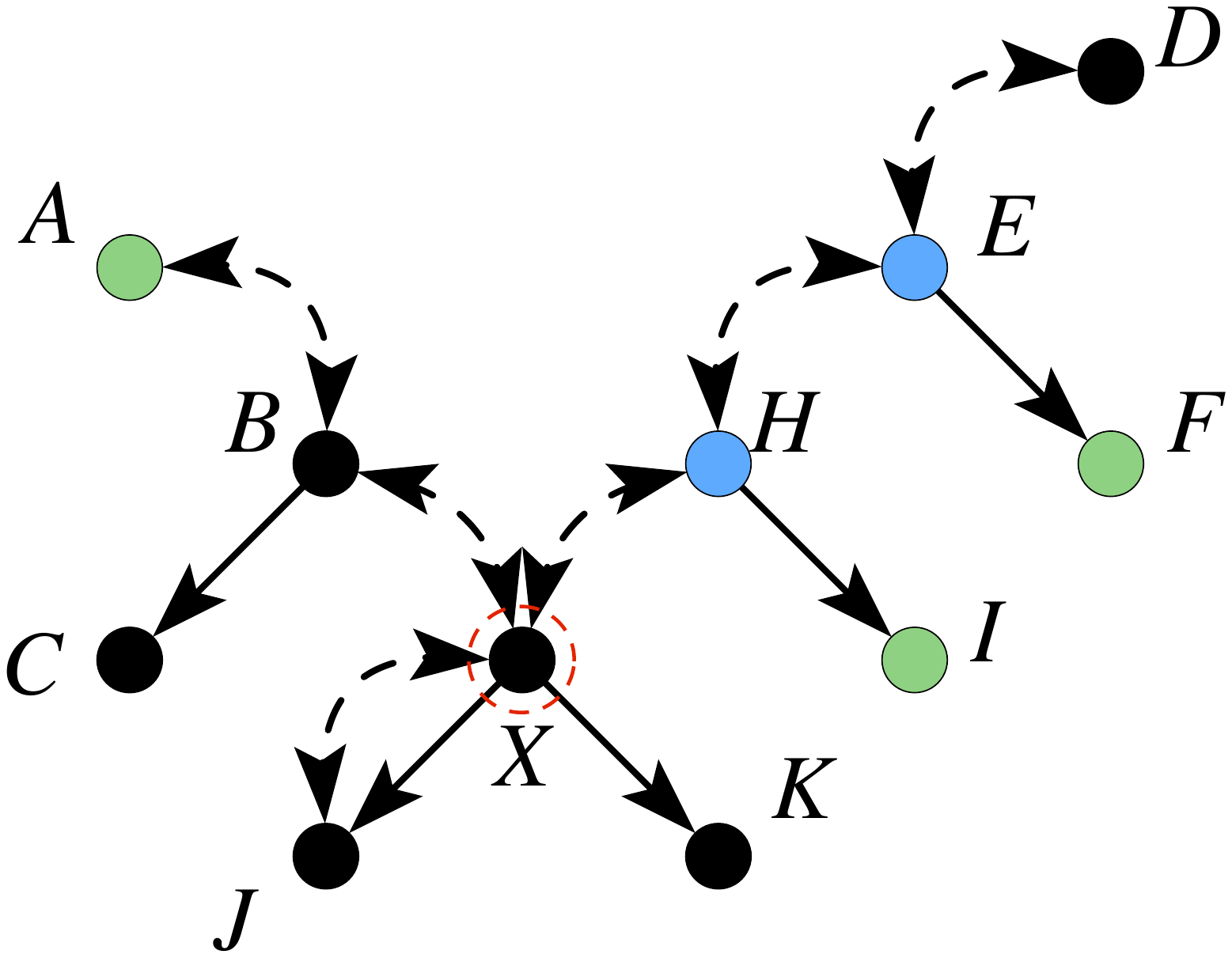}
        \caption{\(X\) is separated from \(F,I\) but not \(D\) when conditioning on \(\{H,E\}.\)}
        \label{fig:clmp_mb3}
    \end{subfigure}
\caption{Three ACs relative to the variable \(X\) in the (same) causal DAG \(\G\).  Assume an ordering \(A \prec B \prec \dots \prec X \prec J \prec K\). 
The ACs relative to \(X\) (excluding \(\{X\}\) itself), shown in blue, separate it from the variables shown in green.
}
\label{fig:clmp_mb}
\end{figure*}

\begin{example} \label{ex:locality}
    Continuing Ex.~\ref{example:markovlmp}, assume \(\{U_1, U_2, U_3\}\) are unobserved in \(\G^1\) (Fig.~\ref{fig:intro_local_nonmarkov}). 
    The second graph \(\G^2\) (Fig.~\ref{fig:intro_local_semimarkov}) is the semi-Markovian projection of \(\G_1\).
    Note that the conditional independence \(H \indep \{A,B,C,E, F\}\mid\{D, U_1, U_3\}\) cannot be tested from the data since \(\{U_1, U_3\}\) are not observed. 
    This means that a different conditioning set is needed to make \(H\) independent of its observed non-descendants. 
    
    One might condition on the observed descendants of \(\{U_1, U_3\}\) that are closest to \(U_1,U_3\), i.e., \(\{C,D\}\).
    These variables are not separable from \(H\) without conditioning on \(U_1, U_2\) or \(U_3\), which is not an option.
    \(\{C,D\}\) have bidirected edges to \(H\) in \(\G^2\), the semi-Markovian projection of \(\G^1\).
    \(\{C,D\}\) are now active on any paths on which \(\{C,D\}\) they are colliders: for instance, on the paths \(E \leftarrow B \to C \leftarrow U_3 \to H\) and \(A \to B \to C \leftarrow U_3 \to H\).
    To block some of these paths, we also condition on the (remaining) observed parents of \(\{C,D\}\), i.e., \(\{B\}\).
    Firstly, conditioning on \(\{C,D\}\) already makes \(B\) and its ancestors active on any paths where they are colliders; secondly, \(B\) is connected to \(H\) when conditioning on \(\{C,D\}\).
    Therefore, conditioning on \(\{B\}\) does not introduce any new active paths to \(X\). 
    Conditioning on \(\{B\}\) additionally blocks paths to \(H\) containing \(B\) on which \(B\) is not a collider.
    Therefore, we have the conditioning set \(\Pa{\*C} \setminus \{H\} = \{B,C,D\}\).
    The CI over observables \(H \indep  \{A,E,F\} \mid \{B,C,D\}\) is thus derived. 
\qed
\end{example}

Ex.~\ref{ex:locality} is relatively simple since the c-component of \(H\) is used to generate the given CI. 
However, the c-components of a variable do not always give rise to CIs.

\begin{example}
\label{ex:onecnoci}
    Consider, as an example DAG, a bidirected path of the form \(V_1 \leftrightarrow V_2 \dots \leftrightarrow V_n\) on variables \(\*V\).
    For each \(V_i\), the c-component including \(V_i\) is the entire graph.
    Therefore, conditioning on the c-component results in the `vacuous' CI: \(V_i \indep \emptyset \mid \*V \setminus V_{i}\).
    Clearly, from this set of vacuous CIs, we cannot derive non-vacuous CIs encoded the graph, such as those of the form \(V_i \indep \{V_j\}\), \(\forall i, j\text{ s.t. }|i-j|>1\) (e.g., \(V_1 \indep \{V_3\}\)). 
\qed
\end{example}

A useful insight due to \cite{richardson2003markov} is that subsets of a variable's c-component can give rise to distinct `surrogates' for its parents and hence distinct CIs. This is because conditioning on a certain variable in a c-component closes some paths while opening others. We generalize c-components to \emph{ancestral c-components} to define these `surrogates.'\footnote{Ancestral c-components can be shown to be equivalent to \emph{ancestrally closed districts} as defined in \cite{richardson:2009}. We thank Robin Evans for bringing this to our attention.}

\begin{definition}{(Ancestral C-component (AC))}
\label{def:ancestralccomponent}
    Given a causal graph \(\G\) and a consistent ordering \(\*V^\prec\), let \(X\) be a variable in \(\*V^\prec\).
    A set of variables \(\*C\) is said to be an ancestral c-component relative to \(X\) if there exists an ancestral set \(\*S \subseteq \*V^{\leq X}\) containing \(X\) such that \(\&{C}(X)_{\G_{\*S}} = \*C\).
    The collection of all such \(\*C\) is denoted:
    \[
        \&{AC}_X = \{ \*C \mid \*C \text{ is an ancestral c-component relative to } X\}.
    \]
\end{definition}

Unlike c-components, there may be many ancestral c-components with respect to a given variable.

\begin{example}
\label{ex:ac}
    Consider the graph \(\G\) in Fig.~\ref{fig:clmp_mb} with ordering \(A \prec B \prec \dots \prec X \prec J \prec K\).
    For the variable \(X\),  \(\{X\}\) is an AC relative to \(X\) induced by the ancestral set \(\*S = \{X\}\); \(\{B,X\}\) is an AC relative to \(X\) induced by the ancestral set \(\*S = \{B,C,D,E,X\}\). \(\{X,A,D,E\}\) is not an AC relative to \(X\) since the exclusion of \(B\) and/or \(H\) disconnects the variables in question.
    For the variable \(J\), \(\{J\}\) is not an AC relative to \(J\) since it excludes the ancestor \(X\) to which \(J\) is connected by a bidirected edge; \(\{X,J\}\) is an AC induced by the ancestral set \(\{X,J\}\). 
\qed
\end{example}

We use ACs to define the c-component local Markov property, which generalizes LMP to semi-Markovian DAGs using this new notion of local independence.

\begin{definition}{(The C-component Local Markov Property (C-LMP))}
\label{def:lmpplus}
    A probability distribution \(P(\*v)\) over a set of variables \(\*V\) is said to satisfy the c-component local Markov property for a causal graph \(\G\) with respect to the consistent ordering \(\*V^\prec\), if, for any variable \(X \in \*V^\prec\) and ancestral c-component \(\*C \in \&{AC}_X\) relative to \(X\),
    \begin{equation*}
        \begin{split}
            X \indep \: &\*S^+ \setminus \Pa{\*C} \mid (\Pa{\*C} \setminus \{X\}) \text{ in } P(\*v) \text{, where } \\
            & \*S^+ = \*V^{\leq X} \setminus \De{\Spo{\*C} \setminus \Pa{\*C}}.
        \end{split}
    \end{equation*}
\end{definition}

\begin{example} \label{example:3mbs}
Continuing Ex.~\ref{ex:ac}. We give a few examples of CIs invoked by C-LMP for the variable \(X\).
\begin{enumerate}
    \item The AC \(\*C = \{X,B\}\) gives the CI \(X \indep \{C,D,E,F\} \mid \{B\}\) (Fig.~\ref{fig:clmp_mb1}), since 
    \begin{align*}
    \Pa{\*C} = \Pa{\{X,B\}} = \{X,B\}
    \end{align*}
    \begin{align*}
    \*S^+ &= \*V^{\leq X} \setminus \De{\Spo{\{X,B\} \setminus \Pa{\{X,B\}}}} \\
    &= \{A,B,C,D,E,F,H,I,X\} \setminus \De{\{A,H\}} \\
    &= \{A,B,C,D,E,F,H,I,X\} \setminus \{A,H,I\} \\
    &=  \{B,C,D,E,F,X\}
    \end{align*}
    \item The AC \(\*C = \{X,H\}\) gives the CI \(X \indep \{A,D,I\} \mid \{H\}\) (Fig.~\ref{fig:clmp_mb2}), since 
    \begin{align*}
        \Pa{\*C} = \Pa{\{X,H\}} = \{X,H\}
    \end{align*} 
    \begin{align*}
        \*S^+ &= \*V^{\leq X} \setminus \De{\Spo{\{X,H\} \setminus \Pa{\{X,H\}}}} \\
        &= \{A,B,C,D,E,F,H,I,X\} \setminus \De{\{B,E\}} \\
        &= \{A,B,C,D,E,F,H,I,X\} \setminus \{B,C,E,F\} \\
        &=  \{A,D,H,I,X\}.
    \end{align*}
    \item The AC \(\*C = \{X,H,E\}\) gives the CI \(X \indep \{A,F,I\} \mid \{H,E\}\) (Fig.~\ref{fig:clmp_mb3}), since 
    \begin{align*}
        \Pa{\*C} = \Pa{\{X,H,E\}} = \{X,H,E\}
    \end{align*}
    \begin{align*}
        \*S^+ &= \*V^{\leq X} \setminus \De{\Spo{\{X,H,E\} \setminus \Pa{\{X,H,E\}}}} \\
        &= \{A,B,C,D,E,F,H,I,X\} \setminus \De{\{B,D\}} \\
        &= \{A,B,C,D,E,F,H,I,X\} \setminus \{B,C,D,E\} \\
        &=  \{A,E,F,H,I,X\}
    \end{align*}
\end{enumerate}
\qed
\end{example}

As a sanity check, let us examine the CIs C-LMP implies for a Markovian DAG \(\G\), where all c-components are singletons. There is exactly one AC \(\*C = \{X\}\) relative to a given variable \(X\). Moreover, \(\Pa{\*C} = \Pa{\{X\}}, \Spo{\{X\}} = \emptyset\) and \(\*S^+ = \*V^{\leq X} \setminus \De{\emptyset} =  \*V^{\leq X}\). Therefore, the CI invoked by C-LMP for \(X\) is
\begin{equation}
\label{eq:clmpcimarkov}
    X \indep \*V^{\leq X} \setminus \Pa{\{X\}} \mid \Pa{\{X\}} \setminus \{X\} 
\end{equation}

Thus, C-LMP reduces to the local well-numbering Markov property \cite{lauritzen:etal90} for a Markovian DAG \(\G\).\footnote{The local well-numbering Markov property was shown to be equivalent to LMP in \cite{lauritzen:etal90}. A subtle difference is that LMP tests the independence of \(X\) from its all non-descendants, not just \(\*V^{\leq X}\) for a given ordering \(\*V^{\prec}\).} %
In semi-Markovian DAGs, c-components are not necessarily singletons. Comparing the CIs invoked by LMP and C-LMP for a given variable \(X\), we see that C-LMP generalizes two concepts:
\begin{enumerate}
    \item The conditioning set \(\Pa{\{X\}} \setminus \{X\}\) stated in LMP is replaced with \(\Pa{\*C} \setminus \{X\}\) in C-LMP, using an AC \(\*C\) relative to \(X\).
    \item The conditioning set \(\Pa{\*C} \setminus \{X\}\) renders \(X\) independent of \(\*S^+ \setminus \Pa{\*C}\) where \(\*S^+ = \*V^{\leq X} \setminus \De{\Spo{\*C} \setminus \Pa{\*C}}\), as stated by C-LMP.
    The set \(\*S^+ \setminus \Pa{\*C}\) replaces the set \(\Nd{\{X\}} \setminus \Pa{\{X\}}\) in LMP.
\end{enumerate}

In Ex.~\ref{ex:locality}, we gave intuition for the generalised conditioning set \(\Pa{\*C} \setminus \{X\}\) (Case 1). Next, we explain the construction of \(\*S^+\) (Case 2) used to compute the maximal set of variables in \(\*V^{\leq X}\) that are independent of \(X\) given \(\Pa{\*C} \setminus \{X\}\).
Consider what happens to a variable \(Y \in \*V^{\leq X} \setminus \Pa{\*C}\) when conditioning on \(\Pa{\*C} \setminus \{X\}\).

\begin{itemize}
    \item If \(Y\) is a descendant (or an ancestor) of some node \(W \in \Pa{\*C}\), we have a directed path \(\pi\) from \(W\) to \(Y\) (or vice-versa). Conditioning on \(\Pa{\*C} \setminus \{X\}\) blocks \(\pi\) (since \(Y \not \in \Pa{\*C}\)), and hence any path from \(X\) to \(Y\) which contains \(\pi\) as a sub-path. For example, in Fig.~\ref{fig:clmp_mb1}, taking \(\*C = \{X,B\}\), \(W = B\) and \(Y = C\), conditioning on \(\{B\}\) blocks the path \(X \leftrightarrow B \to C\).
    \item If \(Y\) is connected by a bidirected path to some node in \(\*C\), but \(Y\) is not in \(\Spo{\*C}\), then some node \(V \in \Spo{\*C} \setminus \Pa{\*C}\) `intercepts' this path, i.e., \(V\) is a closed collider and thus blocks the path from \(X\) to \(Y\).
    For example, in Fig.~\ref{fig:clmp_mb1}, taking \(\*C = \{X,B\}\), \(V = H\) and \(Y = E\), \(H\) blocks the path \(X \leftrightarrow H \leftrightarrow E\).
    \item If \(Y\) is in \(\Spo{\*C} \setminus \Pa{\*C}\), is an active bidirected path from \(X\) to \(Y\) when conditioning on \(\*C \setminus \{X\}\).
    For example, in Fig.~\ref{fig:clmp_mb1}, taking \(\*C = \{X,B\}\) and \(Y = A\), conditioning on \(\{B\}\) opens the path \(X \leftrightarrow B \leftrightarrow A\).
\end{itemize}

Analogous to how, for a given variable \(X\), different conditioning sets give rise to different CIs from \(X\), different ACs also give rise to different CIs from \(X\).
The upshot of defining ACs is that they carve out a relatively small set of CIs (commonly known as a `basis' \cite{bareinboim:etal20}) from which all CIs encoded in the given graph can be derived. The main result of this section, given below, establishes that GMP and C-LMP are equivalent.

\begin{theorem}[Equivalence of C-LMP and GMP]
\label{thm:equivalence:gmp:lmpplus}
    Let \(\G\) be a causal graph and \(\*V^\prec\) a consistent ordering.
    A probability distribution over \(\*V\) satisfies the global Markov property for \(\G\) if and only if it satisfies the c-component local Markov property for \(\G\) with respect to \(\*V^{\prec}\).
\end{theorem}

As a corollary of Thm.~\ref{thm:equivalence:gmp:lmpplus}, we can conclude that C-LMP is equivalent to Richardson's ordered local Markov property \cite{richardson2003markov}, since the latter is equivalent to GMP \citep[Thm.~2, Section~3.1]{richardson2003markov}. 

\begin{corollary}[Equivalence of C-LMP and the Ordered Local Markov Property]
\label{cor:equivalence:lmp:lmpplus}
    Let \(\G\) be a causal graph and \(\*V^\prec\) a consistent ordering.
    A probability distribution over \(\*V\) satisfies the ordered local Markov property \cite{richardson2003markov} for \(\G\) with respect to  \(\*V^\prec\) if and only if it satisfies the c-component local Markov property for \(\G\) with respect to \(\*V^{\prec}\).
\end{corollary}

In Appendix B, we further develop the connection between C-LMP and the ordered local Markov property.
In fact, in Thm. B.2.1, we show that these two properties induce the exact same set of CIs for a given DAG and a consistent ordering.
Thm. B.2.1 thus provides another way to obtain Thm.~\ref{thm:equivalence:gmp:lmpplus} as a corollary.

The equivalence of C-LMP and GMP means that the CIs invoked by C-LMP for a given causal DAG can be used to test the DAG against observational data. 

\subsection{Uniqueness Property of C-LMP}\label{sec:clmp:properties}

By definition, each CI invoked by C-LMP is generated from an AC. We further show that each CI can be generated from exactly one AC.

\begin{theorem}[Unique AC for each CI Invoked by C-LMP]
\label{thm:equivalence:clmpplus}
    Let \(\G\) be a causal graph, \(\*V^\prec\) a consistent ordering, and \(X\) a variable in \(\*V^\prec\).
    For every conditional independence relation invoked by the c-component local Markov property of the form \(X \indep \*W \mid \*Z\), there is exactly one ancestral c-component \(\*C \in \&{AC}_X\) such that \(\*W = \*V^{\leq X} \setminus ((\De{\Spo{\*C} \setminus \Pa{\*C}}) \cup \Pa{\*C})\) and \(\*Z = \Pa{\*C} \setminus \{X\}\).
\end{theorem}

The one-to-one correspondence between ACs and CIs invoked by C-LMP allows us to give bounds on the latter number that are tight in the exponent.

\begin{prop}[Number of CIs Invoked by C-LMP]
\label{prop:lmpsize}
    Given a causal graph \(\G\) and a consistent ordering  \(\*V^\prec\), let \(n\) and \(s \leq n\) denote the number of variables and the size of the largest c-component in \(\G\) respectively. Then, the c-component local Markov property for \(\G\) with respect to \(\*V^\prec\) invokes \(O(n 2^s)\) conditional independencies implied by \(\G\) over \(\*V\).
    Moreover, there exists a graph \(\G\) and a consistent ordering  \(\*V^\prec\)  for which the property induces \(\Omega(2^n)\) conditional independencies.
\end{prop}

This result shows that C-LMP offers an exponential improvement on the \(\Theta(4^n)\) CIs invoked by GMP.
However, C-LMP can still invoke an exponential number of CIs.
For example, in \(\G^{ex}\) (Fig.~\ref{fig:exp_ci}) with \(2n\) nodes, there are \(2^n + (n-3)\) CIs invoked by C-LMP.

The main upshot of the one-to-one correspondence between ACs and CIs invoked by C-LMP is that to list such CIs, it suffices to enumerate ACs.
We study the problem of listing CIs in the next section.

\begin{figure}[t]
    \centering
    \begin{subfigure}{0.26\textwidth}
        \includegraphics[width=\textwidth]{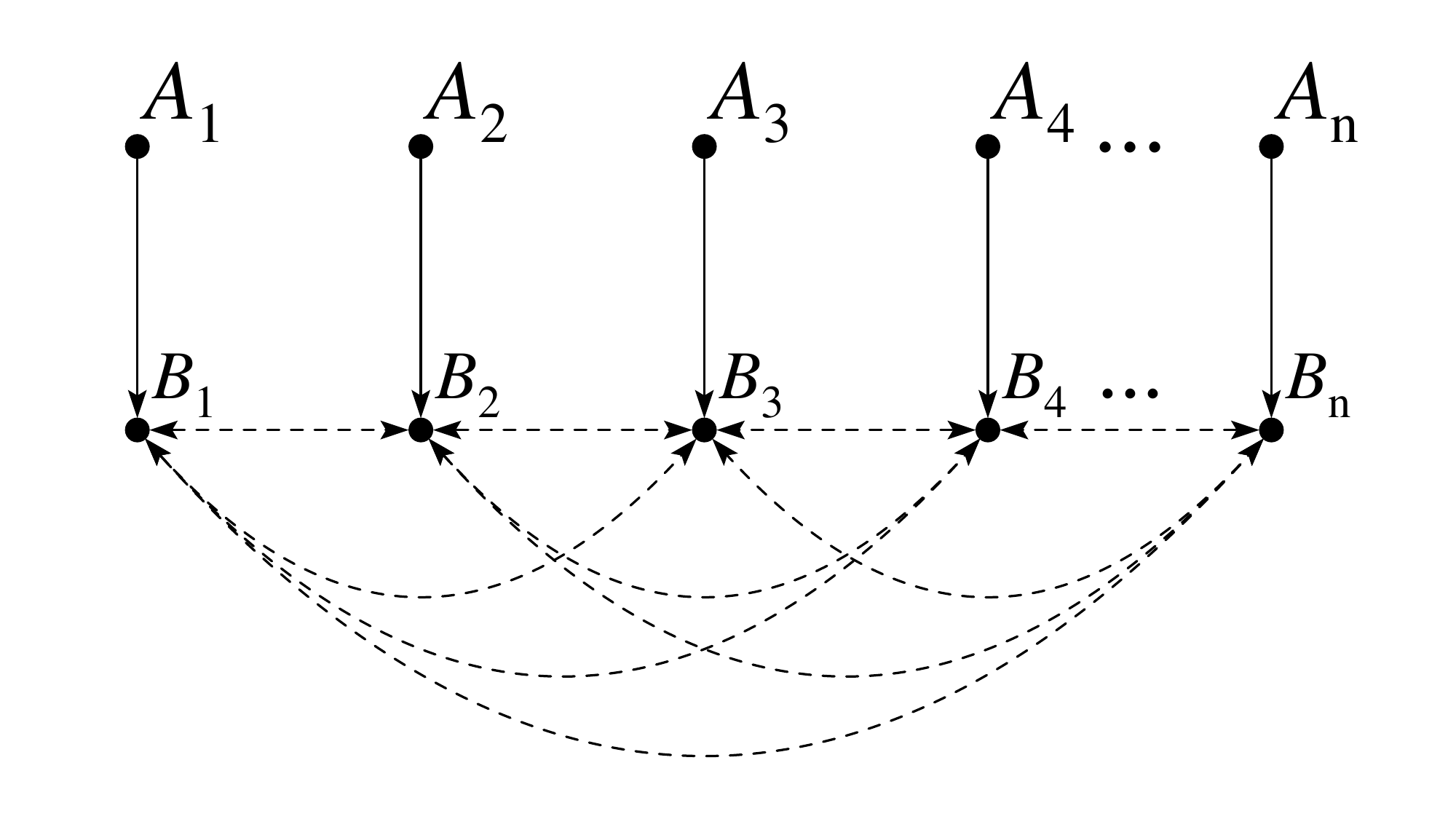}
        \caption{\(\G^{ex}\)}
        \label{fig:exp_ci}
    \end{subfigure}
    \hfill
    \begin{subfigure}{0.18\textwidth}
        \includegraphics[width=\textwidth]{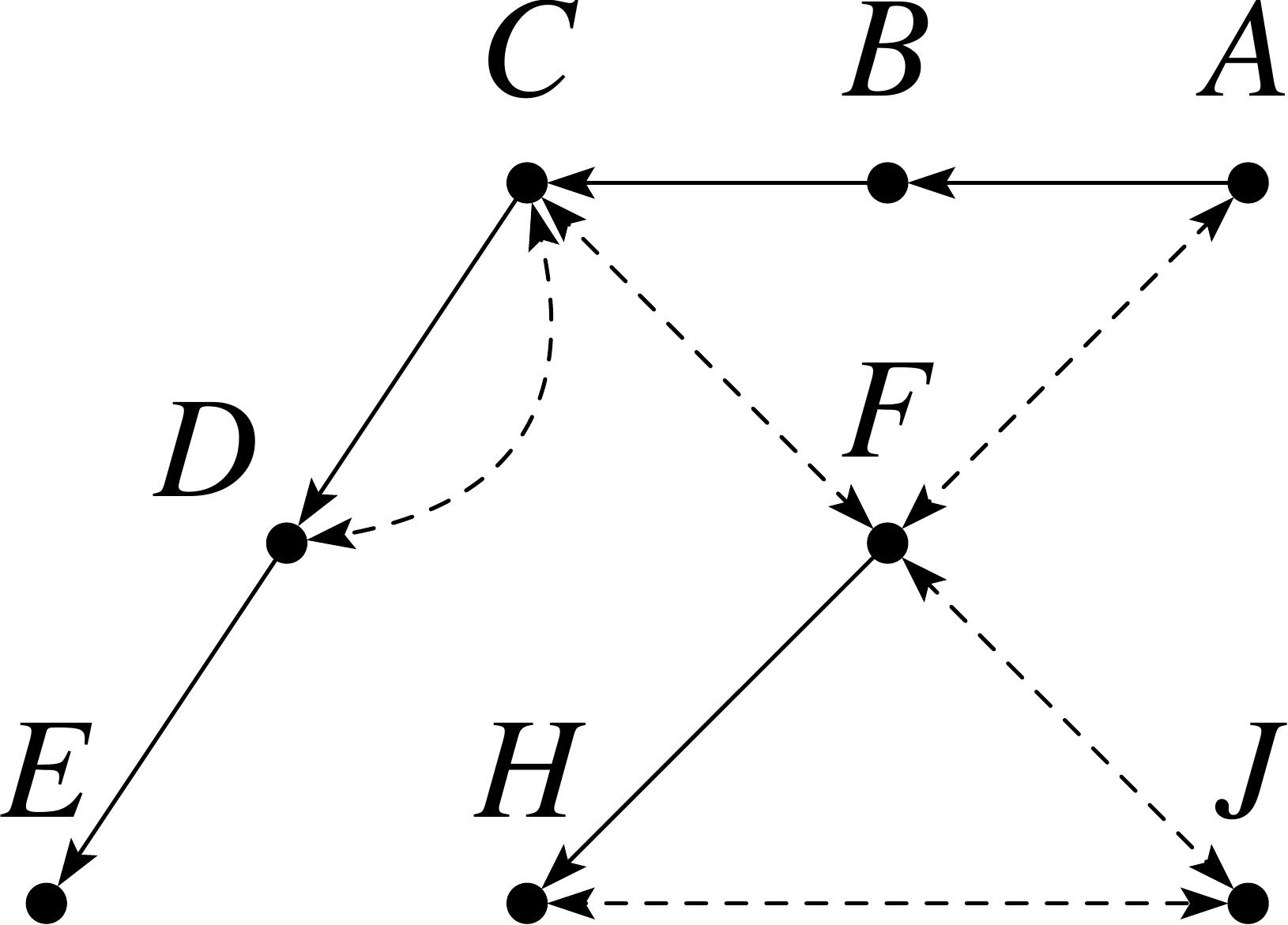}
        \caption{\(\G^3\)}
        \label{fig:listci}
    \end{subfigure}
\caption{
(a) An example showing that C-LMP may invoke an exponential number of CIs.
(b) A causal graph used to show the execution of \textsc{ListCI} in Ex.~\ref{ex:listci}.
}
\end{figure}

\section{Listing CIs}
\label{section:listci}

Our goal in this section is to develop an algorithm that lists CIs invoked by C-LMP. %
In the worst case, there may exist exponentially many such CIs, requiring exponential time to list them all.
In such cases, we look for algorithms that run in polynomial delay \cite{johnson:88}.
Poly-delay algorithms output the first solution (or indicate none is available) in poly-time, and take poly-time to output each consecutive solution.

However, not all CIs invoked by C-LMP are useful for model testing.
C-LMP invokes some `vacuous' CIs of the form \(X \indep \emptyset \mid \*Z\), which do not need testing.
Therefore, we constrain the problem by requiring that we list only \emph{non-vacuous} CIs, as defined below.

\begin{definition}[Vacuous CI and Admissible AC (AAC)]
\label{def:validci}
    Given a conditional independence relation invoked by C-LMP of the form \(X \indep \*W \mid \Pa{\*C} \setminus \{X\} \), where \(\*W = \*S^+ \setminus \Pa{\*C}\) (by Def.~\ref{def:lmpplus}), if \(\*W \neq \emptyset\), the conditional independence relation is said to be non-vacuous and \(\*C\) is said to be an admissible ancestral c-component relative to \(X\).
\end{definition}

\begin{example}
\label{ex:admissibleac}
    Consider the causal graph \(\G^3\) (Fig.~\ref{fig:listci}).
    The AC \(\{J\}\) relative to \(J\) is admissible.
    Given \(\*S^+ = \*V \setminus \{F,H\}\), we have \(\*W = \*S^+ \setminus \{J\} = \{A,B,C,D,E\}\).
    However, the AC \(\{F,J\}\) relative to \(J\) is not admissible.
    Since \(\*S^+ = \{F,J\}\), \(\*W = \*S^+ \setminus \{F,J\} = \emptyset\).
\qed
\end{example}

Listing only non-vacuous CIs is important since C-LMP may invoke exponentially many vacuous CIs.
To witness, consider a bidirected clique on \(n\) nodes such that no two variables are independent of each other given any conditioning set.
Every set \(\*Z \subseteq \*V \setminus \{X\}\) forms an AC, resulting in \(\Omega(2^n)\) vacuous CIs (see  Ex. D.3.1 in Appendix D.3 for details).

Our bounds on the number of CIs invoked by C-LMP are also tight for the number of non-vacuous CIs (Prop.~\ref{prop:lmpsize}). We develop the algorithm \textsc{ListCI} (Alg.~\ref{alg:listci}) to list all non-vacuous CIs invoked by C-LMP in poly-delay.%

\begin{algorithm}[t]
\caption{\textsc{ListCI} (\(\G, \*V^\prec\))}
\label{alg:listci}

\begin{algorithmic} [1]
    \State {\bfseries Input:} \(\G\) a causal diagram; \(\*V^\prec\) an ordering 
    consistent with \(\G\).
    
    \State {\bfseries Output:} Listing non-vacuous CIs invoked by C-LMP for \(\G\) with respect to \(\*V^\prec\).

    \State \textbf{for} {each \(X \in \*V^\prec\)} \textbf{do}
    \Indent
        \State \(\*I \gets \&C(X)_{\G_{\An{\{X\}}}}, \*R \gets \&C(X)_{\G_{\*V^{\leq X}}}\)   \label{alg:listci:i} \label{alg:listci:r}
        \State \(\textsc{ListCIX}(\G_{\*V^{\leq X}}, X, \*V^{\leq X}, \*I, \*R)\) \label{alg:listci:calllistcix}
    \EndIndent

\end{algorithmic}
\end{algorithm}

\begin{example}
\label{ex:listci}
    Consider the causal graph \(\G^3\) (Fig.~\ref{fig:listci}) with \(\*V^\prec = \{A,B,C,D,E,F,H,J\}\).
    \textsc{ListCI}\((\G^3,\*V^\prec)\) lists 11 non-vacuous CIs invoked by C-LMP:
    \(C \indep \{A\} \mid \{B\}\),
    \(D \indep \{A\} \mid \{B, C\}\),
    \(E \indep \{A, B, C\} \mid \{D\}\),
    \(F \indep \{B\} \mid \{A\}\),
    \(F \indep \{E\} \mid \{A, B, C, D\}\),
    \(H \indep \{A, B, C, D, E\} \mid \{F\}\),
    \(J \indep \{A, B, C, D, E\}\),
    \(J \indep \{B\} \mid \{A, F\}\),
    \(J \indep \{B\} \mid \{A, F, H\}\),
    \(J \indep \{E\} \mid \{A, B, C, D, F\}\),
    \(J \indep \{E\} \mid \{A, B, C, D, F, H\}\).
    After, \textsc{ListCI} terminates as there are no more non-vacuous CIs.
\qed
\end{example}

\subsection{Listing CIs for a Given Variable}

The algorithm \textsc{ListCI} iterates over each variable \(X \in \*V^{\prec}\) and lists all non-vacuous CIs invoked by C-LMP for \(X\).
By Defs.~\ref{def:lmpplus}, ~\ref{def:validci} and Thm.~\ref{thm:equivalence:clmpplus}, listing non-vacuous CIs reduces to enumerating AACs.
In this section, we show how to enumerate AACs relative to a given variable  \(X \in \*V^{\prec}\) using the procedure \textsc{ListCIX} (Alg.~\ref{func:listcix}).

\textsc{ListCIX} adopts a divide-and-conquer strategy similar to the algorithm of \cite{Takata2010}.
\textsc{ListCIX} implicitly constructs a binary search tree for \(X\) using a depth-first approach.
Tree nodes of the form \(\&N(\*I',\*R')\) represents the collection of all AACs \(\*C\) with \(\*I' \subseteq \*C \subseteq \*R'\).
The top-level call of \textsc{ListCIX}, at the root node \(\&N(\*I,\*R)\), represents all AACs \(\*C\) relative to \(X\).
This is due to the construction on line \ref{alg:listci:i} of \textsc{ListCI} so that \(\*I\) is contained in and \(\*R\) contains all possible AACs relative to \(X\).
Thus, the top-level call can generate all CIs for \(X\).

Subsequent recursive calls expand this tree by shrinking the range \(\*I' \subseteq \*R'\) one variable at a time.
One requirement of the poly-delay property is that each AAC should appear exactly once in the enumeration.
To expand the tree from \(\&N(\*I',\*R')\), \textsc{ListCIX}  constructs two `disjoint' children (lines~\ref{func:listcix:recursion:rprime}-\ref{func:listcix:recursion:iprime}: a chosen variable \(S \in \*V^{\prec}\) cannot be in any AAC from the left child but must be in every AAC from the right child.

As another requirement of the poly-delay property, we expand the tree from a node \(\&N(\*I',\*R')\) if and only if the expansion is guaranteed to produce a non-vacuous CI. 
Equivalently, there must exist at least one AAC \(\*C\) such that \(\*I' \subseteq \*C \subseteq \*R'\).
If there is no such \(\*C\), we prune the tree and back-track to the previous tree node.
To perform this check for the existence of an AAC in poly-time, \textsc{ListCIX} calls the function \textsc{FindAAC} (Alg.~\ref{func:findadmissiblec}). 
We explain \textsc{FindAAC} in the next subsection.

Finally, a leaf node is reached when \(\*I = \*R\). \textsc{ListCIX} outputs a non-vacuous CI generated from the AAC \(\*C = \*I\) using Def.~\ref{def:lmpplus}.

\begin{figure}[t]
    \centering
\includegraphics[width=.45\textwidth]{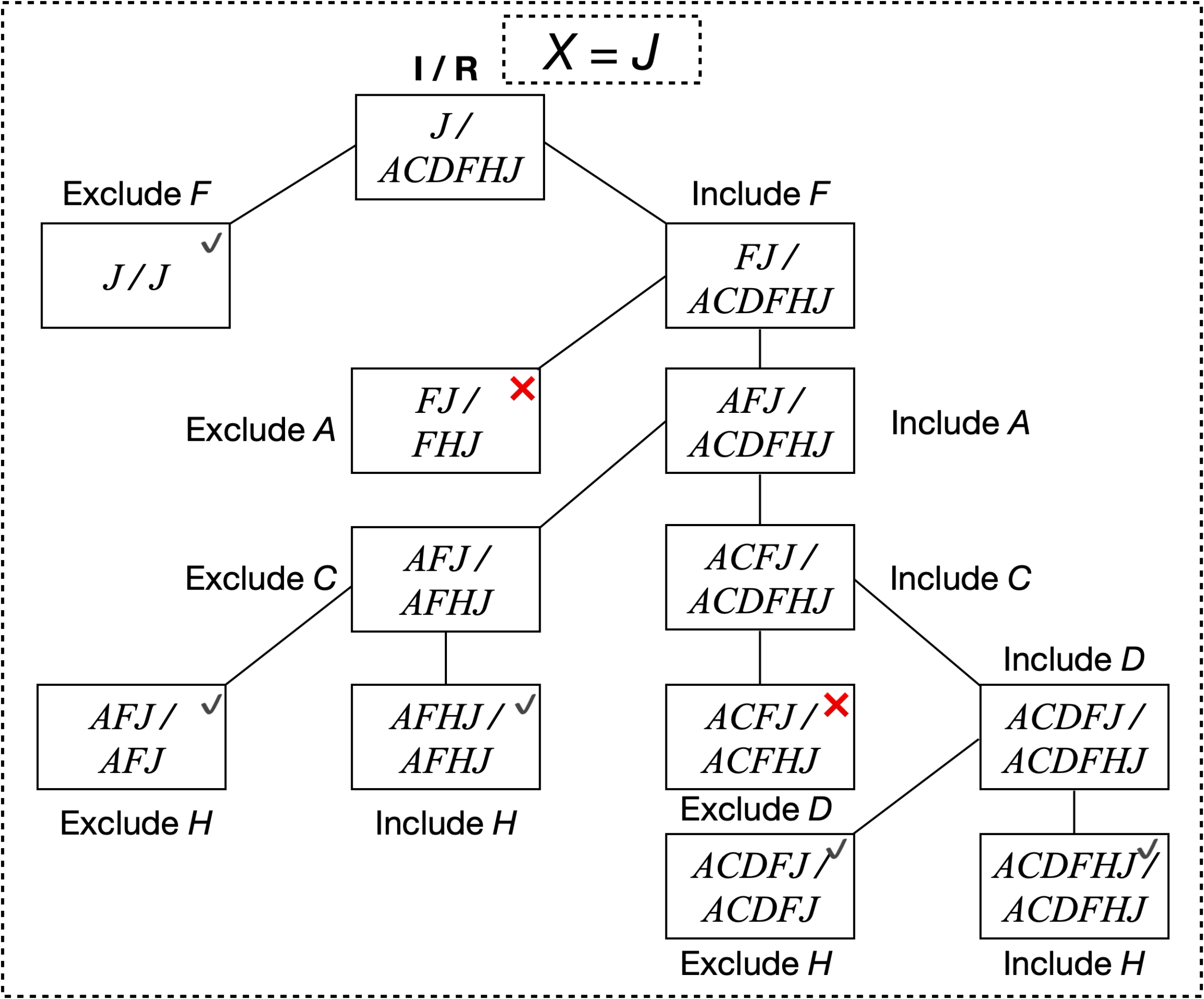}
\caption{
\(\&T^3\) a search tree illustrating the running of \textsc{ListCI} in Ex.~\ref{ex:listci} for \(X = J\).
}
\label{fig:listci:tree}
\end{figure}

\begin{example}
\label{ex:listcix:tree}
    Expanding Ex.~\ref{ex:listci} to demonstrate the construction of the search tree \(\&T^3\) (Fig.~\ref{fig:listci:tree}) generated by running \textsc{ListCI}(\(\G^3, \*V^\prec\)) for \(X = J\).
    With \(\*I = \{J\}\) and \(\*R = \{A,C,D,F,H,J\}\) constructed on line~\ref{alg:listci:i}, the initial search starts from the root node \(\&N(\*I,\*R)\) on line~\ref{alg:listci:calllistcix}.
    On line~\ref{func:listcix:callfindadmissiblec} of \textsc{ListCIX}, \textsc{FindAAC} returns \(\{J\}\).
    With \(S = F\) and \(\*R' = \{J\}\), the recursive call \textsc{ListCIX}\((\G^3,J,\*V^\prec,\*I,\*R')\) is made at line~\ref{func:listcix:recursion:rprime}, spawning a child \(\&N_1(\{J\},\{J\})\).
    The search continues from \(\&N_1\).
    \textsc{FindAAC} returns \(\{J\}\).
    \(\&N_1\) is a leaf node, and \textsc{ListCIX} outputs a CI: \(J \indep \{A, B, C, D, E\}\) on line~\ref{func:listcix:outputci}. The rest of the search tree for \(J\) is shown in \(\&T^3\).
    The full set of search trees is shown in Fig. D.3.1 in Appendix D.3.
\qed
\end{example}

\subsection{Finding an AAC}

In this section, we address the following subproblem, needed for \textsc{ListCIX} to run in poly-delay: given a variable \(X \in \*V^{\prec}\), and two ACs  \(\*I,\*R\) relative to \(X\), how do we find an AAC \(\*C\) such that \(\*I\subseteq \*C \subseteq \*R\) (or indicate that there is none) in poly-time?

The poly-time constraint on solving this subproblem rules out the brute-force approach: namely, iterating over all subsets \(\*C\) such that \(\*I \subseteq \*C \subseteq \*R\) until we find some \(\*C\) that is an AAC (or conclude that there is none).
We find that this exponential search is not necessary.
The key idea behind our solution is to reduce this search to a verification of whether a particular AAC is admissible.

To elaborate, assume there exists an AAC \(\*C\) such that \(\*I \subseteq \*C \subseteq \*R\). It is possible that \(\*C = \*I\), in which case we are done. Otherwise, consider what it means for  \(\*I\) not to be admissible. When conditioning on \(\Pa{\*I} \setminus \{X\}\), every variable in  \(\*V^{\leq X}\) has an active path to \(X\). This means that every such variable (outside the conditioning set) is a descendant of \(\Spo{\*I} \setminus \Pa{\*I}\)
(Def.~\ref{def:lmpplus}). Since \(\*C\) is admissible, then there must be some descendant of \(\Spo{\*I} \setminus \Pa{\*I}\), say \(D\), which has an active path to \(X\) when conditioning on \(\Pa{\*I} \setminus \{X\}\) but not when conditioning on \(\Pa{\*C} \setminus \{X\}\). 
We show that the AAC \(\*C\) can exist if and only if there exists \emph{any} set \(\*Z\) separating \(X\) from some such \(D\), with the restriction that \(\Pa{\*I}  \setminus \{X,D\} \subseteq \*Z \subseteq \Pa{\*R} \setminus \{X,D\}\).  
\(\*Z\) need not be an AAC.
We can check if such \(\*Z\) exists (line ~\ref{func:findadmissiblec:findseparator}) in poly-time using the function \textsc{FindSeparator} (Fig. C.2.2 in Appendix C.2).

\begin{example}
\label{ex:findadmissiblec}
    Expanding Ex.~\ref{ex:listcix:tree} to illustrate the usage of \textsc{FindAAC}.
    Let \(X = J\), \(\*V^{\prec} = \*V^{\leq J}\), \(\*I = \{J\}\), and \(\*R = \{A,C,D,F,H,J\}\).
    \textsc{FindAAC}(\(\G^3, J, \*V^{\leq J}, \*I, \*R\)) returns \(\*C = \{J\}\) since there exists an AAC \(\*C\) relative to \(J\) with \(\*I \subseteq \*C \subseteq \*R\).
    With \(\*I = \{F,J\}\) and \(\*R = \{F, H, J\}\), \textsc{FindAAC}(\(\G^3, J, \*V^{\leq J}, \*I, \*R\)) returns \(\perp\) since none of the ACs \(\*C\) relative to \(J\) with \(\*I \subseteq \*C \subseteq \*R\) are admissible.
\qed
\end{example}

\begin{algorithm}[t]
\caption{\textsc{ListCIX} (\(\G_{\*V^{\leq X}}, X, \*V^{\leq X}, \*I, \*R\))}
\label{func:listcix}

\begin{algorithmic} [1]
    \State {\bfseries Input:} \(\G_{\*V^{\leq X}}\) a causal diagram; \(X\) a variable; \(\*V^{\leq X}\) an ordering consistent with \(\G\); \(\*I\) and \(\*R\) ACs relative to \(X\).
    
    \State {\bfseries Output:} Listing non-vacuous CIs invoked by C-LMP associated with \(X\) and AACs \(\*C\) under the constraint \(\*I \subseteq \*C \subseteq \*R\).

    \State \textbf{if} {\(\textsc{FindAAC}(\G_{\*V^{\leq X}}, X, \*V^{\leq X}, \*I, \*R) \neq \perp\)}  \textbf{then} \label{func:listcix:callfindadmissiblec}
    \Indent
        \State \textbf{if} {\(\*I = \*R\)} \textbf{then}    \label{func:listcix:isleafnode}
        \Indent
            \State \(\*S^+ \gets \*V^{\leq X} \setminus \De{ \Spo{\*{I}} \setminus \Pa{\*I} } \)
            \State Output \(X \indep \*S^+ \setminus \Pa{\*I} \mid \Pa{\*I} \setminus \{X\}\) \label{func:listcix:outputci}
            \State \textbf{return}
        \EndIndent
        
        \State \(\*T \gets \*R \cap (\Spo{\*I} \setminus \*I), S \gets\) Any node in \(\*T\) \label{func:listcix:candidates}

        \State \(\*I' \gets \&C(X)_{\G_{\An{\*I \cup \{S\}} }},\*R' \gets \&C(X)_{\G_{\*R \setminus \De{\{S\}}}}\)    \label{func:listcix:iprime} \label{func:listcix:rprime}

        \State \(\textsc{ListCIX}(\G_{\*V^{\leq X}}, X, \*V^{\leq X}, \*I, \*R')\) \label{func:listcix:recursion:rprime}
        \State \(\textsc{ListCIX}(\G_{\*V^{\leq X}}, X, \*V^{\leq X}, \*I', \*R)\) \label{func:listcix:recursion:iprime}
    \EndIndent
\end{algorithmic}
\end{algorithm}

\begin{lemma}[Correctness of \textsc{FindAAC}]
\label{lemma:FindAAC:correctness}
    Given a causal graph \(\G\), a consistent ordering \(\*V^\prec\), and a variable \(X \in \*V^\prec\), let \(\*I,\*R\) be ancestral c-components relative to \(X\) such that \(\*I \subseteq \*R\). \textsc{FindAAC}(\(\G_{\*V^{\leq X}}, X, \*V^{\leq X}, \*I, \*R\)) outputs an admissible ancestral c-component \(\*{C}\) relative to \(X\) such that \(\*{I} \subseteq \*C \subseteq \*R\) if such a \(\*{C}\) exists, and \(\perp\) otherwise.
\end{lemma}

\begin{lemma}[Correctness of \textsc{ListCIX}]
\label{lemma:listcix}
    \textsc{ListCIX} (\(\G_{\*V^{\leq X}}, X, \*V^{\leq X}, \*I, \*R\)) enumerates all and only all non-vacuous conditional independence relations invoked by the c-component local Markov property associated with \(X\) and admissible ancestral c-components \(\*C\) relative to \(X\) where \(\*I \subseteq \*C \subseteq \*R\).
    Further, \textsc{ListCIX} runs in \(O(n^2(n+m))\) delay where \(n\) and \(m\) represent the number of nodes and edges in \(\G\), respectively.
\end{lemma}

Our results are summarized in the following theorem, which provides the soundness, completeness, and poly-delay complexity of the proposed algorithm. 

\begin{theorem}[Correctness of \textsc{ListCI}]
\label{thm:listci}
    Let \(\G\) be a causal graph and \(\*V^\prec\) a consistent ordering.
    \textsc{ListCI}(\(\G, \*V^\prec\)) enumerates all and only all non-vacuous conditional independence relations invoked by the c-component local Markov property in \(O(n^2(n+m))\) delay where \(n\) and \(m\) represent the number of nodes and edges in \(\G\), respectively.
\end{theorem}

\section{Experiments}
\label{section:experiments}

In this section, we first demonstrate the runtime of \textsc{ListCI} on benchmark DAGs of up to 100 nodes from the \texttt{bnlearn} repository \cite{scutari2010learning}.
Next, we apply \textsc{ListCI} to model testing on a real-world protein signaling dataset with an expert-provided graph \cite{sachs2005causal}.
Third, we provide analysis of the total number of non-vacuous CIs invoked by C-LMP, using \textsc{ListCI} for the analysis.
The details of the three experiments are shown in Appendix F.
\paragraph{Experiment 1 (Comparison of \textsc{ListCI} with other algorithms).}

We compare the runtime of \textsc{ListCI} with two baselines: \textsc{ListGMP} (Fig.~E.0.1 in Appendix E) and \textsc{ListCIBF} (Alg.~B.1.1 in Appendix~B.1)\footnote{Our implementation of \textsc{ListCIBF} can be improved by generating ancestral sets more efficiently. Regardless, we know \textsc{ListCI} performs better in theory (Sec.~\ref{sec:reformulation}), and have strong evidence that it is also better in practice to this more efficient implementation of \textsc{ListCIBF}.}.
\textsc{ListGMP} lists all CIs invoked by GMP (Def.~\ref{def:gmp}); \textsc{ListCIBF} iterates over ancestral sets to list CIs invoked by the ordered local Markov property \cite{richardson2003markov}.
The algorithms were run on DAGs that describe real-world scenarios from the \texttt{bnlearn} repository. 
Since the graphs are Markovian, non-Markovian graphs were generated by randomly assigning \(U\%\) of nodes to be unobserved for \(U \in \{0,10,20,\dots,90\}\). 
For each \(U\), we generated 10 random samples.
For a given graph, algorithm, and \(U\), if any one sample timed out (\(>\) 1 hour), no further samples are tested.
Fig.~\ref{fig:experiments:results} shows the average runtime of the algorithms, with further details in Fig.~F.1.1.

\begin{algorithm}[t]
\caption{\textsc{FindAAC} (\(\G_{\*V^{\leq X}}, X, \*V^{\leq X}, \*I, \*R\))}
\label{func:findadmissiblec}

\begin{algorithmic} [1]
    \State {\bfseries Input:} \(\G_{\*V^{\leq X}}\) a causal diagram; \(X\) a variable; \(\*V^{\leq X}\) an ordering consistent with \(\G\); \(\*I\) and \(\*R\) ACs relative to \(X\).
    
    \State {\bfseries Output:} An AAC \(\*C\) relative to \(X\) under the constraint \(\*I \subseteq \*C \subseteq \*R\), if such \(\*C\) exists; \(\perp\) otherwise.

    \State \textbf{if} {\textsc{IsAdmissible}(\(\G_{\*V^{\leq X}}, X, \*V^{\leq X}, \*{I}\))} \textbf{then} \label{func:FindAAC:callisadmissible}
    \Indent
        \State \textbf{return} \(\*{I}\)\label{func:FindAAC:returnI}
    \EndIndent
    \State \textbf{for} {each \(D \in \De{\Spo{\*{I}} \setminus \Pa{\*{I}}}\)} \textbf{do}  \label{func:FindAAC:foreachd}
    \Indent
       
        \State \(\*Z \gets \textsc{FindSeparator}(\G_{\*{V}^{\leq X}}, \{X\}, \{D\},\) \label{func:findadmissiblec:findseparator}
        \Indent
            \Indent
                \Indent
                    \(Pa(\*{I}), Pa(\*{R}))\)
                \EndIndent
            \EndIndent
        \EndIndent

        \State \textbf{if} {\(\*Z \neq \perp\)} \textbf{then}
        \Indent
            \State \textbf{return} \(\&{C}(X)_{\G_{\An{\*{I} \cup \*Z}}}\)\label{func:FindAAC:returnCZ}
        \EndIndent
    \EndIndent

    \State \textbf{return} \(\perp\)
\end{algorithmic}
\end{algorithm}

The results corroborate our theoretical conclusion that \textsc{ListCI} outperforms the other algorithms.
For \textsc{ListGMP}, the algorithm did not timeout on graphs with \(n < 10\) nodes.
For \textsc{ListCIBF}, we have mixed results.
The algorithm did not time out for some graphs with up to \(n = 35\) nodes, but there were other graphs with \(n=25\) where the algorithm did time out.
For \textsc{ListCI}, the algorithm did not timeout for many graphs up to \(n = 80\), but did time out for some graphs with \(n = 70\).

\paragraph{Experiment 2 (Application to model testing). }

A real-world protein signaling dataset \cite{sachs2005causal} has been used to benchmark causal discovery methods \cite{cundy:2021, zantedeschi:2023}.
The dataset (853 samples) comes with an expert-provided ground-truth DAG (11 nodes, 16 edges).
Using \textsc{ListCI}, we test to what extent this graph is compatible with the available data.
We use a kernel-based CI test from the \texttt{causal-learn} package
\cite{zheng2024causal} with p-value \(p = 0.05\).

For our chosen topological order, seven out of ten CIs invoked by C-LMP resulted in \(p > 0.05\). This suggests the ground-truth DAG may need revision before use as a benchmark for structure learning.
The exact local CIs that are violated may guide experts in this revision process.

\begin{figure}[t]
    \centering
    \includegraphics[width=.48\textwidth]{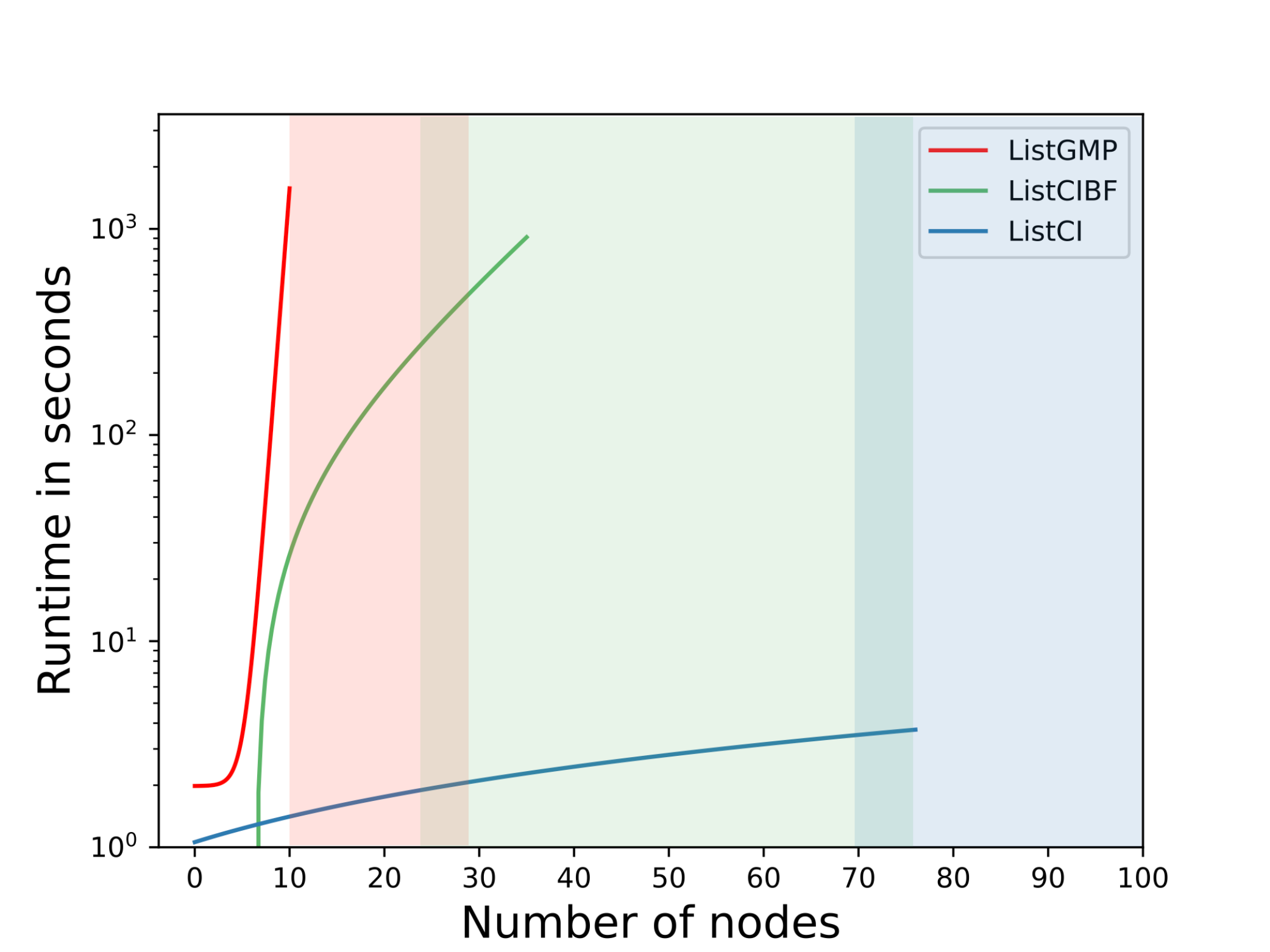}
    
\caption{
Plot of runtimes of the algorithms \textsc{ListGMP}, \textsc{ListCIBF}, and \textsc{ListCI} on graphs of various sizes.
A colored box indicates the interval of \(n\) on which the relevant algorithm has timed out on some graphs with \(n\) nodes.
The y-axis uses a logarithmic scale.
}
\label{fig:experiments:results}
\end{figure}

\paragraph{Experiment 3 (Analysis of C-LMP).}

We use \textsc{ListCI} to understand the total number of non-vacuous CIs invoked by C-LMP.
Let \(\*{CI}\) denote this number.
\(\*{CI}\) is also the number of CIs that need to be tested from a given semi-Markovian causal DAG.
Based on experiments with random graphs shown in Appendix~F.3, we conclude that the graph topology associated with c-components plays a major role in \(\*{CI}\).
More specifically, two factors related to c-components are of primary interest:
\begin{enumerate}
    \item \(s \leq n\): the size of the largest c-component, and
    \item The sparsity of c-components, a proxy for which is the number of bidirected edges.
\end{enumerate}

As we add bidirected edges, while c-components are sparse, \(\*{CI}\) increases exponentially with \(s\), as given by the bound \(O(n 2^s)\).
As c-components become more dense, \(\*{CI}\) decays exponentially with the number of bidirected edges.
As an illustrative example, please refer to Fig.~F.3.1 and the discussion on Case 1 in Appendix~F.3.

\section{Conclusions}
\label{section:conclusion}
In this paper, we introduced a new conditional independence property for causal models with unobserved confounders, namely, \textit{the c-component local Markov property} (C-LMP, Def.~\ref{def:lmpplus}).
Given a DAG \(\G\), C-LMP identifies a small subset of conditional independence constraints (CIs) that together imply all other CIs encoded in \(\G\).
We showed that C-LMP is equivalent to the global Markov property (Thm.~\ref{thm:equivalence:gmp:lmpplus}), and that each CI that C-LMP invokes can be generated from a unique ancestral c-component (Thm.~\ref{thm:equivalence:clmpplus}).
Building on this foundation, we developed the first algorithm \textsc{ListCI} (Alg.~\ref{alg:listci}) capable of listing all CIs invoked by C-LMP in polynomial delay (Thm.~\ref{thm:listci}).
Reducing the number of CI tests needed to
evaluate a causal model is important as it improves runtime per-
formance and helps mitigate concerns about statistical power
and multiple hypothesis testing.
We hope our work will help researchers test their causal assumptions using observational data prior to inference.

\section*{Acknowledgements} 
This research is supported in part by the NSF, DARPA, ONR, AFOSR, DoE, Amazon, JP Morgan, and The Alfred P. Sloan Foundation. 
We thank Robin Evans and the anonymous reviewers for their thoughtful comments.

\bibliography{references}

\theoremstyle{plain}
\newtheorem{adxtheorem}{Theorem}
\newtheorem{adxlemma}{Lemma}
\newtheorem{adxprop}{Proposition}
\newtheorem{adxcorollary}{Corollary}
\theoremstyle{definition}
\newtheorem{adxdefinition}{Definition}
\newtheorem{adxassumption}{Assumption}
\newtheorem{adxexample}{\ding{110} Example}
\theoremstyle{remark}
\newtheorem{adxremark}[theorem]{Remark}

\newtheorem{innercustomlemma}{Lemma}
\newenvironment{customlemma}[1]
  {\renewcommand\theinnercustomlemma{#1}\innercustomlemma}
  {\endinnercustomlemma}
\newtheorem{innercustomthm}{Theorem}
\newenvironment{customthm}[1]
  {\renewcommand\theinnercustomthm{#1}\innercustomthm}
  {\endinnercustomthm}
\newtheorem{innercustomprop}{Proposition}
\newenvironment{customprop}[1]
  {\renewcommand\theinnercustomprop{#1}\innercustomprop}
  {\endinnercustomprop}
  \newtheorem{innercustomcor}{Corollary}
\newenvironment{customcor}[1]
  {\renewcommand\theinnercustomcor{#1}\innercustomcor}
  {\endinnercustomcor}

\appendix

\begin{center}
    \Large
    \textbf{Appendices}
\end{center}

\begin{enumerate}[label=\normalsize\Alph*]

    \item Background and Previous Work
    \begin{enumerate}[label=\normalsize\arabic*]
        \item Background
        \item Related Work
    \end{enumerate}

    \item C-LMP and the ordered local Markov property
    \begin{enumerate}[label=\normalsize\arabic*]
        \item Brute-Force Listing of CIs invoked by (LMP,\(\prec\))
        \item Computing MBs and MASs using ACs
        \item Uniqueness Property of ACs
        \item Proofs
    \end{enumerate}
    
    \item Proofs
    \begin{enumerate}[label=\normalsize\arabic*]
        \item Section~\ref{sec:reformulation} Proofs
        \item Section~\ref{section:listci} Proofs
        \item Appendix Proofs
    \end{enumerate}

    \item Discussion and Examples
    \begin{enumerate}[label=\normalsize\arabic*]
        \item Explaining Markov properties
        \item C-LMP and the Semi-Markov Factorisation
        \item Examples
    \end{enumerate}

    \item Further Results
    \item Experimental Results
    \begin{enumerate}[label=\normalsize\arabic*]
        \item Comparison of \textsc{ListCI} with Other Algorithms
        \item Application to Model Testing
        \item Analysis of C-LMP
    \end{enumerate}

    \item Frequently Asked Questions
\end{enumerate}

\section{Background and Previous Work}

\subsection{Background}
In the appendix, for some integer \(k \geq 0\), we use \([k]\) to denote the set \(\{1, 2, \dots, k\}\) (with \([0] = \emptyset\)).
\subsubsection{Graph preliminaries.}
Let \(\*X\) be a set of variables in a DAG \(\G\) over variables \(\*V\).
We define four kinship relations:

\begin{enumerate}
    \item \emph{Parents} of \(\*X\), denoted \(\Pa{\*X}\): \(\Pa{\*X} = \{Y \in \*V \mid Y \to X \text{ for some } X \in \*X\} \cup \*X\).
    \item \emph{Ancestors} of \(\*X\), denoted \(\An{\*X}\): \(\An{\*X} = \{Y \in \*V \mid \text{ there is a directed path from } Y \text{ to } X \text{ for some } X \in \*X\} \cup \*X\).%
    \item \emph{Descendants} of \(\*X\), denoted \(\De{\*X}\): \(\De{\*X} = \{Y \in \*V \mid \text{ there is a directed path from } X \text{ to } Y \text{ for some } X \in \*X\} \cup \*X\) 

    \item \emph{Non-descendants} of \(\*X\), denoted \(\Nd{\*X}\): \(\*V \setminus \De{\*X}\). Note that \(\Nd{\*X}\) does not include \(\*X\).
\end{enumerate}
Readers may be familiar with spouses of a variable \(X\) as variables \(Y\) such that \(X\) and \(Y\) are both the parent of some \(W\). We use a different sense of spouse consistent with \cite{pearl:2k,richardson2003markov}, defined in Section~\ref{sec:prelim}.

\subsubsection{The ordered local Markov property.} We define the ordered local Markov property \cite{richardson2003markov} for semi-Markovian causal DAGs and its basic components below.  

\begin{adxdefinition}{(Markov Blanket (MB)) \cite{richardson2003markov}}
\label{def:mb}
    Given a causal graph \(\G\) and a consistent ordering \(\*V^\prec\), let \(X\) be a variable in \(\*V^\prec\) and \(\*S\) an ancestral set in \(\G\) such that \(X \in \*S \subseteq \*V^{\leq X}\).
    Then, the Markov blanket of \(X\) with respect to the induced subgraph \(\G_\*S\), denoted  \(\mkbk{X,\*S}\), is defined as \(\mkbk{X, \*S} = \Pa{\&C(X)_{\G_{\*S}}}_{\G_{\*S}} \setminus \{X\}\).
\end{adxdefinition}

\begin{adxdefinition}{(Maximal Ancestral Set (MAS)) \cite{richardson2003markov}}
\label{def:maxanc}
    Given a causal graph \(\G\) and a consistent ordering \(\*V^\prec\), let \(X\) be a variable in \(\*V^\prec\) and \(\*S\) an ancestral set in \(\G\) such that \(X \in \*S \subseteq \*V^{\leq X}\).
    Then, \(\*S\) is said to be maximal with respect to the Markov blanket \(\mkbk{X,\*S}\) if, for any ancestral set \(\*S'\) such that \(X \in \*S \subseteq \*S' \subseteq \*V^{\leq X}\) and \(\mkbk{X,\*S} = \mkbk{X,\*S'}\), we have \(\*S = \*S'\).
\end{adxdefinition}

We state Richardson's \emph{ordered local Markov property} (with quantification MASs instead of all ancestral sets \citep[Section~3.1]{richardson2003markov}).

\begin{adxdefinition}{(The Ordered Local Markov Property (LMP,\(\prec\))) \cite{richardson2003markov}}
\label{def:lmp}
    A probability distribution \(P(\*v)\) over variables \(\*V\) is said to satisfy the ordered local Markov property for \(\G\) with respect to the consistent ordering \(\*V^\prec\) if, for any variable \(X\) and ancestral set \(\*S\) such that \(X \in \*S \subseteq \*V^{\leq X}\) and \(\*S\) is maximal with respect to \(\mkbk{X, \*S}\), \[
        X \indep \*S \setminus (\mkbk{X, \*S} \cup \{X\}) \mid \mkbk{X, \*S} \text{ in } P(\*v).
    \]
\end{adxdefinition}

Finally, we introduce the following collections to understand the web of ancestral sets, MBs, and MASs.
\begin{adxdefinition}
\label{def:collectionsczs}
    Given a causal graph \(\G\), a consistent ordering \(\*V^\prec\), and a variable \(X \in \*V^\prec\),
    define three collections:
    \begin{equation*}
        \begin{split}
            \&{S}_X &= \{ X \in \*S \subseteq \*V^{\leq X} \mid \  \*S \text{ is ancestral }\}, \\
            \&{Z}_X &= \{ \*Z \mid \*Z = \mkbk{X,\*S} \text{ for some } \*S \in \&S_X \}, \text{ and } \\
            \&{S^+}_X &= \{ \*S^+ \in \&S_X \mid \*S^+ \text{ is maximal w.r.t. } \mkbk{X,\*S^+}\}.
        \end{split}
    \end{equation*}
\end{adxdefinition}

\begin{adxexample}
    Consider \(H\) in \(\G^2\) (Fig.~\ref{fig:intro_local_semimarkov}). We have a c-component \(\&{C}(7)_{\G^2} = 347\). The ancestral set \(\*S = 12347\) induces \(\mkbk{7,\*S} = 234\). The MAS with respect to this MB is \(\*{S^+} = 1234567\), resulting in the CI \(7 \indep 156 \mid 234\) invoked by (LMP,\(\prec\)).
\qed
\end{adxexample}

It is known that GMP and (LMP,\(\prec\)) are equivalent: for a causal graph \(\G\) and consistent ordering \(\*V^{\prec}\), a probability distribution satisfies the global Markov property for \(\G\) if and only if it satisfies the local Markov property for \(\G\) with respect to \(\*V^{\prec}\) \citep[Thm.~2, Section~3.1]{richardson2003markov}.

The following lemma provides a (poly-time) test for whether a given set is maximal with respect to the MB that it induces.
\begin{adxlemma}{(Testing Maximality of Ancestral Set) \citep[Lemma.~5]{richardson2003markov}}
\label{adxlemma:maxanc}
    Given a causal graph \(\G\) and a consistent ordering \(\*V^\prec\), let \(X\) be a variable in \(\*V^\prec\).
    An ancestral set \(\*S \in \&{S}_X\) is maximal with respect to the Markov blanket \(\mkbk{X,\*S}\) if and only if:
    \[
        \*S = \*V^{\leq X} \setminus \De{ h(X,\*S) }_{\G}
    \]
    where
    \[
        h(X,\*S) = \Spo{\&{C}(X)_{\G_{\*S}} }_{\G} \setminus ( \mkbk{X,\*S} \cup \{X\} ).
    \]
\end{adxlemma}

\subsubsection{A note on Markov blankets.} We offer some clarification on the term `Markov blanket' as used in this paper (Def.~\ref{def:mb}), introduced by \cite{richardson2003markov}. 
The more widely known concept of a Markov blanket is due to \citep[Def.~3.12]{pearl:88a}.
Given a set of variables \(\*V\) and a variable \(X \in \*V\), a Pearlian Markov blanket (abbreviated as PMB) is a set of variables \(\*Z \subseteq \*V \setminus \{X\}\) such that \(X \indep \*V \setminus (\{X\} \cup \*Z) \mid \*Z\).
Returning to Fig.~\ref{fig:intro_local_semimarkov}, the variable \(8\) has a PMB \(\*Z_1 = 2346\) since \(7 \indep 15 \mid 2346\). 
\(\*Z_1\) is not an MB (per Def.~\ref{def:mb}).
\(8\) has another PMB \(\*Z_2 = 234\) since \(7 \indep 156 \mid  234\). 
\(\*Z_2\) is also, coincidentally, an MB, though not all MBs are PMBs.
The key differences between MBs and PMBs are twofold:
\begin{enumerate}
    \item In the definition of an MB, we choose an ancestral set \(\*S \subseteq \*V^{\leq X}\) containing \(X\), and require that \(X \indep \*S \setminus (\{X\} \cup \mkbk{X,\*S}) \mid \mkbk{X,\*S}\) holds; the MB separates \(X\) from all other variables in \(\*S\) but not necessarily those in \(\*V \setminus \*X\). The PMB must separate \(X\) from all other variables in \(\*V\).
    \item A given ancestral set \(\*S\) induces exactly one MB for a variable \(X\). However, there may be multiple PMBs for \(X\). An MB is more akin to the notion of a \emph{Markov boundary} \citep[Def.~3.12]{pearl:88a}\footnote{A Markov boundary is a minimal (Pearlian) Markov blanket, such that any strict subset of the Markov boundary no longer separates the variable from all other variables in the graph.}, in the sense that it is `minimal'; removing any variable \(Y\) from the MB \(\mkbk{X,\*S}\) no longer guarantees the independence \(X \indep  \*S \setminus (\{X\} \cup (\mkbk{X,\*S} \setminus \{Y\})) \mid \mkbk{X,\*S} \setminus \{Y\}\).
\end{enumerate}
MBs, therefore, are closely related to PMBs but with additional features needed to define and ensure that (LMP, \(\prec\)) is equivalent to GMP.

\subsection{Related Work}
\label{sec:related_work}

In this section, we expand on the Markov properties and algorithms to enumerate them summarised in Table~\ref{table:contribution}.

For model testing, a Markov property which invokes only a polynomial number of CI tests is ideal.
However, currently known poly-size properties assume either 1) there is no latent confounding between variables, or 2) the given causal DAG does not contain any directed mixed cycles, or 3) the observational distribution satisfies certain additional constraints \cite{kang2009markov}.
Intuitively, a directed mixed cycle is a cycle formed by walking through arrows in one direction.
For instance, in the causal DAG \(\G^2\) (Fig.~\ref{fig:intro_local_semimarkov}), the path \(2 \rightarrow 3 \rightarrow 7 \leftrightarrow 3\) is a directed mixed cycle.
Directed mixed cycles are commonly found in semi-Markovian DAGs – even in the basic bow pattern, in which a variable \(X\) is a cause of \(Y\) and \(X,Y\) have a latent confounder \cite{pearl:2k}.
There is no known poly-sized Markov property for the general setting.

There are two known Markov properties for Markovian causal DAGs.

\begin{enumerate}
    \item LMP: The local Markov property \cite{pearl:88a,lauritzen:etal90,lauritzen:96}.
    LMP specifies a linear number of CIs in total: one for each variable \(X\), stating that \(X\) is conditionally independent of its non-descendants given its parents.

    \item PMP: The pairwise Markov property \cite{pearl:mes99}.
    For a graph with \(n\) variables, PMP invokes \(O(n^2)\) CIs: more specifically, one CI for each pair of non-adjacent variables.
    PMP assumes that the given probability distribution is a compositional graphoid: that is, it additionally satisfies the intersection and composition axioms. 
\end{enumerate}

The intersection and composition axioms do not hold in arbitrary distributions. The intersection axioms holds, for example, in distributions which have full support (\(P(\*v) > 0\) for all \(\*v\)), e.g., a multivariate Gaussian. Composition holds in multivariate Gaussians and in probability distributions that are faithful to some DAG.

The following are known Markov properties for semi-Markovian causal DAGs.

\begin{table}[t]
    \footnotesize
    \centering
    \begin{tabular}{@{} *{6}{c} @{}}
        \toprule
         & \multicolumn{2}{c}{Coverage} & \multicolumn{2}{c@{}}{Scalability} \\
        \cmidrule(lr){2-3} \cmidrule(l){4-5}
        Property & \multirow{2}{1cm}{\textbf{Latents}} & \multirow{2}{1.3cm}{\centering \textbf{Any Prob. Distr.}} & \multirow{2}{1.2cm}{\centering \textbf{Poly-size CIs}} & \multirow{2}{*}{\centering \textbf{\text{Poly-Delay}}} \\
         & & & & \\
        \midrule%
        \textbf{LMP} & \crsmark & \chkmark & \chkmark & \chkmark \\
        \midrule
        \textbf{PMP} & \crsmark & \trimarkblue & \chkmark & \chkmark \\
        \midrule
        \textbf{RLMP} & \trimark & \trimarkblue & \chkmark & \chkmark \\
        \midrule
        \textbf{(RLMP,\(\prec\))} & \trimark & \trimarkblue & \chkmark & \chkmark \\
        \midrule
        \textbf{PMP-C} & \trimark & \trimarkblue & \chkmark &  \chkmark \\
        \midrule
        \textbf{PMP-RS} & \trimark & \trimarkblue & \chkmark &  \chkmark \\
        \midrule
        \textbf{S-Markov} & \chkmark & \chkmark & \crsmark & \crsmark \\
        \midrule
        \textbf{(LMP,\(\prec\))} & \chkmark & \chkmark & \crsmark & \crsmark \\
        \midrule
        \textbf{C-LMP} (ours) & \chkmark & \chkmark & \crsmark & \chkmark \\
        \midrule
    \end{tabular}
\caption{\footnotesize{Summary of properties and algorithms to enumerate CIs invoked by such properties.
The first column denotes if the property applies to graphs with unobserved confounders; the second, if it applies to arbitrary observational distributions; the third, if it invokes a polynomial number of CIs; the fourth, if there is a poly-delay algorithm to list its invoked CIs.
\color{Green}{\cmark} \color{Black} denotes an addressed area. \textcolor{red}{\xmark} denotes an unaddressed area.
\trimark \color{Black} denotes that DAGs may contain unobserved variables but not directed mixed cycles \cite{kang2009markov}, or the input is a MAG, a tranformation of a DAG \cite{richardson:spi02}. \trimarkblue \color{Black} denotes that further assumptions must be made on the probability distribution.
}}
\label{table:contribution}
\end{table}
\begin{enumerate}
    \item RLMP: The reduced local Markov property \cite{kang2009markov}.
    RLMP invokes a linear number of CIs in total, one for each variable.
    RLMP states that a variable is independent of the variables that are neither its descendants nor the descendants of its spouses, conditioning on its parents.
    The property assumes that the given probability distribution satisfies the composition axiom and the DAG has no directed mixed cycles.

    \item (RLMP,\(\prec\)): The ordered reduced local Markov property \cite{kang2009markov}.
    Given a specific ordering of variables called a \textit{c-ordering} \cite{kang2009markov}, (RLMP,\(\prec\)) invokes a linear number of CIs in total.
    (RLMP,\(\prec\)) states that each variable is independent of its predecessors (excluding its spouses) in a c-ordering, given its parents.
    The property assumes that the given probability distribution satisfies the composition axiom and the DAG has no directed mixed cycles.

    \item PMP-C: The pairwise Markov property \cite{kang2009markov}.
    Given a c-ordering, PMP-C invokes \(O(n^2)\) many CIs:
    more specifically, one CI for each pair of non-adjacent variables.
    PMP-C assumes that the given probability distribution satisfies the composition axiom and the DAG has no directed mixed cycles.

    \item PMP-RS: The pairwise Markov property given by \cite{richardson:spi02}. PMP-RS invokes \(O(n^2)\) many CIs, one for each pair of non-adjacent variables, for a given \emph{maximal ancestral graph} (MAG). A semi-Markovian DAG can be transformed into a MAG which encodes exactly the same CIs. It thus suffices to test CIs in the resultant MAG \cite{shipley:21}. However, the equivalence between this pairwise Markov property and the global Markov property has only been proved for probability distributions that are compositional graphoids \cite{lauritzen:2018}.    
    \item S-Markov: The \(S\)-Markov property \cite{kang2009markov}.
    S-Markov relaxes the assumption of the given graph containing no directed mixed cycles.
    Still, S-Markov assumes that the observational distribution satisfies the composition axiom.
    For each variables in the graph that can be c-ordered, S-Markov invokes a linear number of CIs.
    However, for variables that are not c-ordered, S-Markov relies on the ordered local Markov property (LMP,\(\prec\)), which, as discussed, is exponential-sized.
\end{enumerate}

CIs are the only type of constraint that Markovian DAGs impose on the observational distribution. In the non-Markovian case, however, DAGs may encode more complex equality and inequality constraints such as \emph{Verma constraints} \cite{verma:pea90a}.
While such constraints are outside the scope of this work, there are algorithms that list these constraints in addition to CIs.
However, these algorithms do not run in poly-delay.

\section{C-LMP and the Ordered Local Markov Property}
\label{appendix:lmp}

(LMP, \(\prec\)) is a well-known Markov property that applies to arbitrary observational distributions and causal graphs with unobserved confounders. In this section, we first explain how naively following the definition of (LMP,\(\prec\)) can take exponential time to output just one CI.
Next, we characterize (LMP, \(\prec\)) in more depth and show how ACs (Def.~\ref{def:ancestralccomponent}) can be used to compute the CIs that (LMP, \(\prec\)) invokes.

\subsection{Brute-Force Listing of CIs Invoked by (LMP,\(\prec\))}
\label{subsection:lmpbruteforce}
\begin{algorithm}[t]
\caption{\textsc{ListCIBF} (\(\G, \*V^\prec\))}
\label{alg:listcibf}

\begin{algorithmic} [1]
    \State {\bfseries Input:} \(\G\) a causal diagram; \(\*V^\prec\) an ordering %
    consistent with \(\G\).
    
    \State {\bfseries Output:} Listing CIs invoked by (LMP,\(\prec\)) for \(\G\) with respect to \(\*V^\prec\).

    \State \textbf{for} {each \(X \in \*V^\prec\)} \textbf{do}
    \Indent
        \State \textbf{for} {each ancestral set \(\*S\) such that \(X \in \*S \subseteq \*V^{\leq X}\)} \textbf{do}
        
        \Indent
            \State \textbf{if} {\(\*S\) is maximal with respect to \(\mkbk{X,\*S}\)\footnotemark{}} \textbf{then}
            \Indent
                \State Output \(X \indep \*S \setminus (\mkbk{X, \*S} \cup \{X\}) | \mkbk{X, \*S}\)
            \EndIndent
        \EndIndent
    \EndIndent
\end{algorithmic}
\end{algorithm}

\begin{figure}[t]
    \centering
    \null\hfill%
    \begin{subfigure}{0.14\textwidth}
        \includegraphics[width=\textwidth]{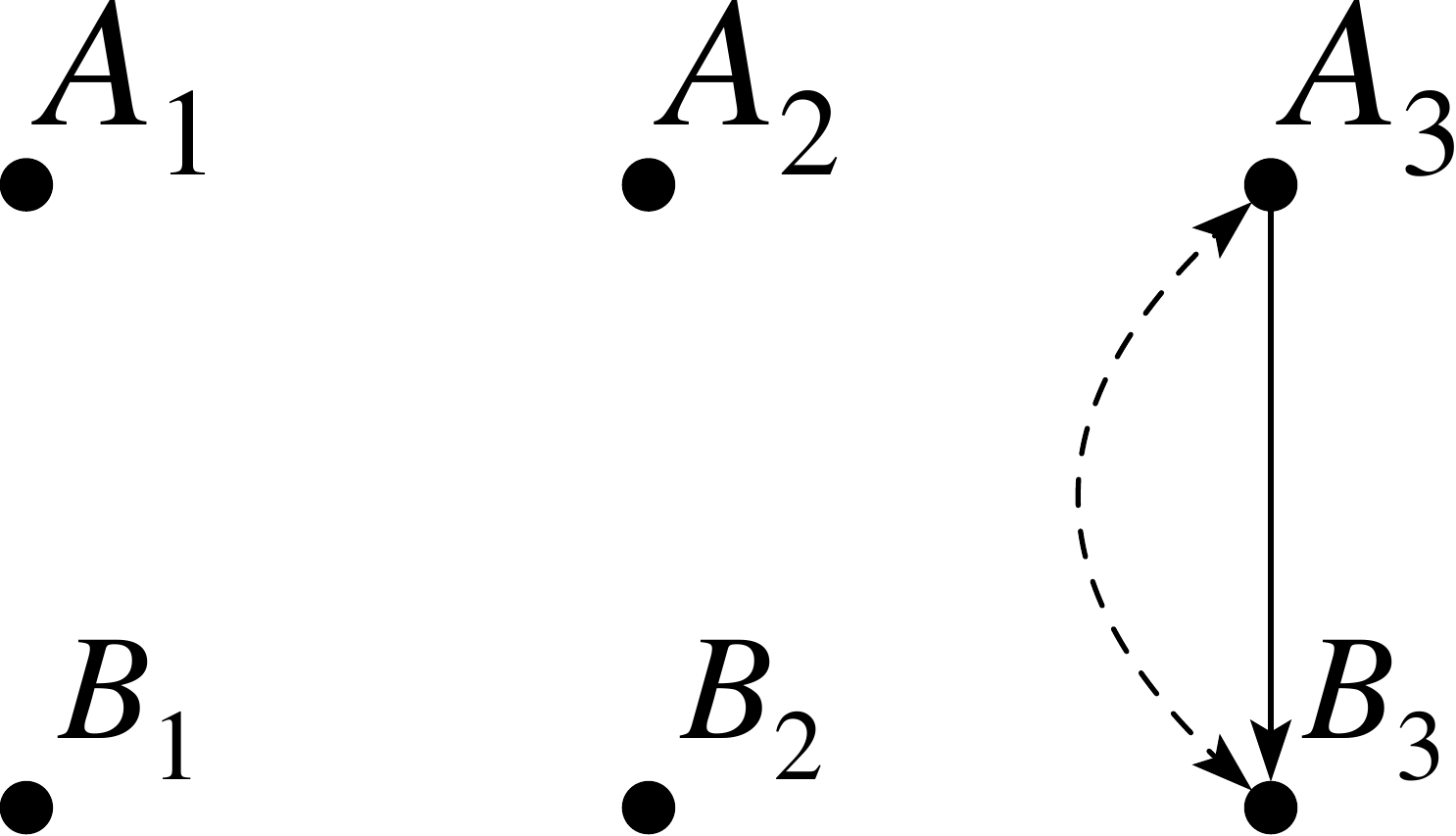}
        \caption{\(\G^{e1}\)}
        \label{fig:reduction:exp:1}
    \end{subfigure}
    \hfill
    \begin{subfigure}{0.25\textwidth}
        \includegraphics[width=\textwidth]{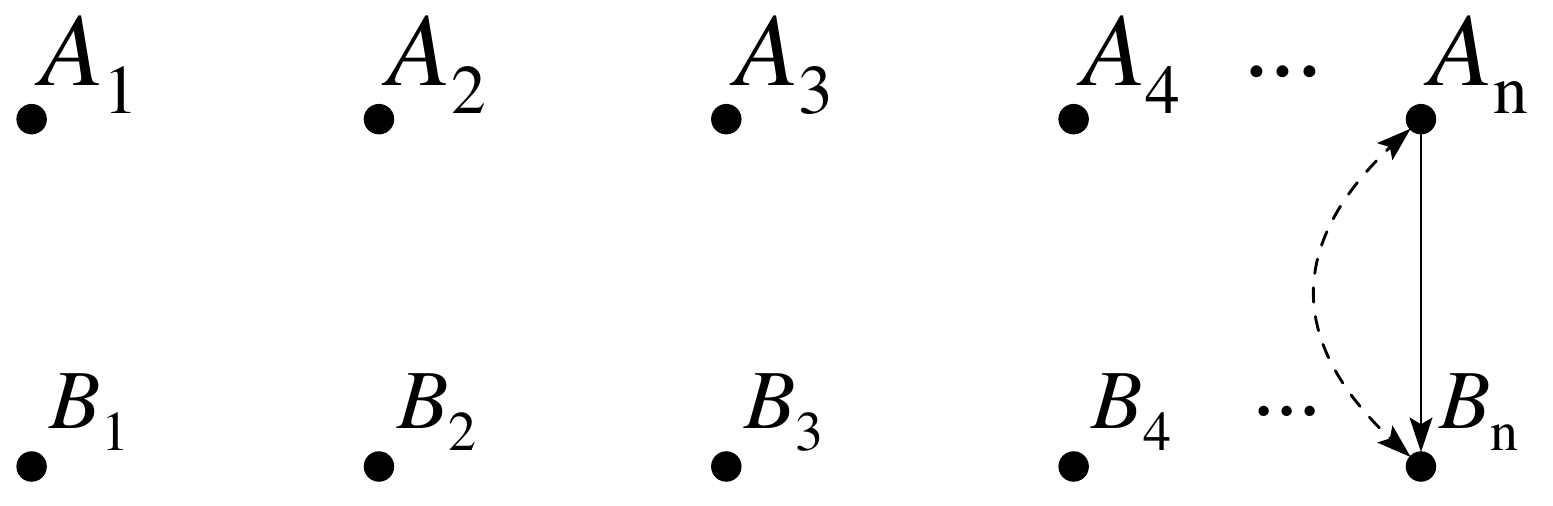}
        \caption{\(\G^{e2}\)}
        \label{fig:reduction:exp:3}
    \end{subfigure}
\caption{
Examples showing that a brute-force approach (\textsc{ListCIBF}) may take exponential time to output one CI invoked by (LMP,\(\prec\)).
}
\label{fig:reduction:exp}
\end{figure}

By definition, we can list the CIs invoked by (LMP,\(\prec\)) (Def.~\ref{def:lmp}) by enumerating over MASs.
However, it is unclear how to enumerate over MASs.
Each MAS is defined relative to an MB, and each MB is defined relative to an ancestral set.
Then, an immediate approach is to iterate over all ancestral sets \(\*S\), verifying if \(\*S\) is maximal with respect to \(\mkbk{X,\*S}\) before we output its corresponding CI constraint.
We implement this approach in the algorithm \textsc{ListCIBF} (Alg.~\ref{alg:listcibf}).

\begin{adxexample}
\label{ex:reduction:listcibf}
    Consider the DAG \(\G^{e1}\) (Fig.~\ref{fig:reduction:exp:1}) with consistent ordering \(\*V^\prec = \{A_1,A_2,A_3,B_1,B_2,B_3\}\).
    \textsc{ListCIBF}(\(\G^{e1},\*V^\prec\)) outputs five CIs invoked by (LMP,\(\prec\)): \(A_2 \indep \{A_1\}\), \(A_3 \indep \{A_1,A_2\}\), \(B_1 \indep \{A_1,A_2,A_3\}\), \(B_2 \indep \{A_1,A_2,A_3,B_2\}\), and \(B_3 \indep \{A_1,A_2,B_1,B_2\} \mid \{A_3\}\).
\qed
\end{adxexample}

In Ex.~\ref{ex:reduction:listcibf}, given \(X = B_3\), \textsc{ListCIBF}(\(\G^{e1},\*V^\prec\)) iterates over \(2^4\) different ancestral sets \(\*S\) with \(B_3 \in  \*S \subseteq \*V^{\leq B_3}\), all of which produce the same \(\mkbk{B_3, \*S} = \{A_3\}\).
However, only \(\*S^+ = \*V^{\leq B_3}\) is maximal with respect to this MB, resulting in the CI: \( B_3 \indep \{A_1,A_2,B_1,B_2\} \mid \{A_3\}\).
\textsc{ListCIBF} goes over \(2^4\) different ancestral sets to output this CI.
Next, we generalize this example to show that \textsc{ListCIBF} may iterate over exponentially many ancestral sets (with respect to the number of variables in \(\G\)) that produce the same MB.

\footnotetext{A poly-time test for whether an ancestral set \(\*S\) is maximal with respect to \(\mkbk{X,\*S}\) is shown in Lemma~\ref{adxlemma:maxanc}.}

\begin{adxexample}
\label{ex:reduction:exp:anc}
    In \(\G^{e2}\) (Fig.~\ref{fig:reduction:exp:3}) with \(2n\) nodes, there are \(2^{n-1} + 2^{n-2} - 1\) ancestral sets and \(n\) of them are maximal.
\qed
\end{adxexample}

In other words, iterating over all ancestral sets naively is potentially sub-optimal.

In the following lemma, we make a key observation: while there may be many ancestral sets producing the same MB (so that \(|\&{S}_X| > |\&{Z}_X|\)), exactly one ancestral set is maximal with respect to this MB.
As a result, \textsc{ListCIBF} may take exponential time to output just one new CI.

\begin{adxlemma}[One-to-one Correspondence between \(\&{Z}_X\) and \(\&{S^{+}}_X\)]
\label{lemma:equivalence:mbsplus}
    Given a causal graph \(\G\) and a consistent ordering \(\*V^\prec\), let \(X\) be a variable in \(\*V^\prec\).
    There is a bijection \(f:\&{Z}_X \to \&{S^{+}}_X\) given by \(f(\*Z) = \*S^+\) where \(\*S^+ \in \&{S^{+}}_X\) is an ancestral set maximal with respect to \(\*Z \in \&{Z}_X\). The inverse of \(f\), \(g: \&{S^{+}}_X \to \&{Z}_X\), is given by \(g(\*S^+) = \mkbk{X,\*S^+}\).
\end{adxlemma}

\begin{adxexample}
    Continuing Ex.~\ref{ex:reduction:listcibf}.
    Given a variable \(B_3\), there exists only one MAS \(\*V^{\leq B_3}\) with respect to the MB \(\*Z = \{A_3\}\) of \(B_3\).
    We have \(\&{Z}_{B_3} = \{ \{A_3\} \}\) and \(\&{S^{+}}_{B_3} = \{ \*V^{\leq B_3} \}\).
    \(\*V^{\leq B_3}\) maps uniquely to \(\{A_3\}\), and vice versa.
\qed
\end{adxexample}

\subsection{Computing MBs and MASs using ACs} \label{adx:sec:zsfromc} 

Listing CIs invoked by (LMP, \(\prec\)) is challenging due to the many-to-one mapping from ancestral sets to CIs. Minimally, we want to be able to list these CIs without brute-force iteration. Fundamental to our solution is the fact that multiple ancestral sets induce the same CI only because they induce the same AC (Def.~\ref{def:ancestralccomponent}).

Observe that exponentially many ancestral sets may induce the same AC.  For instance, in \(\G^{e1}\) (Fig.~\ref{fig:reduction:exp:1}) with \(X = B_3\), \(2^4\) different ancestral sets \(\*S \in \&{S}_{B_3}\) induce the same AC, \(\&{C}(B_3)_{\G_{\*S}} = \{A_3,B_3\}\).  

We show that all ancestral sets inducing the same MB and MAS must induce the same AC.

\begin{adxprop}[Equality of MBs Implies Equality of ACs]
\label{prop:mkbk-ccomp}
    Given a causal graph \(\G\) and a consistent ordering \(\*V^\prec\), for any variable \(X \in \*V^\prec\) and any ancestral sets \(\*S_1, \*S_2 \in \&{S}_X\), if \(\mkbk{X,\*S_1} = \mkbk{X,\*S_2}\), then \(\&{C}(X)_{\G_{\*S_1}} = \&{C}(X)_{\G_{\*S_2}}\).
\end{adxprop}
 
Moreover, the converse is also true: all ancestral sets inducing the same AC must induce the same MB and MAS. In particular, for a variable \(X\), given \(\*C = \&{C}(X)_{\G_\*S}\) for some ancestral set \(\*S\), we can compute \(\*Z = \mkbk{X,\*S}\) and the MAS \(\*{S^+}\) relative to \(\*Z\) in poly-time without using \(\*S\).
The following results show how MB and MAS can be computed from AC.

\begin{adxprop}[Construction of MB from AC]
\label{prop:equivalence:cmb}
    Given a causal graph \(\G\) and a consistent ordering \(\*V^\prec\), let \(X\) be a variable in \(\*V^\prec\).
    Fix an ancestral c-component \(\*C \in \&{AC}_X\). 
    For any ancestral set \(\*S \in \&{S}_X\) such that  \(\&{C}(X)_{\G_{\*S}} = \*C\), we have \(\mkbk{X, \*S} = \Pa{\*C} \setminus \{X\}\).
\end{adxprop}

\begin{adxprop}[Construction of MAS from AC]
\label{prop:equivalence:csplus}
    Given a causal graph \(\G\) and a consistent ordering \(\*V^\prec\), let \(X\) be a variable \(\*V^\prec\).
    Fix an ancestral c-component \(\*C \in \&{AC}_X\).
    For any ancestral set \(\*S \in \&{S}_X\) such that \(\&{C}(X)_{\G_{\*S}} = \*C\), the unique ancestral set \(\*S^+ \in \&{S^+}_X\) maximal with respect to the Markov blanket \(\mkbk{X,\*S}\) is given by \(\*S^+ = \*V^{\leq X} \setminus \De{\Spo{\*C} \setminus \Pa{\*C}}\).
\end{adxprop}

\begin{adxexample}
\label{ex:mbsplus}
    Consider the DAG \(\G^{e1}\) (Fig.~\ref{fig:reduction:exp:1}) with consistent ordering \(\*V^\prec = \{A_1,A_2,A_3,B_1,B_2,B_3\}\).
    Given a variable \(B_3\), \(\*C = \{A_3,B_3\}\) is an AC relative to \(B_3\).
    We compute the MB \(\*Z\) from \(\*C\) as follows: \(\Pa{\{A_3,B_3\}} \setminus \{B_3\} = \{A_3\} = \*Z\).
    For all ancestral sets \(\*S \in \&{S}_{B_3}\), we have \(\mkbk{B_3, \*S} = \Pa{\&C(B_3)_{\G_{\*S}}}_{\G_{\*S}} \setminus \{B_3\} = \{A_3\} = \*Z\).
    The MAS \(\*{S^+}\) relative to \(\*Z\) is given by \(\*{S^+} = \*V^{\leq B_3} \setminus \De{\Spo{\{A_3,B_3\}} \setminus \Pa{\{A_3,B_3\}}} = \*V^{\leq B_3} \setminus \De{\emptyset} = \*V^{\leq B_3} = \*S^+\).
\qed
\end{adxexample}

These results, in part, motivate our definition of a local Markov property via ancestral c-components i.e., C-LMP (Def.~\ref{def:lmpplus}). In fact, we can show the following equivalence between C-LMP and (LMP,\(\prec\)).

\begin{adxtheorem}[Correspondence between C-LMP and (LMP,\(\prec\))]
\label{adxthm:equivalence:lmp:lmpplus}
    Let \(\G\) be a causal graph and \(\*V^\prec\) a consistent ordering.
    The c-component local Markov property and the ordered local Markov property \cite{richardson2003markov} for \(\G\) with respect to \(\*V^\prec\) induce an identical set of conditional independence relations implied by \(\G\) over \(\*V\).
\end{adxtheorem}

\begin{proof}
    Given a causal graph \(\G\) and a consistent ordering \(\*V^\prec\), let $\&{L}^R$ denote the set of CIs implied by the ordered local Markov property for \(\G\) with respect to \(\*V^\prec\), and $\&{L}^C$ the set of CIs implied by the c-component local Markov property for \(\G\) with respect to \(\*V^\prec\). We show that \(\&{L}^R = \&{L}^C\).
    \begin{enumerate}
     \item (\(\&{L}^R \subseteq \&{L}^C\)) Consider a CI statement in  \(\&{L}^R\) of the form
        \[
            X \indep \*S^+ \setminus (\mkbk{X, \*S^+} \cup \{X\}) \mid \mkbk{X, \*S^+}
        \]
        for some variable \(X \in \*V^\prec\) and an ancestral set \(\*S^+ \in \&{S^+}_X\) maximal with respect to \(\mkbk{X, \*S^+}\). We show that the same CI statement is also in \(\&{L}^C\).

        Let \(\*C = \&{C}(X)_{\G_{\*S^+}}\).
        Since \(\*S^+\) is ancestral, \(\*C\) is an AC relative to \(X\).
        By Def.~\ref{def:lmpplus}, the following CI is in \(\&{L}^C\).
        \[
            X \indep \*S^{+'} \setminus \Pa{\*C} \mid (\Pa{\*C} \setminus \{X\})
        \]
        where
        \[
         \*S^{+'} = \*V^{\leq X} \setminus\De{\Spo{\*C} \setminus \Pa{\*C}}
        \]
        By Prop.~\ref{prop:equivalence:cmb}, \(\mkbk{X, \*S^{+}} = \Pa{\*C} \setminus \{X\}\).
        By Prop.~\ref{prop:equivalence:csplus}, \(\*S^+ = \*S^{+'}\).
        Therefore, the two CI statements are identical, and the given CI from \(\&{L}^R\) is also in \(\&{L}^C\).
    \item (\(\&{L}^C \subseteq \&{L}^R\)) Consider a CI statement in \(\&{L}^C\) of the form
        \[
            X \indep \*S^{+} \setminus \Pa{\*C} \mid \Pa{\*C} \setminus \{X\}
        \]
        where
        \[
            \*S^{+} = \*V^{\leq X} \setminus \De{\Spo{\*C} \setminus \Pa{\*C}}
        \]
         for some variable \(X \in \*V^\prec\) and AC \(\*C\in \&{AC}_X\). By Def.~\ref{def:ancestralccomponent}, there exists an ancestral set \(\*S \in \&{S}_X\) such that \(\*C = \&{C}(X)_{\G_{\*S}}\).
         By Prop.~\ref{prop:equivalence:cmb}, \(\mkbk{X,\*S} = \Pa{\*C} \setminus \{X\}\).
         By Prop.~\ref{prop:equivalence:csplus}, \(\*S^+\) is the unique ancestral set maximal with respect to \(\mkbk{X,\*S}\).
         By Def.~\ref{def:lmp}, the following CI is in \(\&{L}^R\)
         \[
            X \indep \*S^+ \setminus (\mkbk{X, \*S^+} \cup \{X\}) \mid \mkbk{X, \*S^+}
         \]
         Therefore, the two CI statements are identical, and the given CI from \(\&{L}^C\) is also in \(\&{L}^R\).
    \end{enumerate}
\end{proof}

The equivalence between C-LMP and GMP (Thm.~\ref{thm:equivalence:gmp:lmpplus}) can also be proved as a corollary of the above Thm.~\ref{adxthm:equivalence:lmp:lmpplus} and the equivalence between (LMP,\(\prec\)) and GMP \citep[Thm.~2, Section~3.1]{richardson2003markov}.

\subsection{Uniqueness Property of ACs}
\label{adx:sec:czsonetoone}

Recall that in (LMP, \(\prec\)), multiple ancestral sets can induce the same MB. Here, we show this can be remedied using ACs: each MB can be computed from exactly one AC. 

\begin{adxlemma}[One-to-one Correspondence between \(\&{AC}_X\) and \(\&{Z}_X\)]
\label{lemma:equivalence:cmb}
    Let \(\G\) be a causal graph, \(\*V^\prec\) a consistent ordering, and \(X\) a variable in \(\*V^\prec\).
    Then, there is a bijection \(f:\&{AC}_X \to \&{Z}_X\) given by \(f(\*C) = \Pa{\*C} \setminus \{X\}\) with \(\*C \in \&{AC}_X\). The inverse of \(f\), \(g: \&{Z}_X \to \&{AC}_X\), is given by \(g(\*Z) = \&{C}(X)_{\G_\*S}\) where \(\*S\) is an arbitrary ancestral set in \(\&{S}_X\) such that \(\*Z = \mkbk{X,\*S}\).
\end{adxlemma}

Both Lemma~\ref{lemma:equivalence:mbsplus} and Lemma~\ref{lemma:equivalence:cmb} imply a one-to-one correspondence between ACs and MASs.

\begin{adxcorollary}[One-to-one Correspondence between \(\&{AC}_X\) and \(\&{S^{+}}_X\)]
\label{cor:equivalence:csplus}
    Let \(\G\) be a causal graph, \(\*V^\prec\) a consistent ordering, and \(X\) a variable in \(\*V^\prec\).
    There is a bijection \(f: \&{AC}_X \to \&{S^{+}}_X\).
\end{adxcorollary}

Fig.~\ref{fig:buildingblocks} provides an overview of the relationships among the sets of ancestral sets, ACs, MBs, and MASs.
A core implication is that each CI invoked by (LMP,\(\prec\)) can be derived from exactly one AC, which we exploit in C-LMP (Def.~\ref{def:lmpplus}).

\subsection{Proofs}
We present proofs of the results in Sections~\ref{adx:sec:zsfromc} and \ref{adx:sec:czsonetoone}. We first prove some technical propositions.

\begin{adxprop}[AC in Union of Subgraphs]
~\label{prop:ccomp-union}
    Given a causal graph \(\G\) over a set of variables \(\*V\), for any subsets \(\*S_1, \*S_2 \subseteq \*V\) and a variable \(X \in \*V\), if \(\&{C}(X)_{\G_{\*S_1}} = \&{C}(X)_{\G_{\*S_2}} = \*C\) then \(\&{C}(X)_{\G_{\*S_1 \cup \*S_2}} = \*C\).
\end{adxprop}

\begin{proof}
    Since \(\*S_1 \subseteq \*S_1 \cup \*S_2\), we have \(\*{C} \subseteq \&{C}(X)_{\G_{\*S_1 \cup \*S_2}}\). 
    To show the other direction, for any variable \(U \in \&{C}(X)_{\G_{\*S_1 \cup \*S_2}} \setminus \{X\}\), let \(\pi = \{X \leftrightarrow V_1, V_1 \leftrightarrow V_2, \dots, V_{k-1} \leftrightarrow V_k, V_k \leftrightarrow V_{k+1} = U\}\) be the bidirected path from \(X\) to \(U\) in \(\G_{\*S_1 \cup \*S_2}\) (for some \(k \geq 0\)).   For each \(i \in [k+1]\), we have \(V_i \in \*S_1 \cup \*S_2\). We prove by induction on the index \(i \in [k+1]\) that \(V_i \in \*C\) for each \(i \in [k+1]\).
    
    \emph{Base case.} If \(V_1 \in \*S_1\)then \(X \leftrightarrow V_1\) implies \(V_1 \in C(X)_{\G_{\*S_1}} = \*C\). 
    Otherwise, \(V_1 \in \*S_2\) %
    implies \(V_1 \in C(X)_{\G_{\*S_2}} = \*C\). 
    
    \emph{Inductive hypothesis.} If \(k \geq 1\), assume for some \(i \in [k]\) we have \(V_i \in \mathbf{C}  = C(X)_{\G_{\*S_1}} =C(X)_{\G_{\*S_2}} \).  
    
    \emph{Inductive step.} Then, either \(V_{i+1} \in \*S_1\) or \(V_{i+1} \in 
    \*S_2\). If \(V_{i+1} \in \*S_1\),  then \(V_{i} \in C(X)_{\G_{\*S_1}}\) (by the induction hypothesis) and \(V_i \leftrightarrow V_{i+1}\) implies \(V_{i+1} \in C(X)_{\G_{\*S_1}}\).  Otherwise, \(V_{i+1} \in \*S_2\)  and  \(V_{i} \in C(X)_{\G_{\*S_2}}\) (by the induction hypothesis) \(V_i \leftrightarrow V_{i+1}\) implies \(V_{i+1} \in C(X)_{\G_{\*S_2}}\).  By induction, it follows that \(U = V_{k+1} \in \mathbf{C}\).
\end{proof}

\begin{proof}[Proof of Prop.~\ref{prop:mkbk-ccomp}]
    Consider \(\*S_1, \*S_2 \in \&{S}_X\). If \(\mkbk{X,\*S_1} = \mkbk{X,\*S_2}\), then \(\mkbk{X,\*S_1} \subseteq \*S_2\) and \(\mkbk{X,\*S_2} \subseteq \*S_1\) since \(\mkbk{X,\*S_1} \subseteq \*S_1, \ \mkbk{X,\*S_2} \subseteq \*S_2\). This implies \(\&{C}(X)_{\G_{\*S_1}} \subseteq \mkbk{X,\*S_1} \cup \{X\} \subseteq \*S_2\) and \(\&{C}(X)_{\G_{\*S_2}} \subseteq \mkbk{X,\*S_2}\cup \{X\} \subseteq \*S_1\). However, \(\&{C}(X)_{\G_{\*S_1}} \subseteq \*S_2 \implies \&{C}(X)_{\G_{\*S_1}} \subseteq \&{C}(X)_{\G_{\*S_2}}\) and similarly, \( \&{C}(X)_{\G_{\*S_2}} \subseteq \&{C}(X)_{\G_{\*S_1}}\). Therefore, \(\&{C}(X)_{\G_{\*S_1}}= \&{C}(X)_{\G_{\*S_S}}\).
\end{proof}

\begin{adxprop}[MB in Union of Subgraphs]
\label{prop:mkbk-union}  
    Given a causal graph \(\G\) and a consistent ordering \(\*V^\prec\), for any variable \(X \in \*V\) and any ancestral sets \(\*S_1, \*S_2 \in \&{S}_X\), if \(\mkbk{X,\*S_1} = \mkbk{X,\*S_2} = \*Z\), then \(\mkbk{X, \*S_1 \cup \*S_2} = \*Z\).
\end{adxprop}

\begin{proof}
    Consider \(\*S_1, \*S_2 \in \&{S}_X\). If \(\mkbk{X,\*S_1} = \mkbk{X,\*S_2}\), then by Prop.~\ref{prop:mkbk-ccomp}, we have \(\&{C}(X)_{\G_{\*S_1}} = \&{C}(X)_{\G_{\*S_2}}\). Since \(\*S_1, \*S_2\) are ancestral, we have \(\An{\*S_1 \cup \*S_2}_\G = \An{\*S_1}_\G \cup \An{\*S_2}_\G = \*S_1 \cup \*S_2\), hence \(\*S_1 \cup \*S_2\) is also ancestral. We get 
    \begin{align}
        \mkbk{X, \*S_1 \cup \*S_2} & = \Pa{\&{C}(X)_{\G_{\*S_1 \cup \*S_2}}}_{\G_{\*S_1 \cup \*S_2}} \setminus \{X\} \\
        &=  \Pa{\&{C}(X)_{\G_{\*S_1}}}_{\G_{\*S_1 \cup \*S_2}}\setminus \{X\} \tag{By Prop.~\ref{prop:ccomp-union} since \(\&{C}(X)_{\G_{\*S_1}} = \&{C}(X)_{\G_{\*S_2}}\)} \\
        &= \Pa{\&{C}(X)_{\G_{\*S_1}}}_{\G_{\*V^{\leq X}}}\setminus \{X\} \tag{\(\*S_1 \cup \*S_2\) ancestral} \\
        &= \Pa{\&{C}(X)_{\G_{\*S_1}}}_{\G_{\*S_1}} \setminus \{X\}\tag{\(\*S_1\) ancestral} \\
        &= \mkbk{X,\*S_1}
    \end{align}
\end{proof}

\begin{figure}[t]
    \centering
    \includegraphics[width=.47\textwidth]{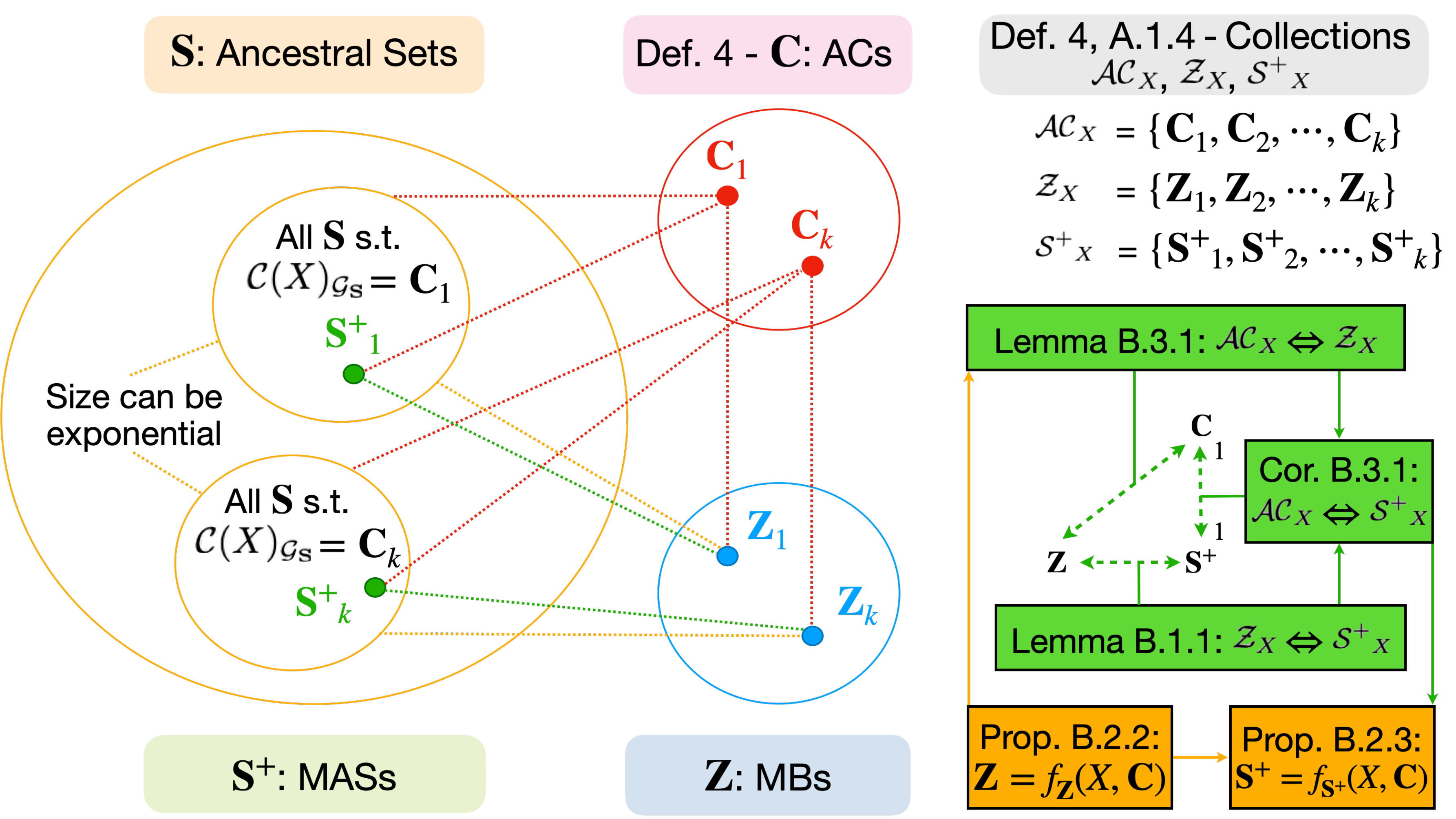}
\caption{
An overview of the relationships among ancestral sets, ACs, MBs,  MASs.
Exponentially many ancestral sets may map to one AC, MB, or MAS.
On the other hand, there is a one-to-one-correspondence among ACs, MBs and MASs.
}
\label{fig:buildingblocks}
\end{figure}

We now move to our main results in Sections \ref{adx:sec:zsfromc}. and \ref{adx:sec:czsonetoone}. 

\begin{proof}[Proof of Lemma~\ref{lemma:equivalence:mbsplus}]
    First, we show the mapping \(f\) is well-defined. Given \(\*Z \in \&{Z}_X\), there exists an ancestral set \(\*S^+ \in \&{S^{+}}_X\) maximal with respect to \(\*Z\). It remains to show that there is exactly one such \(\*S^+\). Let \(\*S_1, \*S_2 \in \&{S^+}_X\) be ancestral sets maximal with respect to \(\*Z\). The equality \(\mkbk{X,\*S_1} = \mkbk{X,\*S_2} = \*Z\) implies \(\mkbk{X,\*S_1 \cup \*S_2} = \*Z\) (by Prop.~\ref{prop:mkbk-union}). Therefore, \(\*S_1 \subseteq \*S_1 \cup \*S_2\) and the maximality of \(\*S_1\) implies \(\*S_1 =\*S_1 \cup \*S_2\) and \(\*S_2 \subseteq \*S_1\). Similarly, \(\*S_1 \subseteq \*S_2\). Therefore, \(\*S_1 = \*S_2\).
    
    Finally, \(g\) is well-defined. \(f(g(\*S^+)) = f(\mkbk{X,\*S^+}) = \*S^+\), and \(g(f(\*Z)) = \*Z\) since \(f(\*Z) = \*S^+\) is maximal with respect to \(\*Z\) if and only if \(\mkbk{X,\*S^+} = \*Z\). Since \(f\) has a two-sided inverse \(g\), \(f\) is bijective.
\end{proof}

\begin{proof}[Proof of Prop.~\ref{prop:equivalence:cmb}]
    For any ancestral set \(\*S \in \&{S}_X\) with \(\&{C}(X)_{\G_{\*S}} = \*C\), we have 
    \begin{align}
        \mkbk{X, \*S} &= \Pa{\&{C}(X)_{\G_{\*S}}}_{\G_\*S} \backslash \{X\} \\
        &= \Pa{\*C}_{\G_\*S} \backslash \{X\}\tag{By definition of \(\*S\)} \\
        &= \Pa{\*C} \backslash \{X\} \tag{\(\*S\) is ancestral and \(\*C \subseteq \*S\)}
    \end{align}
\end{proof}

\begin{proof}[Proof of Prop.~\ref{prop:equivalence:csplus}]
    For any ancestral set \(\*S \in \&{S}_X\) with \(\&{C}(X)_{\G_{\*S}} = \*C\), let \(\*S^+ \in \&{S^+}_X\) be an ancestral set maximal with respect to  \(\mkbk{X, \*S}\). Note that  \(\mkbk{X, \*S} =  \mkbk{X, \*S^+}\) by definition. By Prop.~\ref{prop:mkbk-ccomp}, we have \(\&{C}(X)_{\G_{\*S^+}} = \&{C}(X)_{\G_{\*S}} = \*C\). Then, \(\*S^+\) is maximal with respect to \( \mkbk{X, \*S^+}\) if and only if
    \begin{align}
        \*S^+ & =  \*V^{\leq X} \setminus \De{ \Spo{\&{C}(X)_{\G_{\*S^+}} } \setminus (\mkbk{X,\*S^+} \cup \{X\}) } \tag{Lemma~\ref{adxlemma:maxanc}} \\
        & =  \*V^{\leq X} \setminus \De{ \Spo{\*C} \setminus (\mkbk{X,\*S^+} \cup \{X\}) }\\
        &= \*V^{\leq X} \setminus \De{ \Spo{\*C} \setminus ( (\Pa{\*C} \setminus \{X\} ) \cup \{X\} ) } \tag{Prop.~\ref{prop:equivalence:cmb}}\\
        & = \*V^{\leq X} \setminus \De{\Spo{\*C} \setminus \Pa{\*C}}.
    \end{align}
    The uniqueness of \(\*S^+\) follows from Lemma~\ref{lemma:equivalence:mbsplus}. Note that \(\*S^+\) depends only on \(\*C\), not the particular \(\*S\) such that \(\&{C}(X)_{\G_{\*S}} = \*C\).
\end{proof}

\begin{proof}[Proof of Lemma~\ref{lemma:equivalence:cmb}]
    First, we show the mapping \(f\) is well-defined. Given \(\*C \in \&{AC}_X\), by definition, there exists an ancestral set \(\*S \in \&{S}_X\) such that \(\*C = \&{C}(X)_{\G_\*S}\). Then, \(f(\*C) = \Pa{\*C} \setminus \{X\} = \mkbk{X,\*S}\) by Prop.~\ref{prop:equivalence:cmb}, so \(f(\*C) \in \&{Z}_X\) holds.

    Next, we show that \(f\) is bijective by exhibiting an inverse \(g: \&{Z}_X \to \&{AC}_X\). Given \(\*Z \in \&{Z}_{X}\), fix any \(\*S \in \&{S}_X\) such that \(\*Z = \mkbk{X,\*S}\) (we know such \(\*S\) exists by the definition of \(\*Z\)) and let \(g(\*Z) = \&C(X)_{\G_\*S}\).
    
    To see that \(g\) is well-defined, first note that \(\&C(X)_{\G_\*S} \in \&{AC}_X\) since \(\*S\) is an ancestral set by assumption. Second, we need to show that \(g(\*Z)\) is independent of the particular choice of \(\*S\) (since multiple ancestral sets can induce the same MB). Consider \(\*S_1, \*S_2 \in \&{S}_X\) such that \(\*Z = \mkbk{X,\*S_1} = \mkbk{X,\*S_2}\). Let \(\*C_1 =\&{C}(X)_{\G_{\*S_1}}\) and \(\*C_2 = \&{C}(X)_{\G_{\*S_2}}\).
    
    We show that \(\*C_1 = \*C_2\).
    By Prop.~\ref{prop:equivalence:cmb}, \(\mkbk{X,\*S_1} = \Pa{\*C_1} \setminus \{X\}\) and \(\mkbk{X,\*S_2} = \Pa{\*C_2} \setminus \{X\}\).
    From the equality \(\mkbk{X,\*S_1} = \mkbk{X,\*S_2}\), we have \(\Pa{\*C_1} \setminus \{X\} = \Pa{\*C_2} \setminus \{X\}\) and hence \(\Pa{\*C_1}  = \Pa{\*C_2}\).
    Since \(\*S_1, \*S_2\) are ancestral, we have \(\Pa{\*C_1}  \subseteq \*S_1\), and \(\Pa{\*C_2} \subseteq \*S_2\).
    With \(\Pa{\*C_1} \setminus \{X\} = \Pa{\*C_2} \setminus \{X\}\), we have \(\&{C}(X)_{\G_{\*S_1}} = \*C_1 \subseteq \Pa{\*C_1}  = \Pa{\*C_2}  \subseteq \*S_2\),
    implying \(\&{C}(X)_{\G_{\*S_1}} \subseteq \&{C}(X)_{\G_{\*S_2}}\). By a symmetric argument, we get \(\&{C}(X)_{\G_{\*S_2}} \subseteq \&{C}(X)_{\G_{\*S_1}}\).
    Therefore, \(\*C_1 = \*C_2\).
    
    Finally, we show that \(g\) is a two-sided inverse of \(f\). Given \(\*C \in \&{AC}_X\), fix some \(\*S \in \&{S}_X\) such that \(\*C = \&{C}(X)_{\G_\*S}\). Then 
    \begin{align}
        g(f(\*C)) &= g(\Pa{\*C} \setminus \{X\}) \\
                  &= g(\mkbk{X,\*S}) \tag{Prop.~\ref{prop:equivalence:cmb}} \\
                  &= \&{C}(X)_{\G_\*S} \tag{By definition of \(g\)} \\
                  & = \*C
    \end{align}
    Given \(\*Z \in \&{Z}_X\), fix some \(\*S \in \&{S}_X\) such that \(\*Z = \mkbk{X,\*S}\). Then,
    \begin{align}
        f(g(\*Z)) &= f(\&{C}(X)_{\G_\*S}) \\
        & = \Pa{\&{C}(X)_{\G_\*S}} \setminus \{X\} \\
        & = \Pa{\&{C}(X)_{\G_\*S}}_{\G_{\*S}} \setminus \{X\} \tag{\(\*S\) is ancestral and \(\&{C}(X)_{\G_\*S} \subseteq \*S\)}\\
        & = \mkbk{X,\*S} \\
        & = \*Z
    \end{align}
\end{proof}

\begin{proof}[Proof of Cor.~\ref{cor:equivalence:csplus}]
    The result follows from Lemma~\ref{lemma:equivalence:mbsplus} and Lemma~\ref{lemma:equivalence:cmb}, composing the bijective mappings \(f_1:\&{AC}_X \to \&{Z}_X\) and \(f_2: \&{Z}_X \to \&{S^{+}}_X\).
\end{proof}

\begin{figure}[t]
    \centering
    \includegraphics[height=.25\textheight]{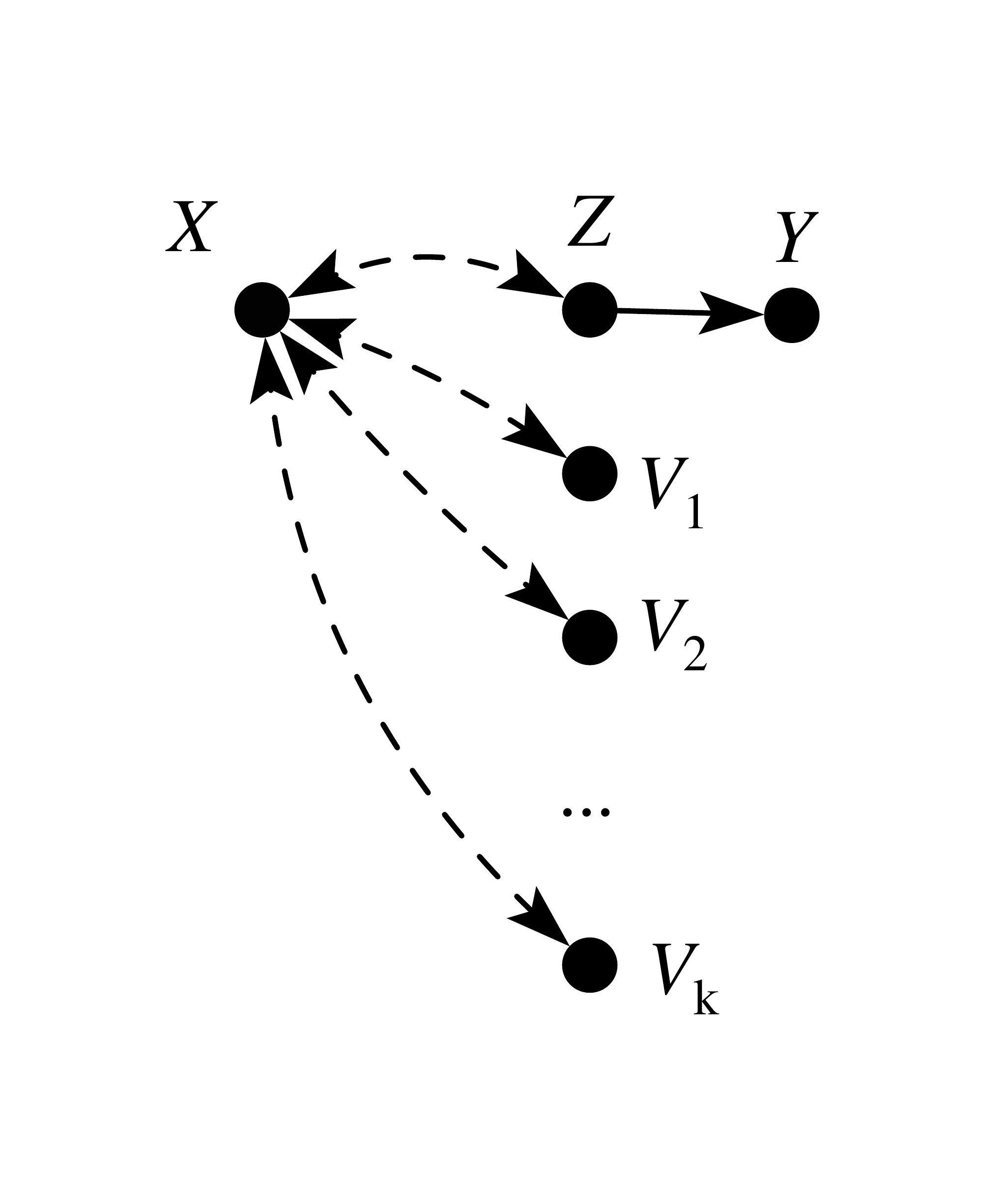}
    \caption{A causal graph \(\G\) with consistent ordering \(Z \prec Y \prec X \prec V_1 \prec \dots \prec V_k\), inducing \(\Omega(2^n)\) number of CIs invoked by C-LMP.}
    \label{fig:lmpsize}
\end{figure}

\section{Proofs}
\label{appendix:proofs}

\subsection{Section~\ref{sec:reformulation} Proofs}
\label{sec:reformulation_proof}

\begin{adxprop}
\label{adxprop:clmpsep}
    Let \(\G\) be a causal graph and \(\*V^\prec\) a consistent ordering. For any variable \(X \in \*V^\prec\) and any ancestral c-component \(\*C \in \&{AC}_X\) relative to \(X\),
    \[
        X \perp_d \*V^{\leq X} \setminus (\De{\Spo{\*C} \setminus \Pa{\*C}} \cup \Pa{\*C}) \mid \Pa{\*C} \setminus \{X\}.
    \]
\end{adxprop}
\begin{proof}
    Since \(\perp_d\) satisfies the composition and decomposition axioms, it suffices to show that \(X \perp_d \{Y\} \mid \Pa{\*C}\setminus\{X\}\) for every \(Y \in  \*V^{\leq X }\setminus (\De{\Spo{\*C} \setminus \Pa{\*C}} \cup \Pa{\*C})\). 
    
    Take a variable \(Y \in \*V^{\leq X }\setminus (\De{\Spo{\*C} \setminus \Pa{\*C}} \cup \Pa{\*C})\) and some path \(\pi = (X,V_1,V_2,\dots,V_n,Y)\) between \(X\) and \(Y\) in \(\G\) for \(n \geq 1\) (note that \(X,Y\) are non-adjacent by assumption).
    Let \(\pi' = (X,V_1,\dots,V_k)\) (with \(k \leq n\)) denote the longest sub-path of \(\pi\) starting from \(X\), not including \(Y\), that contains only bidirected edges.
    If \(\pi'= \emptyset\), then \(V_1 \prec X \implies V_1 \in \Pa{\{X\}} \subseteq \Pa{\*C}\), hence \(\pi\) is blocked.
    Otherwise, if for some \(i \in [k]\), \(V_i \not \in \An{\*C}\), then \(V_i\) blocks \(\pi\). If every \(V_i \in \An{\*C}\), since \(\*C\) is an AC, the existence of \(\pi'\) implies \(V_i \in \*C\). Then, consider the subpath of \(\pi\) from \(V_k\) to \(Y\). Note that \(V_k, Y\) are non-adjacent since \(Y \not \in \De{\Spo{\*C} \setminus \Pa{\*C}} \cup \Pa{\*C}\). The sub-path has either \(V_k \leftarrow V_{k+1} \circ - \circ\) or \(V_k \rightarrow V_{k+1} \rightarrow\), both of which are blocked by \(\Pa{\*C} \setminus \{X\}\). Therefore, \(\pi\) is blocked, and \(X \perp_d \{Y\} \mid \Pa{\*C}\setminus\{X\}\).
\end{proof}

\begin{customthm}{\ref{thm:equivalence:gmp:lmpplus}}[Equivalence of C-LMP and GMP]
    Let \(\G\) be a causal graph and \(\*V^\prec\) a consistent ordering.
    A probability distribution over \(\*V\) satisfies the global Markov property for \(\G\) if and only if it satisfies the c-component local Markov property for \(\G\) with respect to \(\*V^{\prec}\).
\end{customthm}
\begin{proof}
The proof is similar to that of \citep[Prop.~4]{lauritzen:etal90} and \citep[Thm.~2]{richardson2003markov}, which is based on the former.

(\(\implies\)) Prop.~\ref{adxprop:clmpsep} shows that the CIs invoked by C-LMP are a subset of those invoked by GMP. Therefore, if a probability distribution \(P(\*v)\) satisfies the GMP for a given DAG \(\G\), it necessarily satisfies the C-LMP for \(\G\) (with respect to any given ordering).

(\(\impliedby\)) Next, we show that if a probability distribution \(P(\*v)\) satisfies the C-LMP for a given DAG \(\G\) with respect to a given ordering, it necessarily satisfies the GMP for \(\G\).
We show the other direction by induction on the number of nodes. Let \(I_k\) be the statement that for a graph \(\G\) on \(k\) nodes, if a distribution \(P(\*v)\) satisfies the C-LMP for \(\G\), then it satisfies the GMP for \(\G\). The base case is trivial. Assume for some \(k\) that \(I_j\) is true for all \(j \leq k\). We will show this implies \(I_{k+1}\). Fix a graph \(\G\) with \(k+1\) nodes, a consistent ordering \(\prec\), and a distribution \(P(\*v)\) which satisfies the C-LMP for \(\G\) with respect to \(\prec\). Consider a $d$-separation \(\*X \perp_d \*Y \mid \*Z\) in \(\G\) for disjoint sets \(\*X,\*Y,\*Z\). We need to show that \(\*X \indep \*Y \mid \*Z\) in \(P(\*v)\). 

We claim we can assume, without loss of generality, that \(\*X \cup\*Y \cup \*Z = \*V\).
First, we show how we can assume \(\An{\*X \cup \*Y \cup \*Z} = \*V\). Consider \(\G' = \G_{\An{\*X \cup \*Y \cup \*Z}}\), and let \(\prec'\) be the ordering \(\prec\) but removing variables in \(\*V \setminus \An{\*X \cup \*Y \cup \*Z}\). Let \(\*A =  \An{\*X \cup \*Y \cup \*Z}_\G\), so that \(P(\*a) = \sum_{\*v\setminus \*a}P(\*v)\).  Since \(\G'\) is a subgraph on an ancestral set,
any AC in \(\G'\) is an AC in \(\G\), it is easy to see that if \(P(\*v)\) satisfies the C-LMP for \(\G\) with respect to \(\prec\), then \(P(\*a)\) satisfies the C-LMP for \(\G\) with respect to \(\prec'\).  Since \(\*X \perp_d \*Y \mid \*Z\) in \(\G\), and \(\G'\) contains no more edges than \(\G\), we also have \(\*X \perp_d \*Y \mid \*Z\) in \(\G'\). By the inductive assumption for \(\G'\), we have \(\*X \indep \*Y \mid \*Z\) in \(P(\*a)\), which implies  \(\*X \indep \*Y \mid \*Z\) in \(P(\*v)\). Finally, we can extend \(\*X,\*Y\) so that \(\An{\*X \cup \*Y \cup \*Z} = \*X \cup \*Y \cup \*Z\) (and reduce to the original separation statement using the decomposition axiom). For any \(V \in \An{\*X \cup \*Y \cup \*Z} \setminus \*X \cup \*Y \cup \*Z\), either \(V \perp_d \*Y \mid \*Z\) or \(V \perp_d \*X \mid \*Z\). Towards contradiction, assume \(V\) has an active path \(\pi_x\) to some node in \(\*X\)  and an active path \(\pi_y\) to some node in \(\*Y\) when conditioning on \(\*Z\). Then, adjoining \(\pi = \pi_x \cup \pi_y\) gives an active path between \(\*X\) and \(\*Y\) unless \(V\) is an inactive collider on this path. However, if \(V \in \An{\*X}\), the path \(V \rightsquigarrow X\) to the descendant node \(X \in \*X\), adjoined with \(\pi_y\), gives an active path between \(\*X\) and \(\*Y\) unless we condition on some descendant of \(V\); the same applies if \(V \in \An{\*Y}\); and if \(V \in \An{\*Z}\), clearly, \(V\) is active when conditioning on \(\*Z\). We thus arrive at a contradiction.

Now, consider a separation \(\*X \perp_d \*Y \mid \*Z\) in \(\G\) such that \(\*X \cup \*Y \cup \*Z = \*V\). We need to show that \(\*X \indep \*Y \mid \*Z\) in \(P(\*v)\). Let \(V^*\) be the final node in the ordering \(\prec\) so that \(\*V^{\leq V^*}= \*V\). Since \(V \in \*X \cup \*Y \cup \*Z\), there are three cases to consider:

\begin{enumerate}
    \item \(V^* \in \*X\). 
    
    Since \(\*X \perp_d \*Y \mid \*Z\) in \(\G\), we have \(\*X\setminus\{V^*\} \perp_d \*Y \mid \*Z\) in \(\G_{\*V \setminus \{V^*\}}\). Since  \(\G_{\*V \setminus \{V^*\}}\) is ancestral, we apply a similar argument as in justifying the assumption that \(\*X \cup\*Y \cup \*Z = \*V\) to get, by the inductive assumption for \(\G'\), that 
    \[
    \*X\setminus\{V^*\} \indep \*Y \mid \*Z \text{ in } P(\*v).
    \]
    Let \(\*C = \&{C}(V^*)_{\G_{\An{\*X \cup \*Z}}}\). By C-LMP, we have \(V^* \indep \*V \setminus (\De{\Spo{\*C} \setminus \Pa{\*C}} \cup  \Pa{\*C}) \mid \Pa{\*C} \setminus \{V^*\}\).

    First, note that \(\*Y \cap \Pa{\*C} = \emptyset\). Towards contradiction, assume that for some \(Y \in \*Y\), there is a path \(\pi: Y \circ \rightarrow V_1 \leftrightarrow V_{2} \leftrightarrow \dots V_{k} \leftrightarrow V_{k+1} = V^*\) such that each \(V_i \in \*C \subseteq \An{\*X \cup \*Z}\). By induction on \(i \in [k+1]\), we show that \(\pi\) is active when conditioning on \(\*Z\).
    For the base case, clearly, the subpath \(Y \circ \to V_1\) of \(\pi\) is active when conditioned on \(\*Z\).
    Assume that, for some \(i \in [k+1]\), the sub-path of \(\pi\) from \(Y\) to \(V_i\) is active.
    Consider the inductive step. If \(V_i = V^*\), we are done. Otherwise, if \(V_i \in \An{\*Z}\), then \(V_i\) is active in \(\pi\) when conditioning on \(\*Z\).  If \(V_i \in \An{\*X}\), then from the inductive assumption, there is a path from \(Y\) to \(V_i\) plus a path \(V_i \rightsquigarrow X'\) to some \(X' \in \*X\) which is only blocked if \(\De{\{V_i\}} \cap \*Z \neq \emptyset\). This again implies that \(V_1\) is active in \(\pi\) when conditioning on \(\*Z\). In either case, the subpath of \(\pi\) from \(Y\) to \(V_{i+1}\) is active. This contradicts \(\*X \perp_d \*Y \mid \*Z\) in \(\G\). Therefore, we can conclude \(\*Y \cap \Pa{\*C} = \emptyset\).

    Second, note that \(\*Y \cap \De{\Spo{\*C} \setminus \Pa{\*C}} = \emptyset\). This is because \(\Spo{\*C} \setminus \Pa{\*C} = \emptyset\). For any \(U \in \Spo{\*C}, U \in \*X \cup \*Z \implies U \in \*C\) by definition of \(\*C\). Therefore, \(U \in \Spo{\*C} \setminus \Pa{\*C} \implies U \in \*Y = \*V \setminus \*X \cup \*Z\). However, for such \(U\), there is a path \(\pi: U \leftrightarrow V_1 \leftrightarrow V_{2} \leftrightarrow \dots V_{k} \leftrightarrow V_{k+1} = V^*\) with each \(V_i \in \An{\*X \cup \*Z}\). By a similar induction as for the claim \(\*Y \cap \Pa{\*C} = \emptyset\), we can show that \(\pi\) is active when conditioning on \(\*Z\), which contradicts \(\*X \perp_d \*Y \mid \*Z\). 

    We return to the CI statement 
    \[
    V^* \indep \*V \setminus (\De{\Spo{\*C} \setminus \Pa{\*C}} \cup  \Pa{\*C}) \mid \Pa{\*C} \setminus \{V^*\}.\]
    Since \(\*V = \*X \cup \*Y \cup \*Z\) by assumption and \(\*Y \cap \Pa{\*C}= \*Y \cap \De{\Spo{\*C} \setminus \Pa{\*C}} = \emptyset\), we can simplify this statement to 
    \begin{align*}
        V^* \indep & \*Y \cup ((\*X \cup \*Z) \setminus (\De{\Spo{\*C} \setminus \Pa{\*C}} \cup  \Pa{\*C}))\\
        & \mid \Pa{\*C} \setminus \{V^*\}.
    \end{align*}
    For any variable \(W \in \*X \cup \*Z\), \(W \in \De{\Spo{\*C}}\) implies there is some variable \(B \in \Spo{\*C} \cap \An{\{W\}} \subseteq \Spo{\*C} \cap \An{\*X \cup \*Z}\), hence \(B \in \*C\). Therefore, \(\De{\Spo{\*C} \setminus \Pa{\*C}} = \emptyset\). This further implies
    \[
    V^* \indep \*Y \cup ((\*X \cup \*Z) \setminus  \Pa{\*C}) \mid \Pa{\*C} \setminus \{V^*\}
    \]
    By the weak union axiom, we get 
    \[
    V^* \indep \*Y  \mid (\*X \setminus \{V^*\}) \cup \*Z
    \]
    Applying the contraction axiom to \(\*X\setminus\{V^*\} \indep \*Y \mid \*Z\) and \(V^* \indep \*Y  \mid (\*X \setminus \{V^*\}) \cup \*Z\) gives \(\*X \indep \*Y \mid \*Z\) in \(P(\*v)\).
    
    \item \(V^* \in \*Y\). This is similar to the case \(V \in \*X\) (switching \(\*X, \*Y\) in the proof).
    \item \(V^* \in \*Z \). Since \(\*X \perp_d \*Y \mid \*Z\) in \(\G\), we have \(\*X \perp_d \*Y \mid \*Z \setminus \{V^*\}\) in \(\G_{\*V \setminus \{V^*\}}\).
    Since \(\G_{\*V \setminus \{V^*\}}\) is a subgraph on an ancestral set,
    we apply a similar argument as in justifying the assumption that \(\*X \cup\*Y \cup \*Z = \*V\) to get, by the inductive assumption for \(\G'\), that 
    \[
        \*X \indep \*Y \mid \*Z \setminus\{V^*\} \text{ in } P(\*v).
    \]
    Let \(\*C = \&{C}(V^*)_\G\). By C-LMP, we have \(V^* \indep \*V \setminus (\De{\Spo{\*C} \setminus \Pa{\*C}} \cup  \Pa{\*C}) \mid \Pa{\*C} \setminus \{V^*\}\).

    First, we show that either \(\Pa{\*C} \cap \*Y = \emptyset\) or \(\Pa{\*C} \cap \*X = \emptyset\). Assume, toward contradiction, that \(\Pa{\*C} \cap \*Y \neq \emptyset\) and \(\Pa{\*C} \cap \*X \neq \emptyset\).
    Then, there are variables \(Y \in \*Y, X\in \*X\) and a path \(\pi: Y \circ \rightarrow V_1 \leftrightarrow V_2 \leftrightarrow \dots V_k \leftarrow \circ V_{k+1} = X\) for some \(k\geq 0\). Let \(\pi'\) be a subpath of \(\pi\) such that one endpoint node of \(\pi'\) is in \(\*X\), the other endpoint node in \(\*Y\), and all intermediate nodes (if any) are in \(\*Z\). It is easy to see \(\pi'\) must exist since \(\pi' = \pi\) if for each \(i \in [k]\), we have \(V_i \in \*Z\); otherwise, we can construct \(\pi'\) by removing variables from \(\pi\). Then, \(\pi'\) is active when conditioning on \(\*Z\), which contradicts \(\*X \perp_d \*Y \mid \*Z\).

    Moreover, \(\Spo{\*C} \setminus \Pa{\*C} = \emptyset\) because \(\*C\) is defined over \(\*V = \*X \cup \*Y \cup \*Z\). Thus, \(\*X \cap \De{\Spo{\*C} \setminus \Pa{\*C}} = \*Y \cap \De{\Spo{\*C} \setminus \Pa{\*C}} = \emptyset\).

    Return to the CI statement: \(V^* \indep \*V \setminus (\De{\Spo{\*C} \setminus \Pa{\*C}} \cup  \Pa{\*C}) \mid \Pa{\*C} \setminus \{V^*\}\).
    If \(\Pa{\*C} \cap \*Y = \emptyset\), this simplifies to \(V^* \indep \*Y \mid \*X \cup (\*Z \setminus \{V^*\})\) by an argument similar to Case (1). The contraction axiom applied to  \(V^* \indep \*Y \mid \*X \cup (\*Z \setminus \{V^*\})\) and \(\*X \indep \*Y \mid \*Z \setminus \{V^*\}\) gives \(\*X \cup \{V^*\} \indep \*Y  \mid \*Z \setminus \{V^*\}\). Applying the weak union axiom to this last CI, we get \(\*X \indep \*Y \mid \*Z\). A similar argument applies if  \(\Pa{\*C} \cap \*X = \emptyset\).
    \end{enumerate}

\end{proof}

\begin{customcor}{\ref{cor:equivalence:lmp:lmpplus}}[Equivalence of C-LMP and the Ordered Local Markov Property]
\label{cor:equivalence:gmp:lmpplus}
    Let \(\G\) be a causal graph and \(\*V^\prec\) a consistent ordering.
    A probability distribution over \(\*V\) satisfies the ordered local Markov property for \(\G\) with respect to  \(\*V^\prec\) if and only if it satisfies the c-component local Markov property for \(\G\) with respect to \(\*V^{\prec}\).
\end{customcor}
\begin{proof}
    By Thm.~\ref{thm:equivalence:gmp:lmpplus}, a probability distribution \(P(\*v)\) over \(\*V\) satisfies the C-LMP for \(\G\) with respect to \(\*V^\prec\) if and only if it satisfies the GMP for \(\G\). By \citep[Thm.~2, Section~3.1]{richardson2003markov}, a probability distribution \(P(\*v)\) over \(\*V\) satisfies the ordered local Markov property for \(\G\) with respect to \(\*V^\prec\) if and only if it satisfies the GMP for \(\G\).
\end{proof}

\begin{customthm}{\ref{thm:equivalence:clmpplus}}[Unique AC for each CI Invoked by C-LMP]
\label{adxcor:equivalence:clmpplus}
    Let \(\G\) be a causal graph, \(\*V^\prec\) a consistent ordering, and \(X\) a variable in \(\*V^\prec\).
    For every conditional independence relation invoked by the c-component local Markov property of the form \(X \indep \*W \mid \*Z\), there is exactly one ancestral c-component \(\*C \in \&{AC}_X\) such that \(\*W = \*V^{\leq X} \setminus ((\De{\Spo{\*C} \setminus \Pa{\*C}}) \cup \Pa{\*C})\) and \(\*Z = \Pa{\*C} \setminus \{X\}\).
\end{customthm}

\begin{proof}
    The result follows from Def.~\ref{def:lmpplus}, Lemma~\ref{lemma:equivalence:cmb}, and Cor.~\ref{cor:equivalence:csplus}. 
\end{proof}

\begin{customprop}{~\ref{prop:lmpsize}}[Number of CIs Invoked by C-LMP]
    Given a causal graph \(\G\) and a consistent ordering  \(\*V^\prec\), let \(n\) and \(s \leq n\) denote the number of variables and the size of the largest c-component in \(\G\) respectively. Then, the c-component local Markov property for \(\G\) with respect to \(\*V^\prec\) invokes \(O(n 2^s)\) conditional independencies implied by \(\G\) over \(\*V\).
    Moreover, there exists a graph \(\G\) and a consistent ordering  \(\*V^\prec\)  for which the property induces \(\Omega(2^n)\) conditional independencies.
\end{customprop}

\begin{proof}
    By Def.~\ref{def:lmpplus}, the set of CIs invoked by C-LMP for a variable \(X \in \*V^\prec\) is in bijection with the set of ACs, \(\&{AC}_X\).
    Therefore, it suffices to bound \(|\&{AC}_X|\). Recall that \(\&{AC}_X \subseteq \mathcal{P}(\&C(X)_\G)\) (where \(\mathcal{P}(\cdot)\) denotes the power-set operation). Then, \(|\&C(X)_\G| \leq s \implies |\mathcal{P}(\&C(X)_\G)| \leq 2^s \implies |\&{AC}_X| \leq 2^s.\) Total number of CIs \(k\) invoked by C-LMP for all variables is thus \(k \leq n 2^s \in O(n 2^s)\).

    Next, consider the graph \(\G\) shown in Fig.~\ref{fig:lmpsize} for which C-LMP invokes \(\Omega(2^n)\) CIs.

    Fix \(V_i, i \in [k]\). For each \(\*C \subseteq \{V_j\}_{j < i}\), we get an AAC \(\{V_i,X,Z\} \cup \*C\) relative to \(V_i\) inducing the CI: \(V_i \indep \{Y\} \mid \{X,Z\} \cup \*C\) (The definition of admissibility of AC is given by Def.~\ref{def:validci}). There are \(2^{i-1}\) such CIs for each \(V_i\). Then, the total number of CIs across all such \(V_i\) is 
    \[
        \sum_{i=1}^k 2^{i-1} = 2^{k-1} - 1 = 2^{n-4} - 1 \in \Omega(2^n) 
    \]
    Since, for any \(\G\), we have that \(s \leq n\), the upper bound \(O(n 2^s)\) is thus tight ignoring the linear term in \(n\).
\end{proof}

\subsection{Section~\ref{section:listci} Proofs}
\label{sec:listci_proof}

\textbf{Notation.} For the proofs in this section, given a causal graph \(\G\) defined on a set of variables \(\*V\), and variables \(X, Y \in \*V\), we use \(X \sim Y\) to denote an arbitrary path (possibly of length 0, when \(X = Y\)) between \(X\) and \(Y\) in \(\G\); \(X \rightsquigarrow Y\) to denote a directed path (possibly of length 0, when \(X = Y\)) from \(X\) to \(Y\) in \(\G\); and \(X \circ \rightarrow Y\) to denote that there is either an edge \(X \rightarrow Y\) or \(X \leftrightarrow Y\) in \(\G\).

\begin{adxprop}[Time Complexity of Computing a C-component]~\label{prop:constc}
     Given a causal graph \(\G\) over a set of variables \(\*V\) and a variable \(X \in \*V\), the c-component \(\&{C}(X)_\G\) containing \(X\) in \(\G\) is computable in time \(O(n+m)\), where \(n\) and \(m\) are the numbers of nodes and edges in \(\G\) respectively.
\end{adxprop}
\begin{proof}
    Using breadth-first search (BFS), compute the set of nodes reachable from the starting node \(X\) by following only bidirected edges. This takes time \(O(n+m)\), the complexity of BFS.
\end{proof}

\begin{figure}[t]
\begin{algorithmic}[1]
\Function {IsAdmissible} {$\G_{\*V^{\leq X}}, X, \*V^{\leq X}, \*C$}

    \State {\bfseries Output:} True if a given AC \(\*C\) relative to \(X\) is admissible; False otherwise.

    \State \(\*S^+ \gets \*V^{\leq X} \setminus \De{ \Spo{\*C} \setminus \Pa{\*C} } \)
    \State \(\*W \gets \*S^+ \setminus \Pa{\*C}\)

    \State \textbf{if} \(\*W \neq \emptyset\) \textbf{then}
    \Indent
        \State \textbf{return} True
    \EndIndent
    \State \textbf{else}
    \Indent
        \State \textbf{return} False
    \EndIndent
    
\EndFunction
\end{algorithmic}
\caption{A function that checks if a given AC is admissible.}
\label{func:isadmissible}
\end{figure}

\begin{adxlemma}[Correctness of \textsc{IsAdmissible}]
\label{adxlemma:isadmissible}
    Given a causal graph \(\G_{\*V^{\leq X}}\), a variable \(X\), and a set of variables \(\*V^{\leq X}\), let \(\*C\) be an ancestral c-component relative to \(X\).
    Then, \textsc{IsAdmissible} returns True if \(\*C\) is admissible, and False otherwise.
    \textsc{IsAdmissible} takes \(O(n+m)\) time where \(n\) and \(m\) represent the number of nodes and edges in \(\G\), respectively.
\end{adxlemma}

\begin{proof}
    For correctness, it immediately follows from Def.~\ref{def:lmpplus} and Def.~\ref{def:validci}.
    
    \textsc{IsAdmissible} runs in \(O(n+m)\) time since the construction of the sets \(\*S^+\) and \(\*W\) takes \(O(n+m)\) time for each set.
\end{proof}

\begin{adxlemma}[Existence of a Separator]
\label{lemma:existsep}
    Given a causal graph \(\G\), let \(\*I, \*R, \*X, \*Y\) be sets of nodes with \(\*I \subseteq \*R\) and \(\*R \cap (\*X \cup \*Y) = \emptyset\). If there exists a set \(\*Z_0\) separating \(\*X\) and \(\*Y\) in \(\G\) such that \(\*{I} \subseteq \*{Z_0} \subseteq \*{R}\), then \(\*Z = \An{\*X \cup \*Y \cup \*I}_\G \cap \*R\) is such a set.
\end{adxlemma}

\begin{proof}
    Assume there exists \(\*Z_0\) separating \(\*X,\*Y\) such that \(\*{I} \subseteq \*{Z_0} \subseteq \*{R}\). For some \(X \in \*X\), \(Y \in \*Y\), consider a path \(\pi\) from \(X\) to \(Y\) in \(\G\), consisting of nodes \(\{X = V_0, V_1, \dots, V_n, V_{n+1} = Y\}\) where \(V_i, V_i+1\) are adjacent in \(\G\) for \(0 \leq i \leq n\). Note that we must have \(n \geq 1\); otherwise, \(X,Y\) are adjacent and cannot be separated. 
    
    If none of the variables \(V_i, i \in [n]\) is a collider, then each \(V_i\) must be in \(\An{\{X, Y\}}\).  If  \(V_i \not \in \*R\) for any \(i \in [n]\), then  \(V_i \not \in \*{Z_0}\subseteq \*R\) for any \(i \in [n]\), and hence \(\*Z_0\) does not block \(\pi\), which is a contradiction. Therefore, there exists \(V_i\) such that \(V_i \in \An{\*X \cup \*Y \cup \*I}_\G \cap \*R = \*Z\) and hence \(\*Z\) blocks  \(\pi\). 
    
    If some \(V_i\)  is a collider, let \(\*C = \{C_{i_1}, \dots, C_{i_k}\} \subseteq \{V_1, \dots, V_n\}\) denote the set of colliders on \(\pi\) such that \(i_j < i_{j+1}\) for \(1 \leq j \leq k-1\). If there is a variable \(C \in \*C\) such that \(\*Z \cap \De{\{C\}}_\G = \emptyset\)  (in other words, \(\*Z\) does not contain \(C\) or any of its descendants), then \(\*Z\) blocks \(\pi\) due to the inactive collider \(C\). Otherwise, consider the case that \(\*C \subseteq \An{\*Z}_\G\)  i.e. for every \(C \in \*C\), either \(C\) is in \(\*Z\) or some descendant of \(C\) is in \(\*Z\), and  hence \(C\) is active (when conditioning on \(\*Z\). Since \(\*Z = \An{\*X \cup \*Y \cup \*I}_\G \cap \*R \subseteq  \An{\*X \cup \*Y \cup \*I}_\G\), we have \(\*C \subseteq \An{\*Z}_\G \implies \An{\*C}_\G \subseteq \An{\*X \cup \*Y \cup \*I}_\G\). For any \(V_i\) in \(\pi\), either \(V_i \in \*C\) or  \(V_i \in \An{\{X,Y\}}\) or \(V_i \in \An{\*C}\); therefore, \(\{V_i\}_{1:n} \subseteq \An{\{X,Y\} \cup \*C}_\G \subseteq \An{\*X \cup \*Y \cup \*I}_\G\). Hence, \(\pi\) is blocked by \(\*Z\) unless every \(V_i \in \*R\) is a collider; that is, \(\{V_i\}_{1:n} \cap \*R \subseteq \*C\). Assume toward contradiction that  \(\{V_i\}_{1:n} \cap \*R \subseteq \*C\).
    
    We show, by induction on the index \(i_j, j \in [k]\) of \(\*C\), that for every \(j \in [k]\), there exists a variable \(X_0 \in \*X\) such that there is an active \(X_0 \sim C_{i_k} \leftarrow \circ V_{i_j + 1}\) path when conditioning on \(\*{Z_0}\).

    \emph{Base case}. Consider \(C_{i_1} \in \An{\*X \cup \*Y \cup \*I}_\G\). The sub-path of \(\pi\) from \(X\) to \(C_{i_1}\) is unblocked by \(\*{Z_0} \subseteq \*R\). This is because for any node \(V\) on this sub-path (excluding \(X\) and \(C_{i_1}\)), \(V \not \in \*C\) by assumption and  \(\{V_i\}_{1:n} \cap \*R \subseteq \*C \implies V \not \in \*R \implies V \not \in \*Z_0\). 
    
    \begin{itemize}
        \item If \(C_{i_1} \in \An{\*Y}_\G\), there is a directed path \(\pi'\) from \(C_{i_1}\) to \(Y'\) for some \(Y' \in \*Y\). The \(X \sim C_{i_1}\) sub-path of \(\pi\) (which is unblocked by \(\*Z_0\)), adjoined with \(\pi'\), gives an active path from \(X \in \*X\) to \(Y' \in \*Y\). For \(\*Z_0\) to block this path, it must block \(\pi'\). Hence, \(\*{Z_0}\) contains a descendant of \(C_{i_1}\) and \(C_{i_1}\) is active when conditioning on \(\*{Z_0}\), giving an active sub-path of \(\pi\), \(X \sim V_{i_1-1} \circ \rightarrow C_{i_1} \leftarrow \circ V_{i_1 + 1}\).
       \item If \(C_{i_1} \in \An{\*X}_\G\), there is a directed path \(\pi'\) from \(C_{i_1}\) to \(X'\) for some \(X' \in \*X\). If \(\*{Z_0}\) contains a descendant of \(C_{i_1}\), then \(C_{i_1}\) is active when conditioning on \(\*{Z_0}\). Therefore, the sub-path of \(\pi\), \(X  \sim V_{i_1-1} \circ \rightarrow C_{i_1} \leftarrow \circ V_{i_1 + 1}\) is active. If \(\*{Z_0}\) contains no descendants of \(C_{i_1}\), then \(\pi'\) is unblocked by \(\*{Z_0}\), giving an active \(X' \leftsquigarrow C_{i_1} \leftarrow \circ V_{i_1 + 1}\) path. 
        \item If \(C_{i_1} \in \An{\*I}_\G\), since \(\*I \subseteq \*{Z_0}\), we condition on a descendant of \(C_{i_1}\) and \(C_{i_1}\) is active,  giving an active sub-path of \(\pi\), \(X \sim  V_{i_1-1} \circ \rightarrow \C_{i_1} \leftarrow \circ V_{i_1 + 1}\). 
    \end{itemize}

    \emph{Inductive assumption.} Assume for some \(j \in [k-1]\), there is an active \(X_0 \sim\*C_{i_j} \leftarrow \circ V_{i_j + 1}\) path for some \(X_0 \in \*X\).

    \emph{Inductive step.} We show this implies the existence of an active \(X_0' \sim\*C_{i_{j+1}} \leftarrow \circ V_{i_{j+1} + 1}\) path for some \(X_0' \in \*X\).  Note that the sub-path of \(\pi\) from \(C_{i_j}\) to \(C_{i_{j+1}}\) is unblocked by \(\*{Z_0}\). This is because for any node \(V\) on this sub-path (excluding \(C_j\) and \(C_{j+1}\)), \(V \not \in \*C\) by assumption and  \(\{V_i\}_{1:n} \cap \*R \subseteq \*C \implies V \not \in \*R \implies V \not \in \*Z_0\).
    \begin{itemize}
        \item If \(C_{i_{j+1}} \in \An{\*Y}_\G\), there is a directed path \(\pi': C_{i_{j+1}} \rightsquigarrow Y'\) for some \(Y' \in \*Y\). By the inductive assumption, we get an active path \(X_0 \sim\*C_{i_j} \leftarrow \circ V_{i_j + 1} \circ \rightarrow C_{i_{j+1}}\rightsquigarrow Y'\). For \(\*{Z_0}\) to block this path, it must block \(\pi'\). Hence, \(\*{Z_0}\) contains a descendant of \(C_{i_{j+1}}\) and \(C_{i_{j+1}} \) is active when conditioning on \(\*{Z_0}\), giving an active 
        \(X_0 \sim C_{i_{j}} \leftarrow \circ V_{i_j + 1} \sim  V_{i_{j+1} - 1} \circ \rightarrow C_{i_{j+1}} \leftarrow \circ V_{i_{j+1} + 1}\) path.
        \item If \(C_{i_{j+1}} \in \An{\*X}_\G\), there is a directed path \(\pi'\) from \(C_{i_{j+1}}\) to \(X'\) for some \(X' \in \*X\). If  \(\*{Z_0}\) contains a descendant of \(C_{i_{j+1}}\), then \(C_{i_{j+1}}\) is active when conditioning on \(\*{Z_0}\). Using the inductive assumption, we get an active path  \(X_0 \sim C_{i_{j}} \leftarrow \circ V_{i_j + 1} \sim   V_{i_{j+1} - 1} \circ \rightarrow C_{i_{j+1}} \leftarrow \circ V_{i_{j+1} + 1} \).  If  \(\*{Z_0}\) contains no descendants of \(C_{i_{j+1}}\), then \(\pi'\) is unblocked by \(\*{Z_0}\), giving an active path \(X' \leftsquigarrow C_{i_{j+1}} \leftarrow \circ V_{i_{j+1} + 1}\) path.
        \item If \(C_{i_{j+1}}  \in \An{\*I}_\G\), since \(\*I \subseteq \*{Z_0}\), we condition on a descendant of \(C_{i_{j+1}}\) and \(C_{i_{j+1}}\) is active. Using the inductive assumption, we get an active path \(X_0 \sim C_{i_{j}} \leftarrow \circ V_{i_j + 1} \sim  V_{i_{j+1} - 1}  \circ \rightarrow C_{i_{j+1}} \leftarrow \circ V_{i_{j+1} + 1}\).
    \end{itemize}

    By induction, we have an active \(X_0 \sim C_{i_k} \leftarrow \circ V_{i_k+1}\) path for some \(X_0 \in \*X\). The \(V_{i_k+1} \sim Y\) sub-path of \(\pi\) is active when conditioning on \(\*{Z_0}\). This is because for any node \(V\) on this sub-path (excluding \(V_{i_k+1}\) and \(Y\)), \(V \not \in \*C\) by assumption and  \(\{V_i\}_{1:n} \cap \*R \subseteq \*C \implies V \not \in \*R \implies V \not \in \*Z_0\). Recall that by assumption, \( V_{i_k+1}\) is not a collider.  We thus have an \(X_0 \sim Y\) path which is active when conditioning on \(\*{Z_0}\). We thus have a contradiction.
\end{proof}

\begin{figure}[t]
\begin{algorithmic}[1]
\Function {FindSeparator} {$\G, \*X, \*Y, \*I, \*R$}
    \State {\bfseries Output:} A set of variables \(\*Z\) $d$-separating \(\*X\) and \(\*Y\) in \(\G\) under the constraint \(\*I \setminus (\*X \cup \*Y) \subseteq \*Z \subseteq \*R \setminus (\*X \cup \*Y)\) if such \(\*Z\) exists; \(\perp\) otherwise.

    \State \(\*{R'} \gets \*{R} \setminus (\*X \cup \*Y)\)
    \State \(\*Z \gets \An{\*X \cup \*Y \cup \*I}_\G \cap \*R'\) \label{func:findseparator:z}

    \State \textbf{if} \(\*Z\) $d$-separates \(\*X, \*Y\) in \(\G\) \textbf{then}   \label{func:findsep:dsep}
    \Indent
        \State \textbf{return} \(\*Z\)
    \EndIndent
    \State \textbf{else}
    \Indent
        \State \textbf{return} \(\perp\)
    \EndIndent

\EndFunction
\end{algorithmic}
\caption{A function that finds a separator of a given pair of sets of variables, if it exists.}
\label{func:findseparator}
\end{figure}

\begin{adxlemma}[Correctness of \textsc{FindSeparator}]
\label{lemma:findsep}
    Given a causal graph \(\G\), let \(\*I, \*R, \*X, \*Y\) be sets of nodes with \(\*I \subseteq \*R\). \textsc{FindSeparator}(\(\G, \*X, \*Y, \*I, \*R\)) has a non-empty output if and only if there exists a set \(\*Z\) separating \(\*X,\*Y\) in \(\G\) such that \(\*{I} \setminus (\*X \cup \*Y) \subseteq \*{Z} \subseteq \*{R} \setminus (\*X \cup \*Y)\). Moreover, any output \(\*Z \neq \perp\) satisfies \(\*X \perp_\G \*Y \mid \*Z\) and \(\*{I} \setminus (\*X \cup \*Y) \subseteq \*{Z} \subseteq \*{R} \setminus (\*X \cup \*Y)\). Finally, \(\textsc{FindSeparator}\)  runs in time \(O(n+m)\), where \(n\) and \(m\) are the numbers of nodes and edges respectively in \(\G\).
\end{adxlemma}

\begin{proof}
    The correctness is immediate from the construction of \textsc{FindSeparator} and Lemma~\ref{lemma:existsep}. For the runtime, constructing \(\*R'\) and \(\*Z\) in the algorithm takes time \(O(n)\) and \(O(n+m)\) respectively.
    Verifying whether \(\*Z\) $d$-separates \(\*X\) from \(\*Y\) in \(\G\) or not, as shown in line~\ref{func:findsep:dsep}, may be performed by using the Bayes-Ball algorithm \cite{shachter2013bayesball} on a modified graph \(\G'\) of \(\G\) where \(\G'\) is constructed as follows: start from \(\G' = \G\), and replace each edge \(X \leftrightarrow Y\) with an explicit latent common cause \(X \leftarrow U_{XY} \rightarrow Y\).
    The construction of \(\G'\) takes \(O(n+m)\) time, and the Bayes-Ball algorithm runs in \(O(n+m)\) time.
    The overall runtime of \textsc{FindSeparator} is thus \(O(n+m)\).
\end{proof}

Since the size of the input graph \(\G\) is \(O(n+m)\), \textsc{FindSeparator} is asymptotically optimal.

\begin{customlemma}{\ref{lemma:FindAAC:correctness}}[Correctness of \textsc{FindAAC}]
\label{adxlemma:FindAAC:correctness}
    Given a causal graph \(\G\), a consistent ordering \(\*V^\prec\), and a variable \(X \in \*V^\prec\), let \(\*I,\*R\) be ancestral c-components relative to \(X\) such that \(\*I \subseteq \*R\). \textsc{FindAAC}(\(\G_{\*V^{\leq X}}, X, \*V^{\leq X}, \*I, \*R\)) outputs an admissible ancestral c-component \(\*{C}\) relative to \(X\) such that \(\*{I} \subseteq \*C \subseteq \*R\) if such a \(\*{C}\) exists, and \(\perp\) otherwise.
\end{customlemma}

\begin{proof}
By assumption, \(\*I\) is an AC relative to \(X\) in the desired range since \(\*I \subseteq \*I \subseteq \*R\). \(\textsc{FindAAC}\) outputs \(\*I\) (at line \ref{func:FindAAC:returnI}) if and only if \(\*I\) is admissible. This follows from the correctness of \(\textsc{IsAdmissible}\) (by Lemma~\ref{adxlemma:isadmissible}).

Assume \(\*I\) is not admissible. It remains to show that there exists an AAC \(\*{C_0}\) relative to \(X\) such that \(\*I \subsetneq \*{C_0} \subseteq \*R\) if and only if  there exists a variable \(D \in \De{\Spo{\*{I}} \setminus \Pa{\*{I}}}\) and a set \(\*Z\) such that \(\Pa{\*{I}} \setminus \{X,D\} \subseteq \*Z \subseteq \Pa{\*{R}} \setminus \{X,D\}\) and \(X \perp_\G D \mid \*Z\). Moreover, the output of \(\textsc{FindAAC}\) at line ~\ref{func:FindAAC:returnCZ} must be an AAC relative to \(X\) in the given range.

(\(\implies\)) Since \(\*{I}\) is not admissible, by Def.~\ref{def:validci} we have
    \begin{equation}  \label{eq:FindAAC:Inotadm}
        \*V^{\leq X} \setminus (\Pa{\*{I}} \cup \De{\Spo{\*{I}} \setminus \Pa{\*{I}}}) = \emptyset
    \end{equation}
    Since \(\*{C_0} \supsetneq \*{I}\) is admissible, by Def.~\ref{def:validci} we have 
    \begin{equation} \label{eq:FindAAC:c0adm}
        \*V^{\leq X} \setminus (\Pa{\*{C_0}} \cup \De{\Spo{\*{C_0}} \setminus \Pa{\*{C_0}}}) \neq \emptyset
    \end{equation}
        
    However, \(\*{I} \subsetneq \*{C_0} \implies \Pa{\*{I}} \subseteq \Pa{\*{C_0}}\). Therefore, Eq.~(\ref{eq:FindAAC:Inotadm}) and Eq.~(\ref{eq:FindAAC:c0adm}) imply that there exists a variable \(D \in \*V^{\leq X}\) such that \(D \in \De{\Spo{\*{I}} \setminus \Pa{\*{I}}}\) and \(d \not \in  \Pa{\*{C_0}}  \cup \De{\Spo{\*{C_0}} \setminus \Pa{\*{C_0}}}\).
    By the definition of C-LMP (shown in Def.~\ref{def:lmpplus}), we have that \(X \perp_\G D \mid \Pa{\*{C_0}} \setminus \{X\}\).
    Since \(\*{I} \subseteq \*{C_0} \subseteq \*R\) and \(D \not \in \Pa{\*{C_0}}\), we have \(\Pa{\*{I}} \setminus \{X,D\} \subseteq \Pa{\*{C_0}} \setminus \{X\}  \subseteq \Pa{\*{R}} \setminus \{X,D\}\). Therefore, \(\*Z = \Pa{\*{C_0}} \setminus \{X\}\) is a set such that \(\Pa{\*{I}} \setminus \{X,D\} \subseteq \*Z \subseteq \Pa{\*{R}} \setminus \{X,D\}\) and \(X \perp_\G D \mid \*Z\).
    The correctness of \textsc{FindSeparator} (Lemma~\ref{lemma:findsep}) implies that \textsc{FindAAC} detects the existence of \(\*Z\) and outputs \(\*C = \&{C}(X)_{\G_{\An{\*{I} \cup \*Z}}}\) at line ~\ref{func:FindAAC:returnCZ}. In the proof for the reverse direction, we will show that \(\*C\) thus defined is in fact admissible.

(\(\impliedby\)) Consider some \(D \in \De{\Spo{\*I} \setminus \Pa{\*I}}\) such that
\begin{equation*}
    \begin{split}
        \*Z &= \textsc{FindSeparator}(\G_{\*{V}^{\leq X}}, \{X\}, \{D\}, Pa(\*{I}), Pa(\*{R})) \\
        & \neq \perp
    \end{split}
\end{equation*}
By the correctness of \textsc{FindSeparator} (Lemma~\ref{lemma:findsep}), we have \(X \perp_\G D \mid \*Z\). We give a constructive proof of existence by showing that \(\*C = \&{C}(X)_{\G_{\An{\*{I} \cup \*Z}}}\) is an AAC relative to \(X\) such that \(\*I \subsetneq \*C \subseteq \*R\).

Clearly, \(\*C\) is an AC by construction and \(\*{I} = \&{C}(X)_{\G_{\*{I}}} \subseteq \*C\). Moreover,
\begin{align}
    \*{I}, \*Z \subseteq \Pa{\*{R}} & \implies \*{I} \cup \*Z \subseteq \Pa{\*{R}} \\
    & \implies \An{\*{I} \cup \*Z} \subseteq \An{ \*{R}} \\
    & \implies \&{C}(X)_{\G_{\An{\*{I} \cup \*Z}}} \subseteq \&{C}(X)_{\G_{\An{ \*{R}}}} \\
    & \implies \&{C}(X)_{\G_{\An{\*{I} \cup \*Z}}} \subseteq \*{R} 
\end{align}
where the last implication follows since \(\*{R}\) is an AC relative to \(X\) by assumption, implying that \(\*{R} = \&{C}(X)_{\G_{\An{\*{R}}}}\).
Moreover, we claim that \(\*C\) is admissible, i.e.,
\[
    \*S^+ = \*V^{\leq X} \setminus (\Pa{\*C} \cup \De{\Spo{\*C} \setminus \Pa{\*C}}) \neq \emptyset
\]

We will show that \(D \not \in \Pa{\*C} \cup \De{\Spo{\*C} \setminus \Pa{\*C}}\), hence \(\*S^+\) contains \(D\) (and is therefore non-empty). We know \(D \in \*V^{\leq X}\). Assume, towards contradiction, that \(D \in \Pa{\*C} \cup \De{\Spo{\*C} \setminus \Pa{\*C}}\). Note that \(\*I \setminus \{X,D\} \subseteq \*Z \implies \An{\*I} \subseteq \An{\*Z \cup \{X,D\}}\). Therefore, \(\*C \subseteq \An{\*{I} \cup \*Z} \subseteq  \An{\*Z \cup \{X,D\}}\). Since \(X,D\) are non-adjacent (because \(\*Z\) separates them), this implies the existence of a path \(\pi\) of one of the following types:
\begin{enumerate}
    \item If \(D \in \*C\), then \(X \leftrightarrow V_1 \dots \leftrightarrow V_n \leftrightarrow V_{n+1} = D\) with \(n \geq 1\) and each \(V_i \in \An{\*Z \cup \{X,D\}}\) for \(i \in [n+1]\)
    \item If \(D \in \Pa{\*C}\), then \(X \leftrightarrow V_1 \dots \leftrightarrow V_n \leftarrow V_{n+1} = D\) with \(n \geq 1\) and each \(V_i \in  \An{\*{I} \cup \*Z}\) for \(i \in [n+1]\)
    \item If \(D \in \De{\Spo{\*C} \setminus \Pa{\*C}})\), then \(X \leftrightarrow V_1 \leftrightarrow \dots \leftrightarrow V_n \leftrightarrow V_{n+1} = A \rightsquigarrow D\) with each \(V_i \in \An{\*Z \cup \{X,D\}}\) for each \(i \in [n+1]\) with \(n \geq  0\) and \(A \in \Spo{\*C} \setminus \Pa{\*C}\). It is possible that the path \(A \rightsquigarrow d\) has length 0, i.e., \(A = D\).
\end{enumerate}

 We show by induction that there is an active \(X \sim V_i \leftarrow \circ V_{i+1}\) path for each \(i \in [n]\) when conditioning on \(\*Z\). 

\emph{Base case.} We know \(V_1 \in \An{\*Z \cup \{X,D\}}\).
\begin{itemize}
    \item If \(V_1 \in \An{\*Z}\), \(V_1\) is active when conditioning on \(\*Z\), hence the sub-path \(X \leftrightarrow V_1 \leftarrow \circ V_{2}\) of \(\pi\) is active.
    \item If \(V_1 \in \An{\{D\}}\), then there is a path \(X \leftrightarrow V_1 \rightsquigarrow D\). Since \(\*Z\) must block this path, we have \(\*Z \cap \De{\{V_1\}} \neq \emptyset\), hence \(V_1\) is active when conditioning on \(\*Z\) and the sub-path \(X \leftrightarrow V_1 \leftarrow \circ V_{2}\) of \(\pi\) is active.
    \item  If \(V_1 \in \An{\{X\}}\), then \(X \in \*{I}\), \(V_1 \in \Spo{\{X\}}\), and \(\*{I}\) is an AC implies that \(V_1 \in \*{I}\). Since \(V_1 \not \in \{X,D\}\), this implies that \(V_1 \in \*Z\).  Therefore, \(V_1\) is active when conditioning on \(\*Z\) and the sub-path \(X \leftrightarrow V_1 \leftarrow \circ V_{2}\) of \(\pi\) is active.
\end{itemize}

\emph{Inductive assumption}. Assume, for some \(i \in [n]\), there is an active \(X \sim V_i \leftrightarrow V_{i+1}\) path when conditioning on \(\*Z\). 

\emph{Inductive step.} We show that there is an active \(X \sim V_{i+1} \leftarrow \circ V_{i+2}\) path  when conditioning on \(\*Z\). We know \(V_{i+1} \in  \An{\*Z \cup \{X,D\}}\).
\begin{itemize}
    \item If \(V_{i+1} \in \An{\*Z}\), \(V_{i+1}\) is active when conditioning on \(\*Z\). Therefore, the inductive assumption gives us an active path \(X \sim V_i \leftrightarrow V_{i+1} \leftarrow \circ V_{i+2}\) when conditioning on \(\*Z\).
    \item If \(V_{i+1} \in \An{\{D\}}\), then there is a path \(V_{i+1} \rightsquigarrow D\). By the inductive assumption, there is an active path \(X \sim V_i \leftrightarrow V_{i+1}\) when conditioning on \(\*Z\). Since \(\*Z\) must block the path \(X \sim V_i \leftrightarrow V_{i+1} \rightsquigarrow D\), we have \(\*Z \cap \De{\{V_{i+1}\}} \neq \emptyset\), hence \(V_{i+1}\) is active when conditioning on \(\*Z\) and the path \(X \sim V_i \leftrightarrow V_{i+1} \leftarrow \circ V_{i+2}\) is active.
    \item  If \(V_{i+1} \in \An{\{X\}}\), then there is a path \(X \leftsquigarrow V_{i+1}\). If \(\De{\{V_{i+1}\}} \cap \*Z \neq \emptyset\), then \(V_{i+1}\) is active when conditioning on \(\*Z\) and by the inductive assumption, the path \(X \sim V_i \leftrightarrow V_{i+1} \leftarrow \circ V_{i+2}\) is active. If  \(\De{\{V_{i+1}\}} \cap \*Z = \emptyset\), then the path \(X \leftsquigarrow V_{i+1} \leftarrow \circ V_{i+2}\) is active.
\end{itemize}
Therefore, by induction, there is an active \(X \sim V_n \leftarrow \circ V_{n+1} \) path when conditioning on \(\*Z\). If  \(V_{n+1} = D\), this contradicts \(X \perp_\G d \mid \*Z\). Otherwise, if \(V_{n+1} = A \rightsquigarrow D\) in Case (3), then \(\*Z\) must block the path \(A \rightsquigarrow D\).
This implies that \(A \in \An{\*Z}\); moreover, \(A \in \Spo{\*C}\) and \(\*C = \&{C}(X)_{\G_{\An{\*I \cup \*Z}}}\) implies that \(A \in \*C\), which contradicts the assumption that \(A \in \Spo{\*C} \setminus \Pa{\*C}\).
\end{proof}

\begin{adxprop}[Runtime of \textsc{FindAAC}]
\label{prop:FindAAC:runtime}
    Given a causal graph \(\G\), a consistent ordering \(\*V^\prec\), and a variable \(X \in \*V^\prec\), let \(\*I,\*R\) be ancestral c-components relative to \(X\) such that \(\*I \subseteq \*R\).
    \textsc{FindAAC}(\(\G_{\*V^{\leq X}}, X, \*V^{\leq X}, \*I, \*R\)) runs in \(O(n(n+m))\) time where \(n\) and \(m\) denote the numbers of nodes and edges in \(\G\) respectively.
\end{adxprop}

\begin{proof}
    A call to the function \textsc{IsAdmissible} in line~\ref{func:FindAAC:callisadmissible} takes \(O(n+m)\) time (by Lemma~\ref{adxlemma:isadmissible}).
    \textsc{FindAAC} computes a set of variables \(\De{\Spo{\*I} \setminus \Pa{\*I}}\) (shown in line~\ref{func:FindAAC:foreachd}) only once, which takes \(O(n+m)\) time. There are at most \(O(n)\) iterations of the for loop, within which a call to the function \textsc{FindSeparator} (by Lemma~\ref{lemma:findsep}) and the construction of a c-component \(\&{C}(X)_{\G_{\An{\*{I} \cup \*Z}}}\) in line \ref{func:FindAAC:returnCZ} (by Prop.~\ref{prop:constc}) take time \(O(n+m)\). Thus, the total runtime of \textsc{FindAAC} is \(O(n(n+m))\).
\end{proof}

\begin{adxprop}[Ancestrality of Modified ACs] \label{prop:accdesc}
    Given a causal graph \(\G\), a consistent ordering \(\*V^\prec\), and a variable \(X \in \*V^\prec\), let \(\*C\) be an ancestral c-component relative to \(X\). For any \(\*S \subseteq \*V^\prec\) such that \(X \not \in \De{\*S}\), \(\*{C_S} = \&{C}(X)_{\G_{\*C \setminus \De{\*S}}}\) is an ancestral c-component relative to \(X\).
\end{adxprop}

\begin{proof}
    It suffices to show that \(\*{C_S} = \&{C}(X)_{\G_{\An{\*{C_S}}}}\). Since \(\*{C_S \subseteq \An{\*{C_S}}}\) and \(\*{C_S} = \&{C}(X)_{\G_{\*{C_S}}}\), we have \(\*{C_S} \subseteq  \&{C}(X)_{\G_{\An{\*{C_S}}}}\). To show \(\*{C_S} \supseteq  \&{C}(X)_{\G_{\An{\*{C_S}}}}\), we make use of two facts.
    Since  \(\*{C_S} = \&{C}(X)_{\G_{\*C \setminus \De{\*S}}}\), we have \(\*{C_S} \cap \De{\*S} = \emptyset\). This further implies that  \(\An{\*{C_S}} \cap \De{\*S} = \emptyset\) (if some \(W \in \An{\*{C_S}} \cap \De{\*S}\), then \(\exists S \in \*S\) such that \(S \in \An{\{W\}} \subseteq \An{\*{C_S}}\) contradicts \(\*{C_S} \cap \De{\*S} = \emptyset\)). Therefore, we have \(\&{C}(X)_{\G_{\An{\*{C_S} \setminus \De{\*S}} \setminus \De{\*S}}} = \&{C}(X)_{\G_{\An{\*{C_S}}}}\); Let \(\*A =  \&{C}(X)_{\G_{\An{\*{C_S} \setminus \De{\*S}} \setminus \De{\*S}}} \). We now show that \(\*{C_S} \supseteq  \*A\). Consider some variable \(W \in \*A\). Then, there exists a variable \(Y \in \*{C_S} \setminus \De{\*S}\) such that \(W \in \An{\{Y\}}\). Moreover, since \(Y \in \*{C_S} \subseteq \*A\), we either have \(W = Y\) (and hence \(W \in \*{C_S}\)) or a path \(W = V_k \leftrightarrow \dots \leftrightarrow V_1 \leftrightarrow Y\) for some \(k \geq 1\) with \(V_i \in \An{\*{C_S} \setminus \De{\*S}} \setminus \De{\*S}\) for each \(i \in [k]\) (by the construction of \(\*A\)). We show by induction that \(V_i \in \*{C_S}\) for each \(i \in [k]\). 
    
    \emph{Base case}. \(k = 1\). Since \(V_1 \in \An{\*{C_S} \setminus \De{\*S}} \setminus \De{\*S}\), we have \(\{V_1\} \cap \De{\*S} = \emptyset\). Moreover, \(V_1 \in \An{\*{C_S}} \subseteq \An{\*C}\) (since \(\*{C_S} \subseteq \*C\)). Furthermore, \(Y \in \*{C_S} \subseteq \*C\). So, \(Y \in \*C\), \(V_1 \in \An{\*C}\), \(V_1 \leftrightarrow Y\), and \(\*C\) is an AC by assumption implies that \(V_1 \in \*C\). Therefore, \(V_1 \in \*C \setminus \De{\*S}\). Since \(Y \in \*{C_S} = \&{C}(X)_{\G_{\*C \setminus \De{\*S}}}\), \(Y \leftrightarrow V_1\), and \(V_1 \in \*C \setminus \De{\*S}\), we get \(V_1 \in \*{C_S}\).
    
    \emph{Inductive assumption}. Assume, for some \(i \in [k-1]\), we have \(V_i \in \*{C_S}\).
    
    \emph{Inductive step}. By similar reasoning as in the base case, we show that \(V_{i+1} \in \*{C_S}\).  Since \(V_{i+1} \in \An{\*{C_S} \setminus \De{\*S}} \setminus \De{\*S}\), we have \(\{V_{i+1}\} \cap \De{\*S} = \emptyset\). Moreover, \(V_{i+1} \in \An{\*{C_S}} \subseteq \An{\*{C}}\). Furthermore, by the inductive assumption, \(V_i \in \*{C_S} \subseteq \*{C}\). So, \(V_i \in \*{C}\), \(V_{i+1} \in \An{\*{C}}\), \(V_{i+1} \leftrightarrow V_i\), and \(\*{C}\) being an AC implies that \(V_{i+1} \in \*{C}\). Therefore, \(V_{i+1} \in \*{C} \setminus \De{\*S}\). Since \(V_i \in \*{C_S}\), \(V_{i+1} \leftrightarrow V_i\), and \(V_{i+1} \in \*{C} \setminus \De{\*S}\), we get \(V_{i+1} \in \*{C_S}\).
    
    Therefore, \(W \in \*{C_S}\) and since \(W\) was chosen arbitrarily from \(\*A\), we have \(\*A \subseteq \*{C_S}\).
\end{proof}

\begin{customlemma}{\ref{lemma:listcix}}[Correctness of \textsc{ListCIX}]
\label{adxlemma:listcix}
    \textsc{ListCIX} (\(\G_{\*V^{\leq X}}, X, \*V^{\leq X}, \*I, \*R\)) enumerates all and only all non-vacuous conditional independence relations invoked by the c-component local Markov property associated with \(X\) and admissible ancestral c-components \(\*C\) relative to \(X\) where \(\*I \subseteq \*C \subseteq \*R\).
    Further, \textsc{ListCIX} runs in \(O(n^2(n+m))\) delay where \(n\) and \(m\) represent the number of nodes and edges in \(\G\), respectively.
   
\end{customlemma}

\begin{proof}
We show the correctness of \textsc{ListCIX} and the running time that \textsc{ListCIX} runs in \(O(n^2(n+m))\) delay.

\begin{itemize}
    \item Correctness:
        We prove correctness by structural induction on the binary recursion tree for \textsc{ListCIX}, rooted at \(\&{N}(\*I,\*R)\). We claim that \textsc{ListCIX} called at a node \(\&{N}(\*I',\*R')\) enumerates all and only non-vacuous CIs of \(X\) invoked by C-LMP (Def.~\ref{def:lmpplus}) such that the conditioning set is of the form \(\Pa{\*C} \setminus \{X\}\) for some AC \(\*C\) such that \(\*I' \subseteq \*C \subseteq \*R'\).

        \emph{Base case.} Consider a leaf node \(\&{N}(\*I',\*R')\) \footnote{A leaf node is a node that has no children.}. Let \[
            \*C := \textsc{FindAAC}(\G_{\*V^{\leq X}}, X, \*V^{\leq X}, \*I', \*R').
        \]
        Since we are at a leaf node, we either have
        \begin{enumerate}
            \item \(\*C = \perp\), in which case \textsc{ListCIX} outputs nothing at \(\&{N}(\*I',\*R')\). By the correctness of \textsc{FindAAC} (Lemma ~\ref{lemma:FindAAC:correctness}) there are no CIs of \(X\) invoked by C-LMP such that the conditioning set is of the form \(\Pa{\*C} \setminus \{X\}\) for some AC \(\*C\) such that \(\*I' \subseteq \*C \subseteq \*R'\).
            Therefore, the output is correct.
            \item \(\*I' = \*R'\) and hence \(\*C = \*I'\). Similarly, by the correctness of \textsc{FindAAC} and the definition of C-LMP, \textsc{ListCIX} outputs the unique non-vacuous CI of the desired form at \(\&{N}(\*I',\*R')\).
        \end{enumerate}
        Note that these are the only conditions under which we are at a leaf node. If \(\*C \neq \perp\) and \(\*I' \neq \*R'\),  then \(\*T = \*R' \cap (\Spo{\*I'} \setminus \*I')\) must be non-empty (because \(\*R'\) is a c-component such that \(\*R' \supsetneq \*I'\)) and we recurse.

        \emph{Inductive assumption.} Assume the claim holds for some nodes \(\&{N}_1(\*I_1,\*R_1),\) \(\&{N}_2(\*I_2,\*R_2)\).

        \emph{Inductive step.} We show that the claim holds for any node \(\&{N}_0(\*I',\*R')\) whose children are \(\&{N}_1(\*I_1,\*R_1), \&{N}_2(\*I_2,\*R_2)\). We first define three collections of ACs relative to \(X\). Recall that \(\&{AC}_X\) (Def.~\ref{def:collectionsczs}) denotes the set of all ACs relative to \(X\).
        \begin{equation*}
            \begin{split}
                \&{AC}_0 = \{\*C \subseteq \*V^{\leq X} \mid \: &\*C \in \&{AC}_X, \*C \text{ is AAC, } \\
                &\*I' \subseteq \*C \subseteq \*R'\}
            \end{split}
        \end{equation*}
        \begin{equation*}
            \begin{split}
                \&{AC}_1 = \{\*C \subseteq \*V^{\leq X} \mid \: &\*C \in \&{AC}_X, \*C \text{ is AAC, } \\
                &\*I_1 \subseteq \*C \subseteq \*R_1\}
            \end{split}
        \end{equation*}
        \begin{equation*}
            \begin{split}
                \&{AC}_2 = \{\*C \subseteq \*V^{\leq X} \mid \: &\*C \in \&{AC}_X, \*C \text{ is AAC, } \\
                &\*I_2 \subseteq \*C \subseteq \*R_2\}
            \end{split}
        \end{equation*}
      
        It suffices to show that \(\&{AC}_0 = \&{AC}_1 \cup \&{AC}_2\), since this implies that any CI that should be output by \textsc{ListCIX} at \(\&{N}_0(\*I',\*R')\) is output by \textsc{ListCIX} at either \(\&{N}_1(\*I_1,\*R_1)\) or \(\&{N}_2(\*I_2,\*R_2)\). Since \textsc{ListCIX} at \(\&{N}_0(\*I',\*R')\) calls  \textsc{ListCIX} at both \(\&{N}_1(\*I_1,\*R_1)\) and \(\&{N}_2(\*I_2,\*R_2)\), we prove the claim.

        By construction at lines~\ref{func:listcix:recursion:rprime}-\ref{func:listcix:recursion:iprime}, since \(\&{N}_0(\*I',\*R')\) has children \(\&{N}_1(\*I_1,\*R_1)\) and \(\&{N}_2(\*I_2,\*R_2)\), we have (without loss of generality) that \((\*I_1,\*R_1) = (\&{C}(X)_{\G_{\An{\*I' \cup \{S\}}}}, \*R')\) and \((\*I_2,\*R_2) = (\*I', \&{C}(X)_{\G_{\*R' \setminus \De{\{S\}}}})\) for some \(S \in \*T = \*R' \cap (\Spo{\*I'} \setminus \*I')\).

        First, some technicalities. we want to show that \((\*I_1,\*R_1), (\*I_2,\*R_2)\) are well-defined and non-vacuous inputs to \textsc{ListCIX}.
        \begin{itemize}
            \item Since \(S \in \*T \subseteq \*R'\) and \(\*I' \subseteq \*R'\), and \(\*R'\) is an AC relative to \(X\), we have \(\*I_1 = \&{C}(X)_{\G_{\An{\*I' \cup \{S\}}}} \subseteq \*R' = \*R_1\). Furthermore, \(\*I_1\) is an AC by construction.
            \item Since \(S \in \Spo{\*I'} \setminus \*I'\) and \(\*I'\) is an AC, we have \(\De{\{S\}} \cap \*I' = \emptyset\). Since \(\*I' \subseteq \*R'\), this further implies that \(\*I_2 = \*I' \subseteq \*R_2 = \&{C}(X)_{\G_{\*R' \setminus \De{\{S\}}}}\). Moreover, \(\*R_2\) is an AC by Prop.~\ref{prop:accdesc}.
        \end{itemize}

        We show the equality of \(\&{AC}_0\) and \(\&{AC}_1 \cup \&{AC}_2\) in both directions.

        \begin{itemize}
            \item \( \&{AC}_1 \cup \&{AC}_2 \subseteq \&{AC}_0\)

            For any \(\*C \in \&{AC}_1 \cup \&{AC}_2\), we have either 
            \begin{itemize}
                \item \(\*I_1 \subseteq \*C \subseteq \*R_1\) and hence \(\*I' \subseteq \*C \subseteq \*R'\) since \(\*I' \subseteq \*I_1, \*R' = \*R_1\), \emph{or}
                \item  \(\*I_2 \subseteq \*C \subseteq \*R_2\) and hence \(\*I' \subseteq \*C \subseteq \*R'\) since \(\*I' = \*I_2, \*R_2 \subseteq \*R'\).
            \end{itemize}
            Therefore, \( \&{AC}_1 \cup \&{AC}_2 \subseteq \&{AC}_0\).

            \item \(\&{AC}_1 \cup \&{AC}_2 \supseteq \&{AC}_0\)

            For any \(\*C \in \&{AC}_0\), we have \(\*I' \subseteq \*C \subseteq \*R'\). Then, we have either
            \begin{itemize}
                \item \(S \in \*C\). Therefore, \(\*I' \cup \{S\} \subseteq \*C\). Since \(\*C\) is an AC, we must have \(\&{C}(X)_{\G_{\An{\*I' \cup \{S\}}}} \subseteq \*C\) (otherwise, there is an ancestor of \(\*I' \cup \{S\}\)  not in \(\*C\) which is connected by a bi-directed path to some node in \(\*I \cup \{S\} \subseteq \*C\)). Therefore, we have \(\*I_1 \subseteq \*C \subseteq \*R_1 = \*R'\) and hence \(\*C \in \&{AC}_1\).
                \item \(S \not \in \*C\). Since \(\*C\) is an AC, this implies \(\De{\{S\}} \cap \*C = \emptyset\). Moreover, \(\*C \subseteq \*R'\) further implies that \(\*C \subseteq \&{C}(X)_{\G_{\*R' \setminus \De{\{S\}}}}\).  Therefore, we have \(\*I_2 = \*I_1 \subseteq \*C \subseteq \*R_2\), and \(\*C \in \&{AC}_2\).
            \end{itemize}
            Therefore, \(\&{AC}_1 \cup \&{AC}_2 \supseteq \&{AC}_0\).
        \end{itemize}

        Hence, \(\&{AC}_0 = \&{AC}_1 \cup \&{AC}_2\) and we are done. Since \textsc{ListCIX} enumerates all AACs relative to \(X\) correctly, by Thm.~\ref{thm:equivalence:clmpplus}, it enumerates all non-vacuous CIs invoked by C-LMP for \(X\) correctly.

    \item Running time:

    Consider the recursion tree for \textsc{ListCIX}.
    Whenever a tree node \(\&{N}(\*I',\*R')\) is visited, the function \textsc{FindAAC} is called, which takes \(O(n(n+m))\) time (by Lemma~\ref{adxlemma:FindAAC:correctness}).
    If \textsc{FindAAC} outputs \(\perp\), then \textsc{ListCIX} does not search further from  \(\&N\) because there exists no AAC \(\*C\) relative to \(X\) with \(\*I' \subseteq \*C \subseteq \*R'\).
    Otherwise, recursion continues until a leaf tree node is visited.
    In each level of the tree, a single node \(S\) is removed from \(\*T\).
    Either the variable \(S\), is added to \(\*I\), resulting in \(\*I'\) given at line~\ref{func:listcix:iprime} (by \(S \in \Spo{\*I} \setminus \*I\)), or the variable \(S\) is removed from \(\*R\) to construct \(\*R'\) which is shown at line~\ref{func:listcix:rprime}.
    Any \(\*C\) is uniquely contained in one child; therefore, no CI is output more than once.
    The depth of the tree is at most \(n\), and the time required to find one \(\*C\) and output one non-vacuous CI invoked by C-LMP associated with \(\*C\) (following Def.~\ref{def:lmpplus}), is \(O(n^2(n+m))\).
    In the worst case scenario, \(n\) branches will be aborted (i.e., \textsc{FindAAC} outputs \(\perp\) on every level of the tree) before reaching the first leaf.
    It takes \(O(n^2(n+m))\) time to produce either the first output or halt.
\end{itemize}
\end{proof}

\begin{customthm}{\ref{thm:listci}}[Correctness of \textsc{ListCI}]
\label{adxthm:listci}
    Let \(\G\) be a causal graph and \(\*V^\prec\) a consistent ordering.
    \textsc{ListCI}(\(\G, \*V^\prec\)) enumerates all and only all non-vacuous conditional independence relations invoked by the c-component local Markov property in \(O(n^2(n+m))\) delay where \(n\) and \(m\) represent the number of nodes and edges in \(\G\), respectively.
\end{customthm}

\begin{proof}
We show the correctness of \textsc{ListCI} and the running time that \textsc{ListCI} runs in \(O(n^2(n+m))\) delay.

\begin{itemize}
    \item Correctness: 

    \textsc{ListCI}(\(\G, \*V^\prec\)) iterates over each variable \(X \in \*V^\prec\).
    For each \(X\), \textsc{ListCI} constructs two ACs relative to \(X\): \(\*I = \&C(X)_{\G_{\An{\{X\}}}}\) and \(\*R = \&C(X)_{\G_{\*V^{\leq X}}}\).
    \(\*I\) is the minimum-size AC relative to \(X\) and \(\*R\) is the maximum-size AC relative to \(X\).
    Then, \textsc{ListCI} calls the function \textsc{ListCIX}(\(\G_{\*V^{\leq X}}, X, \*V^{\leq X}, \*I, \*R\)) that outputs all and only all non-vacuous CIs invoked by C-LMP associated with \(X\), which is performed by generating all and only all AACs \(\*C\) relative to \(X\) under the constraint \(\*I \subseteq \*C \subseteq \*R\) (by Lemma~\ref{lemma:listcix}).
    \textsc{ListCI}(\(\G, \*V^\prec\)) iterates over each variable \(X \in \*V^\prec\), and thus \textsc{ListCI}(\(\G, \*V^\prec\)) lists all and only all non-vacuous CIs invoked by C-LMP.

    \item Running time:
    
    There are two types of worst case scenarios.

    \begin{enumerate}
        \item No CI is invoked by C-LMP.

        By the assumption that no CI is invoked by C-LMP, Def.~\ref{def:lmpplus} and Def.~\ref{def:validci}, none of the ACs \(\*C\) relative to \(X\) for any \(X \in \*V^{\prec}\) is admissible.
        For each \(X\) visited by \textsc{ListCI}(\(\G, \*V^\prec\)), \textsc{ListCI} calls the function \textsc{ListCIX}(\(\G_{\*V^{\leq X}}, X, \*V^{\leq X}, \*I, \*R\)) at line~\ref{alg:listci:calllistcix}.
        \textsc{ListCIX} spends \(O(n(n+m))\) time and terminates with no output since \textsc{ListCIX} calls the function \(\textsc{FindAAC}(\G_{\*V^{\leq X}}, X, \*V^{\leq X}, \*I, \*R)\) at line~\ref{func:listcix:callfindadmissiblec}, which returns \(\perp\) (by Lemma~\ref{lemma:FindAAC:correctness}).
        \textsc{ListCI} checks the next variable of \(X\), if any exists.
        Since \(|\*V^\prec| = n\), \textsc{ListCI} spends \(O(n^2(n+m))\) time and then terminates with no output.

        \item No CI invoked by C-LMP exists for all variables \(X \in \*V^\prec\) except for the last variable \(X_n\) in the ordering \(\*V^\prec\).

        For the first \(n-1\) variables in \(\*V^\prec\), \textsc{ListCI} spends \(O(n^2(n+m)) = O(n(n+m) * (n-1))\) time producing no output.
        More specifically, for each variable, \textsc{ListCIX} spends \(O(n(n+m))\) time and terminates with no output since \textsc{ListCIX} calls \(\textsc{FindAAC}(\G_{\*V^{\leq X}}, X, \*V^{\leq X}, \*I, \*R)\) at line~\ref{func:listcix:callfindadmissiblec}, which returns \(\perp\).
        When \(X = X_n\), \textsc{ListCI} calls the function \textsc{ListCIX}(\(\G_{\*V^{\leq X}}, X, \*V^{\leq X}, \*I, \*R\)) at line~\ref{alg:listci:calllistcix} where \textsc{ListCIX} spends \(O(n^2(n+m))\) time to output one non-vacuous CI invoked by C-LMP that is associated with \(X_n\).
        In total, \textsc{ListCI} spends \(O(n^2(n+m))\) time to produce an output.
    \end{enumerate}
\end{itemize}
\end{proof}

\subsection{Appendix Proofs}

\begin{adxprop}[Total Number of CIs Invoked by GMP]
\label{prop:gmpsize}
    Given a causal graph \(\G\) over a set of variables \(\*V\) with \(n = |\*V|\), the global Markov property for \(\G\) invokes \(O(4^n)\) number of conditional independence relations. Moreover, there exists a causal graph \(\G\) for which the bound is tight, that is, the global Markov property for \(\G\) implies \(\Omega(4^n)\) number of conditional independence relations.
\end{adxprop}

\begin{proof}
    Each CI invoked by GMP is given by a choice of disjoint sets \(\*{X,Y,Z \subseteq V}\) where \(\*{X,Y} \neq \emptyset\) and a $d$-separation statement \(\*X \perp_d \*Y \mid \*Z\). The number of such statements is upper-bounded by
    \begin{align}
        d(n) &= \frac{1}{2}\sum_{i=1}^n \binom{n}{i} \sum_{j=1}^{n-i}\binom{n-i}{j} \sum_{k=0}^{n-i-j}\binom{n-i-j}{k} \\
        &= \frac{2^n(2^n-1)}{2} - 3^n + 2^n
    \end{align}
    where we divide the quantity by 2 to avoid double-counting the following two symmetrical statements: \(\*X \perp_d \*Y \mid \*Z\) and \(\*Y \perp_d \*X \mid \*Z\) (since $d$-separation is symmetric). We first simplify the inner-most sum.
    \begin{equation}
        \sum_{k=0}^{n-i-j}\binom{n-i-j}{k} = 2^{n-i-j}
    \end{equation}
    This gives
    \begin{equation}
        d(n) = \frac{1}{2}\sum_{i=1}^n \binom{n}{i} \sum_{j=1}^{n-i}\binom{n-i}{j}2^{n-i-j}
    \end{equation}
    We then simplify the second nested sum. Note that by the binomial theorem,
    \begin{align}
        3^{n-i} &= (2+1)^{n-i} \\
        &= \sum_{j=0}^{n-i}\binom{n-i}{j}2^{n-i-j}1^j \\
        &= 2^{n-i} + \sum_{j=1}^{n-i}\binom{n-i}{j}1^j 2^{n-i-j}
    \end{align}
    Therefore,
    \begin{align}
        d(n) & = \frac{1}{2}\sum_{i=1}^n \binom{n}{i} (3^{n-i} -  2^{n-i}) \\
            & = \frac{1}{2} \left(\sum_{i=1}^n \binom{n}{i} 3^{n-i} - \sum_{i=1}^n \binom{n}{i} 2^{n-i}) \right)\\
            & = \frac{1}{2} \left(4^n  + 2^n\right) - 3^n \in O(4^n).
    \end{align}

    Moreover, for \(\G\) given by the independent set on \(n\) variables (i.e., \(\G\) contains no edges), every possible $d$-separation holds, therefore we get \(d(n) \in \Omega(4^n)\) number of CIs implied by GMP.
\end{proof}

\section{Discussion and Examples}

\subsection{Explaining Markov properties} \label{sec:adx:clmp}

A causal DAG on \(n\) variables may encode \(\Theta(4^n)\) CIs (Prop.~\ref{prop:gmpsize}). 
In this section, we explain how a subset of these CIs, often considered a `basis' \cite{bareinboim:etal20}, may imply all others.

The CI relation is a \emph{semi-graphoid} \cite{pearl:86a, pearl:88a}.
Given an arbitrary probability distribution \(P(\*v)\) over a set of variables \(\*V\), CIs in \(P(\*v)\) must exhibit certain properties.
Specifically, for disjoint sets of variables \(\*X,\*Y, \*W, \*Z\), where \(\*X,\*Y \neq \emptyset\), the probability axioms can be used to show that the following properties hold:

\begin{enumerate}
    \item Symmetry: \(\*X \indep \*Y \mid \*Z \iff \*Y \indep \*X \mid \*Z\)
    \item Decomposition: \(\*X \indep \*Y \cup \*W \mid \*Z \implies \*X \indep \*Y \mid \*Z\) and \(\*X \indep \*W \mid \*Z\)
    \item Weak union:\(\*X \indep \*Y \cup \*W \mid \*Z \implies \*X \indep \*Y \mid \*Z \cup \*W\) and \(\*X \indep \*W \mid \*Z \cup \*Y\)
    \item Contraction: \(\*X \indep \*Y \mid \*Z\) and \(\*X \indep \*W \mid \*Z \cup \*Y \implies \*X \indep \*Y \cup \*W \mid \*Z\)
\end{enumerate} 
    
We give an example to show how these axioms can be applied.

\begin{adxexample}
\label{adxexample:basis}
    The  DAG \(\G\) (Fig.~\ref{fig:intro_basis_dag})  encodes 5 CIs, but the colored subsets of CIs can be used to derive all others, as shown in Fig.~\ref{fig:intro_basis_axioms}.
    For one example, the CI \(X_4 \indep \{X_1,X_2\}\) implies all others.
    In the context of testing \(\G\) against observational data, it suffices to test only \(X_4 \indep \{X_1,X_2\}\).
    For another example, \(X_4 \indep X_2\) and \(X_1 \indep X_4 \mid X_2\) together imply \(X_4 \indep \{X_1, X_2\}\) by the contraction axiom, and hence all other CIs.
    Therefore, it suffices also to only test \(X_4 \indep X_2\) and \(X_1 \indep X_4 \mid X_2\).
\qed
\end{adxexample}

In the given example, scrutiny revealed which CIs are sufficient to derive others via the semi-graphoid axioms. 
Markov properties, however, provide a systematic way to identify such CIs. 
The semi-graphoid axioms can be used to show equivalence between Markov properties: for example, between the local Markov property and GMP for Markovian DAGs \cite{pearl:88a,lauritzen:etal90,lauritzen:96}, and between (LMP,\(\prec\)) and GMP for semi-Markovian DAGs \cite{richardson2003markov}.

\subsection{C-LMP and the Semi-Markov Factorisation}
\label{sec:adx:paplus}

In this section, we develop a connection between Markov properties and a related notion of compatibility between causal graphs and observational data – the factorisation that the distribution should admit.
This offers another perspective on the combinatorial explosion in the number of CIs invoked by C-LMP in the semi-Markovian case, compared with the Markovian case.

An observational distribution \(P(\*v)\) over a set of variables \(\*V\) factorizes, according to the chain rule, as
\begin{align*}
    P(\*v) = \prod_{V_i \in \*V}p(v_i \mid v_1,\dots, v_{i-1})
\end{align*}

However, if we know \(P(\*v)\) is compatible with a graph \(\G\), CIs implied by \(\G\) can be used to simplify the factorisation above to a \emph{Markov factorisation}. We first define the Markov factorisation for Markovian DAGs.

\begin{figure}[t]
    \centering
    \begin{subfigure}{0.3\textwidth}
        \includegraphics[width=\textwidth]{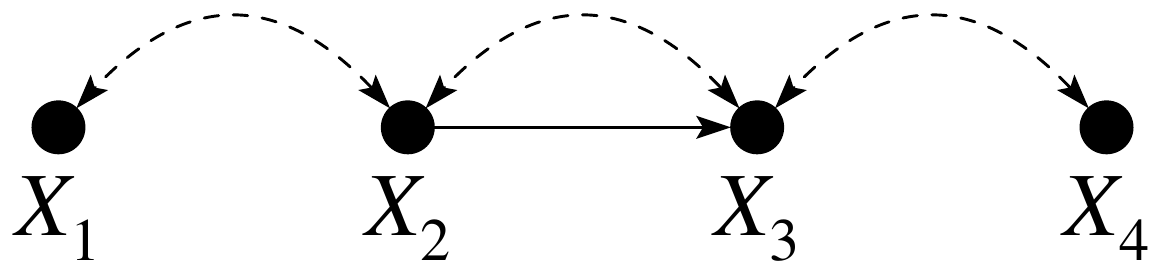}
        \caption{A causal DAG \(\G\).}
        \label{fig:intro_basis_dag}
    \end{subfigure}
    \begin{subfigure}{0.45\textwidth}
        \begin{tikzpicture}[thick, scale=0.7,
                    arrow/.style={-{Latex[length=3mm]}, shorten >=3pt, shorten <=3pt}]

\node (X4X2) [draw, rectangle, fill=cyan!50] at (0, 1.5) {$X_4 \indep X_2$};
\node (X4X1) [draw, rectangle, fill=orange!50] at (8, 1.5) {$X_4 \indep X_1$};
\node (X4X1X2) [draw, rectangle, fill=red!50] at (4, 0) {$X_4 \indep \{X_1, X_2\}$};
\node (X1X4X2) [draw, rectangle, fill=cyan!50] at (0, -1.5) {$X_1 \indep X_4 \mid X_2$};
\node (X2X4X1) [draw, rectangle, fill=orange!50] at (8, -1.5) {$X_2 \indep X_4 \mid X_1$};

\node (dec) [fill=white, inner sep=1pt] at (4,1.5) {\footnotesize decomposition};
\node (wu) [fill=white, inner sep=1pt] at (4,-1.5) {\footnotesize weak union};
\node (cont1) [fill=white, inner sep=1pt] at (0,0) {\footnotesize contraction};
\node (cont2) [fill=white, inner sep=1pt] at (8,0) {\footnotesize contraction};

\begin{pgfonlayer}{background}

\draw (X4X2) -- (X1X4X2);
\draw[arrow] (cont1) -- (X4X1X2);
\draw (X4X1) -- (X2X4X1);
\draw[arrow] (cont2) -- (X4X1X2);
\draw (X4X1X2) -- (dec);
\draw[arrow] (dec) -- (X4X2);
\draw[arrow] (dec) -- (X4X1);
\draw (X4X1X2) -- (wu);
\draw[arrow] (wu) -- (X1X4X2);
\draw[arrow] (wu) --  (X2X4X1);
\end{pgfonlayer}
\end{tikzpicture}
        \caption{CIs encoded in \(\G\).}
        \label{fig:intro_basis_axioms}
    \end{subfigure}
    \caption{
    A causal DAG \(\G\) and a hyper-graph depicting all the CIs encoded in \(\G\). An edge indicates that we can derive the CIs at the arrowheads from the CIs at the tails using the semi-graphoid axioms. The highlighted subsets of CIs in blue, orange, and pink are each sufficient to derive all other CIs in the graph.}
    \label{fig:intro_basis}
\end{figure}

\begin{adxdefinition}[Markov Relative \cite{bareinboim:etal20}]
\label{adxdef:markovrel}
    An observational distribution \(P(\*v)\) is said to be Markov relative to a graph \(\G\) (over a set of variables \(\*V\)) if, for a given ordering \(V_1, \dots, V_n\) consistent with \(\G\), \(P(\*v)\) factorizes as
    \begin{align}
        P(\*v) = \prod_{V_i \in \*V}p(v_i \mid pa^-_i)
    \end{align}
    where \(Pa^-_i = \Pa{\{V_i\}}_\G \setminus \{V_i\}\).
\end{adxdefinition}
For example, in the simple three-node graph \(X \to Y \to Z\), the observational distribution $P(x,y,z)$ factorizes as 
\begin{align*}
    P(x)P(y\mid x)P(z \mid y,x) = P(x)P(y\mid x)P(z \mid y)   
\end{align*}

By means of this factorisation, a graph imposes CI constraints on the distribution \(P(\*v)\): in our example, \(Z \indep X \mid Y\). This gives an equivalent definition of `Markov relative': \(P(\*v)\) is Markov relative to \(\G\) if, for a given ordering \(\*V^\prec\) consistent with \(\G\), and for each \(V_i \in \*V\),
\begin{align*}
    V_i \indep \*V^{\leq V_i} \setminus \Pa{\{V_i\}} \mid \Pa{\{V_i\}} \setminus \{V_i\}
\end{align*}

Notice how this set of CI constraints is identical to C-LMP in the Markovian case (Eq.~(\ref{eq:clmpcimarkov})), discussed in Section~\ref{sec:reformulation}. Therefore, if \(P(\*v)\) is Markov relative to a given \(\G\), all the CI constraints encoded in \(\G\) must hold in \(P(\*v)\).

If \(\G\) contains bidirected edges, the factorisation in Def.~\ref{adxdef:markovrel} no longer applies. For instance, a variable \(V_i\) may be connected to a non-descendant \(V_j\) by means of a bidirected edge, \(V_i\) is not independent of \(V_j\) when conditioning on the (observed) parents \(Pa^-_i\). This leads to the more general definition of compatibility for semi-Markovian graphs, given below.

\begin{adxdefinition}[Semi-Markov Relative \cite{bareinboim:etal20}] \label{adxdef:semimarkovrel}
    An observational distribution \(P(\*v)\) is said to be Semi-Markov relative to a graph \(\G\) (over a set of variables \(\*V\)) if, for every ordering \(V_1, \dots, V_n\) consistent with \(\G\), \(P(\*v)\) factorizes as
    \begin{align}
        P(\*v) = \prod_{V_i \in \*V}p(v_i \mid pa^+_i)
    \end{align}
    where \(Pa^+_i = \Pa{\&{C}(V_i)_{\G_{\*V^{\leq V_i}}}}_\G \setminus \{V_i\}\).
\end{adxdefinition}
\begin{figure*}[t]
    \centering
    \begin{subfigure}{0.25\textwidth}
        \includegraphics[width=\textwidth]{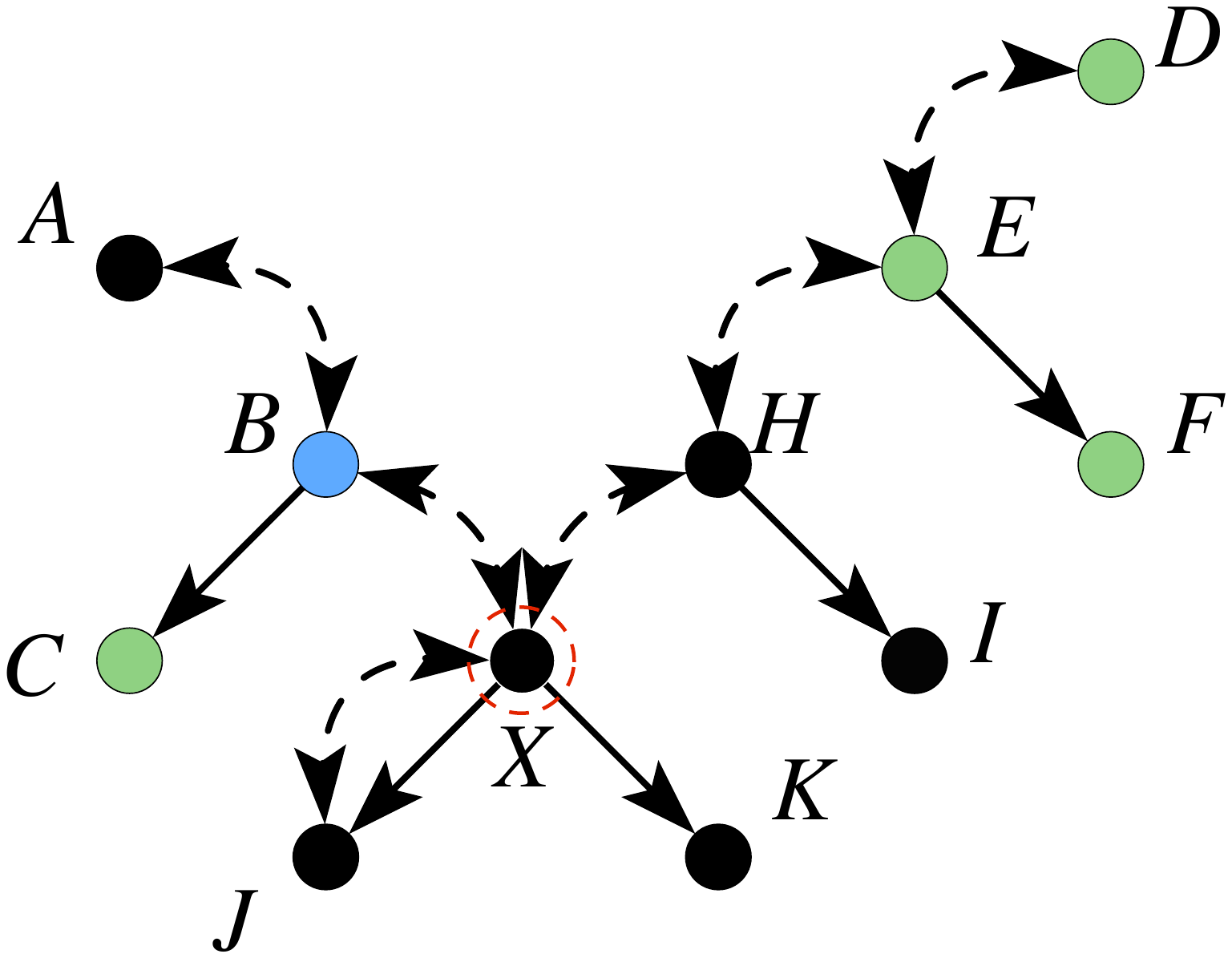}
        \caption{\(X\) is separated from \(C\) but not \(A\) when conditioning on \(B\).}
        \label{adxfig:clmp_mb1}
    \end{subfigure}
    \hfill
    \begin{subfigure}{0.25\textwidth}
        \includegraphics[width=\textwidth]{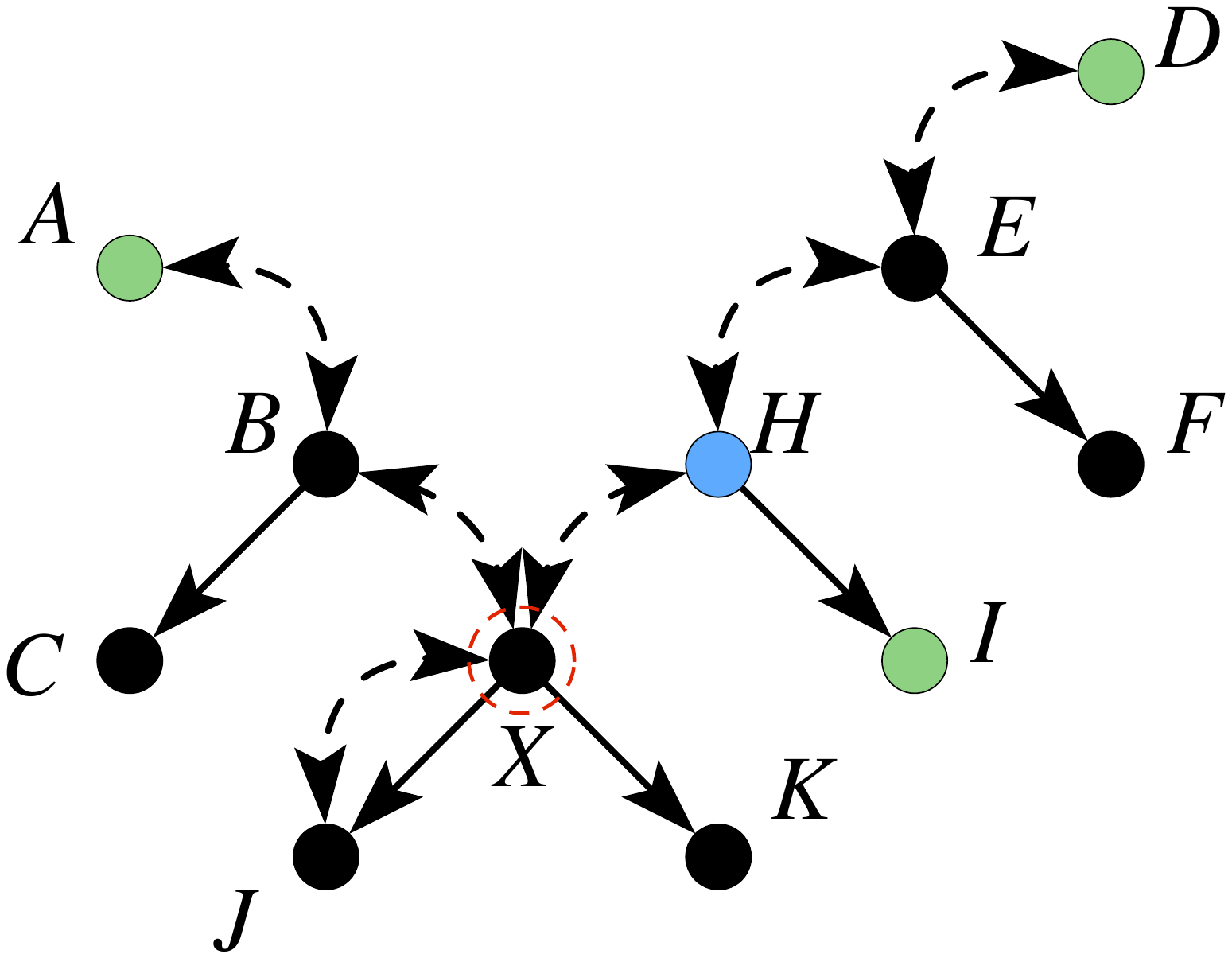}
        \caption{\(X\) is separated from \(A\) but not \(C\) when not conditioning on \(B\).}
        \label{adxfig:clmp_mb2}
    \end{subfigure}
    \hfill
    \begin{subfigure}{0.25\textwidth}
        \includegraphics[width=\textwidth]{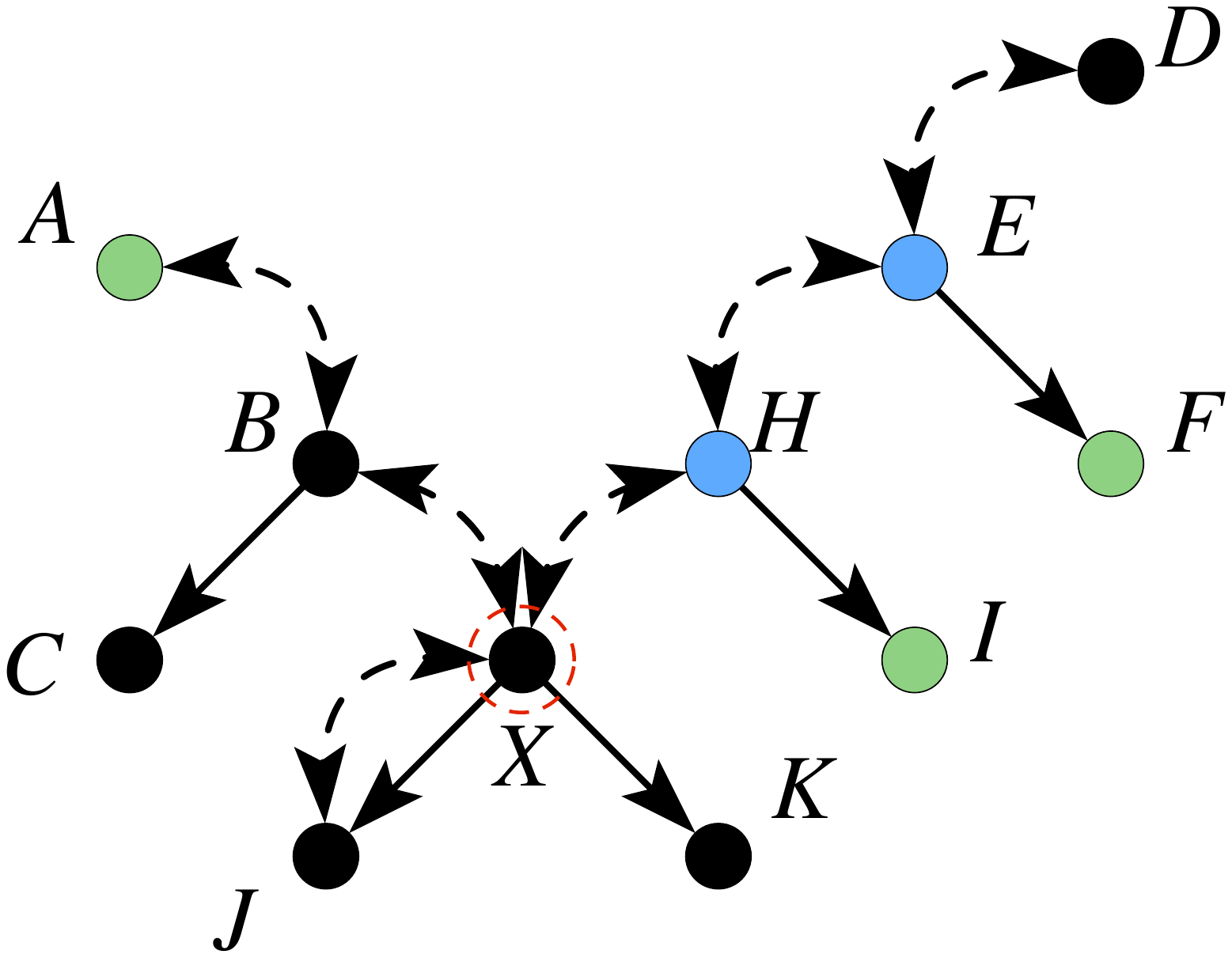}
        \caption{\(X\) is separated from \(F,I\) but not \(D\) when conditioning on \(H,E\).}
        \label{adxfig:clmp_mb3}
    \end{subfigure}
\caption{Fig.~\ref{fig:clmp_mb} reproduced for convenience. Three ACs relative to the variable \(X\) in the (same) causal DAG \(\G\).  Assume an ordering \(A \prec B \prec \dots \prec X \prec J \prec K\). 
The ACs relative to \(X\) (excluding \(\{X\}\) itself), shown in blue, separate it from the variables shown in green.
}
\label{adxfig:clmp_mb}
\end{figure*}
\begin{adxexample}
\label{adxexample:semimarkovrel}
    Consider the semi-Markovian graph \(\G\) in Fig.~\ref{fig:intro_basis_dag}. There are 12 possible orderings of the 4 nodes; each ordering induces  a factorisation of \(P(\*v)\) in Def.~\ref{adxdef:semimarkovrel}. We give four examples.
    \begin{enumerate}
        \item \(X_1 \prec X_2 \prec X_3 \prec X_4\). This implies \(Pa^+(\{X_1\}) = \emptyset, \ Pa^+(\{X_2\}) = \{X_1\}, \ Pa^+(\{X_3\}) = \{X_1,X_2\}, \ Pa^+(\{X_4\}) = \{X_1, X_2, X_3\}\).
        The resultant semi-Markov factorisation is \(p(x_1,x_2,x_3,x_4) = p(x_1)p(x_2\mid x_1)p(x_3 \mid x_1,x_2)p(x_4 \mid x_1,x_2,x_3)\).
        This is equivalent to the factorisation given by the chain rule, and implies no CI constraints.
        \item  \(X_1 \prec X_2 \prec X_4 \prec X_3\). This implies \(Pa^+(\{X_1\}) = \emptyset, \ Pa^+(\{X_2\}) = \{X_1\}, \ Pa^+(\{X_3\}) = \{X_1,X_2,X_4\}, \ Pa^+(\{X_4\}) = \emptyset\).
        The resultant semi-Markov factorisation is \(p(x_1,x_2,x_3,x_4) = p(x_1)p(x_2\mid x_1)p(x_4)p(x_3 \mid x_1,x_2,x_4)\), which implies the CI constraint \(X_4 \indep \{X_1,X_2\}\).
        \item \(X_2 \prec X_4 \prec X_1 \prec X_3\). This implies \(Pa^+(\{X_2\}) = \emptyset, \ Pa^+(\{X_4\}) = \emptyset, \ Pa^+(\{X_1\}) = \{X_2\}, \ Pa^+(\{X_3\}) = \{X_1,X_2,X_4\}\).
        The resultant semi-Markov factorisation is \(p(x_1,x_2,x_3,x_4) = p(x_2)p(x_4)p(x_1\mid x_2)p(x_3 \mid x_1,x_2,x_4)\), which implies the CI constraints \(X_4 \indep \{X_2\}\) and \(X_1 \indep \{X_4\} \mid\{X_2\}\).
        \item \(X_4 \prec X_1 \prec X_2 \prec X_3\). This implies \(Pa^+(\{X_4\}) = \emptyset, \ Pa^+(\{X_1\}) = \emptyset, \ Pa^+(\{X_2\}) = \{X_1\}, \ Pa^+(\{X_3\}) = \{X_1,X_2,X_4\}\).
        The resultant semi-Markov factorisation is \(p(x_1,x_2,x_3,x_4) = p(x_4)p(x_1)p(x_2\mid x_1)p(x_3 \mid x_1,x_2,x_4)\), which implies the CI constraints \(X_1 \indep \{X_4\}\) and \(X_2 \indep \{X_4\} \mid\{X_1\}\).
    \end{enumerate}
    The CIs induced by the first ordering (namely, none) are clearly insufficient to derive all CIs encoded in \(\G\), similar to Ex.~\ref{ex:onecnoci}. The four orderings above are chosen to be representative. Each of the eight remaining orderings induces exactly the same CIs as one the four orderings.
\qed
\end{adxexample}

We define the exact set of CI constraints implied by the semi-Markov factorisation below.

\begin{adxdefinition}[Semi-Markov Relative CI Constraints]
\label{adxdef:semimarkovrelci}
    Let \(\G\) be a causal graph over variables \(\*V\) and \(P(\*v)\) a probability distribution over \(\*V\) that is semi-Markov relative to \(\G\). Then, the conditional independence constraints encoded in the factorisation of \(P(\*v)\) are given by:
    For every ordering \(\*V^\prec\) consistent with \(\G\), for every variable \(V_i \in \*V\),
    \begin{equation*}
        V_i \indep \*V^{\leq V_i} \setminus (Pa^+(\{V_i\}) \cup \{V_i\}) \mid Pa^+(\{V_i\})
    \end{equation*}
    where \(Pa^+(\{V_i\}) = \Pa{\&{C}(V_i)_{\G_{\*V^{\leq V_i}}}}_{\G_{\*V^{\leq V_i}}} \setminus \{V_i\}\) and \(\*V^{\leq V_i}\) depends on the ordering \(\*V^{\prec}\).
\end{adxdefinition}
Note that we take a union over all orderings in Def.~\ref{adxdef:semimarkovrelci}.

\begin{adxexample}
\label{adxexample:semimarkovfactorbasis}
    Continuing Ex.~\ref{adxexample:semimarkovrel}. The set of CIs induced by the semi-Markov factorization for \(\G\) (Fig.~\ref{fig:intro_basis_dag}) is the union of all CIs listed in Ex.\ref{adxexample:semimarkovrel}: \(X_4 \indep \{X_1,X_2\}, X_4 \indep \{X_2\},X_1 \indep \{X_4\} \mid \{X_2\},X_1 \indep \{X_4\}, X_2 \indep \{X_4\} \mid \{X_1\}\).
\qed
\end{adxexample}

Ex.~\ref{adxexample:semimarkovrel} and Ex.~\ref{adxexample:semimarkovfactorbasis} show an important contrast between the Markovian and semi-Markovian cases. In the Markovian case, we can fix an arbitrary ordering: compatibility requires that \(P(\*v)\) factorize according to the product in Def.~\ref{adxdef:markovrel} for any one ordering.
In the semi-Markovian case, we cannot fix an arbitrary ordering; the ordering \(X_1 \prec X_2 \prec X_3 \prec X_4\) in Ex.~\ref{adxexample:semimarkovrel} provides no CI constraints.
Coincidentally, the ordering \(X_1 \prec X_2 \prec X_4 \prec X_3\) does suffice to derive all CIs encoded in \(\G\) from \(X_4 \indep \{X_1,X_2\}\), as seen in Fig.~\ref{fig:intro_basis_axioms}.
However, no method is known for choosing an ordering (or subset of orderings) that a priori guarantees that all CIs encoded in the graph can be derived from the resulting factorisation(s).
Therefore, in the semi-Markovian case, it is required that \(P(\*v)\) factorizes according to the product in Def.~\ref{adxdef:semimarkovrel} for all possible orderings.

Applying Def.~\ref{adxdef:semimarkovrelci} to a Markovian graph \(\G\) reveals why considering all orderings is not necessary in the Markovian case. We make two observations for Markovian \(\G\):

\begin{enumerate}
    \item Each c-component in \(\G\) is a singleton. This means \(\&{C}(V_i)_{\G_{\*V^{\leq V_i}}} = \{V_i\}\).
    \item The parents of a variable precede it in every ordering, and do not depend on the ordering. This means \(\Pa{\{V_i\}}_{\G_{\*V^{\leq V_i}}} = \Pa{\{V_i\}}_\G\).
\end{enumerate}

Therefore, for any ordering \(\*V^{\prec}\) and any variable \(V_i \in \*V\), the set \(Pa^+(\{V_i\})\) simplifies to \(\Pa{\{V_i\}}_\G \setminus \{V_i\}\). The set of CIs induced by Def.~\ref{adxdef:semimarkovrelci} contains: for every ordering  \(\*V^\prec\) consistent with \(\G\), for every variable \(V_i \in \*V\),
\begin{equation}
\label{adxeq:basiscimarkov}
    V_i \indep \*V^{\leq V_i} \setminus \Pa{\{V_i\}}_\G \mid \Pa{\{V_i\}}_\G \setminus \{V_i\}
\end{equation}

Let \(\Phi\) denote this set of CIs. The set of CIs induced by the Markov factorisation for \(\G\) (Def.~\ref{adxdef:markovrel}) – which fixes one ordering – is a subset of \(\Phi\).
Moreover, contrast \(\Phi\) with the set of CIs induced by LMP.
LMP abstracts away the ordering of variables.
Since \(\Nd{\{V_i\}} = \bigcup\limits_{\prec \text{ ordering of } \G} \*V^{\leq V_i}_{\prec} \setminus \{V_i\}\), LMP tests the CI: \(V_i \indep \Nd{\{V_i\}} \setminus \Pa{\{V_i\}}_\G \mid (\Pa{\{V_i}\}_\G \setminus \{V_i\})\).
This CI implies the CI in Eq.~(\ref{adxeq:basiscimarkov}) for every possible ordering by the decomposition axiom.

Therefore, we have another perspective on the combinatorial explosion of the number of CIs in the semi-Markovian case, relative to the Markovian case. This explosion was introduced in Section~\ref{sec:reformulation}, and characterised in terms of ACs. Here, we understand it through the many possible orderings of a given graph. To tie together these two concepts, we show an equivalence between C-LMP and the CI constraints invoked by the semi-Markov factorization.

\begin{adxprop}
\label{adxprop:equivalence:factorclmp}
    Given a causal graph \(\G\) over a set of variables \(\*V\) and a consistent ordering \(\*V^\prec\), let \(\mathcal{L}^C\) denote the set of conditional independence constraints that the c-component local Markov property invokes for \(\G\) with respect to \(\*V^\prec\) and \(\mathcal{L}^P\) denote the set of conditional independence constraints that the semi-Markov factorisation induces for \(\G\). Then, \(\mathcal{L}^C \subseteq \mathcal{L}^P\). Moreover, there exists \(\G, \*V^\prec\) for which \(\mathcal{L}^C \subsetneq \mathcal{L}^P\).
\end{adxprop}

\begin{proof}
    Consider a CI statement in \(\&{L}^C\) of the form
    \[
        X \indep \*S^{+} \setminus \Pa{\*C} \mid \Pa{\*C} \setminus \{X\} \text{, where}
    \]
    \[
        \*S^{+} = \*V^{\leq X}_\prec \setminus \De{\Spo{\*C} \setminus \Pa{\*C}}
    \]
    for some variable \(X \in \*V^\prec\) and AC \(\*C\in \&{AC}_X\). By Def.~\ref{def:ancestralccomponent}, there exists an ancestral set \(\*S \in \&{S}_X\) such that \(\*C = \&{C}(X)_{\G_{\*S}}\). By Props.~\ref{prop:equivalence:csplus} and ~\ref{prop:mkbk-ccomp}, \(\*S^{+}\) is ancestral and \(\*C = \&{C}(X)_{\G_{\*S^+}}\).

    First, we construct an ordering \(\*V^{\prec_*}\) under which \(\*C = \&{C}(X)_{\G_{\*V_{\prec_*}^{\leq X}}}\) using a `pivot' technique about \(X\).
    Given $\prec^*, \*S^+$, initialise $\prec_* = \prec$. We re-order $\prec_*$ as follows. Let the pivot $P = X$. For each $Y \in \*V^{\leq X}_{\prec} \backslash \*S^+$ in order of $\prec$, move $Y$ to immediately succeed $P$ in $\prec$ and update $P=Y$.
    Then, $\prec_*$ is a valid ordering since $\*S^+$ is ancestral. Moreover, $\*S^+ = \*V^{\leq X}_{\prec_*}$ and hence \(\*C = \&{C}(X)_{\G_{\*V_{\prec_*}^{\leq X}}}\). 
 
    By definition, \(\mathcal{L}^P\) contains the CI
    \begin{equation*}
        X \indep \*V^{\leq X}_{\prec_*} \setminus (Pa^+(\*C) \cup \{X\}) \mid Pa^+(\{V_i\})
    \end{equation*}
    where \(Pa^+(\{V_i\}) = \Pa{\&{C}(V_i)_{\G_{\*V_{\prec_*}^{\leq V_i}}}}_{\G_{\*V_{\prec_*}^{\leq V_i}}}\).
    Thus, the given CI from \(\&{L}^C\) has an identical counterpart in \(\mathcal{L}^P\).
    
    The graph \(\G\) in Fig.~\ref{fig:intro_basis_dag} with the ordering \(X_1 \prec X_2 \prec X_4 \prec X_3\) provides an example where \(\mathcal{L}^C \subsetneq \mathcal{L}^P\). We have  \(\&{L}^C = \{X_4 \indep \{X_1, X_2\}\}\). However,  \(\mathcal{L}^P = \{X_4 \indep \{X_1,X_2\}, X_4 \indep \{X_2\},X_1 \indep \{X_4\} \mid \{X_2\},X_1 \indep \{X_4\}, X_2 \indep \{X_4\} \mid \{X_1\}\}\), as shown in Ex.~\ref{adxexample:semimarkovfactorbasis}.
\end{proof}

We reproduce Fig.~\ref{adxfig:clmp_mb} to demonstrate the `pivot' technique used in the proof above.

\begin{adxexample}
\label{adxexample:clmporders}
    Continuing Ex.~\ref{example:3mbs}, we demonstrate the construction used in the proof of Prop.\ref{adxprop:equivalence:factorclmp} for the graph \(\G\) in Fig.~\ref{adxfig:clmp_mb}. We fix the ordering \(A \prec B \prec C \prec D \prec E \prec F \prec H \prec I \prec X \prec J \prec K\) for C-LMP. 
    \begin{enumerate}
        \item Fig.~\ref{adxfig:clmp_mb1} depicts the CI \(X \indep \{C,D,E,F\} \mid \{B\}\). Here, \(\*C = \Pa{\*C} = \{X,B\}, \*S^+ = \{B,C,D,E,F,X\}\). We construct the ordering \(\prec_*:  B \prec C \prec D \prec E \prec F  \prec X \prec A \prec H \prec I \prec J \prec K\).
        \item Fig.~\ref{adxfig:clmp_mb2} depicts the CI \(X \indep \{A,D,I\} \mid \{H\}\). Here, \(\*C = \Pa{\*C} = \{X,H\}, \*S^+ = \{A,D,H,I,X\}\). We construct the ordering \(\prec_*:  A  \prec D  \prec H \prec I \prec X \prec B \prec C \prec E \prec F \prec J \prec K\).
        \item Fig.~\ref{adxfig:clmp_mb3} depicts the CI \(X \indep \{A,F,I\} \mid \{H,E\}\). Here, \(\*C = \Pa{\*C} = \{X,H,E\}, \*S^+ = \{A,E,F,H,I,X\}\). We construct the ordering \(\prec_*: A \prec E \prec F \prec H \prec I \prec X \prec B \prec C \prec D \prec J \prec K\).
    \end{enumerate}
    Each  ordering \(\prec_*\) constructed for the given \(\*C, \*S^+\) implies \(\*V^{\leq X}_{\prec_*} = \*S^+\) and \(Pa^+(\{X\}) = \Pa{\*C} \setminus \{X\}\) in Def.~\ref{adxdef:semimarkovrelci}.
\qed
\end{adxexample}
Prop.~\ref{adxprop:equivalence:factorclmp} thus implies the following corollary.

\begin{adxcorollary}
    Let \(\G\) be a causal graph, \(\*V^\prec\) a consistent ordering, and \(P(\*v)\) a probability distribution over the set of variables \(\*V\). Then, the following conditions are equivalent.
    \begin{itemize}
        \item[] \emph{(G)} \(P(\*v)\)  satisfies the global Markov property for \(\G\).
        \item[] \emph{(L)} \(P(\*v)\)  satisfies the c-component local Markov property for \(\G\) with respect to \(\*V^\prec\).
        \item[] \emph{(F)}  \(P(\*v)\) is semi-Markov relative to \(\G\).
    \end{itemize}
\end{adxcorollary}
\begin{proof}
The equivalence of (G) and (L) follows from Thm.~\ref{thm:equivalence:gmp:lmpplus}.

(G) \(\implies\) (F). Given a DAG \(\G\) and a distribution \(P(\*v)\), we need to show \(P(\*v)\) factorizes according to Def.~\ref{adxdef:semimarkovrel} for every ordering \(\*V^\prec\) consistent with \(\G\). Fix an arbitrary ordering \(\*V^\prec\). Using the chain rule, we factorize
\begin{align*}
    P(\*v) = \prod_{V_i \in \*V^\prec} p(v_i \mid v_1,\dots, v_{i-1})
\end{align*}
Then, let \(Pa^+(\{V_i\}) = \Pa{\&{C}(V_i)_{\G_{\*V^{\leq V_i}}}}_{\G_{\*V^{\leq V_i}}} \setminus \{V_i\}\). It suffices to show that 
\[
    V_i \perp_d \*V^{\leq V_i} \setminus (Pa^+(\{V_i\}) \cup \{V_i\}) \mid Pa^+(\{V_i\}) \text{   in   } \G
\]
Since \(\*V^{\leq V_i}\) is an ancestral set, \(\*C = \&{C}(V_i)_{\G_{\*V^{\leq V_i}}}\) is an AC relative to \(V_i\). By definition, \(\Spo{\*C} \setminus \*C = \emptyset\), hence \(\*V^{\leq V_i} \setminus \De{\Spo{\*C} \setminus \Pa{\*C}} = \*V^{\leq V_i}\). By Prop.~\ref{adxprop:clmpsep}, we get the required d-separation. Since \(P(\*v)\) satisfies the global Markov property for \(\G\), this d-separation implies that 
\[
    V_i \indep \*V^{\leq V_i} \setminus (Pa^+(\{V_i\}) \cup \{V_i\}) \mid Pa^+(\{V_i\}) \text{   in   } P(\*v).
\]
This allows us to simplify the factorisation of \(P(\*v)\) to 
\begin{align*}
    P(\*v) = \prod_{V_i \in \*V^\prec} p(v_i \mid pa^+_i)
\end{align*}

(F) \(\implies\) (C). If \(P(\*v)\) is semi-Markov relative to \(\G\), then each of the semi-Markov relative CIs of \(\G\) (Def.~\ref{adxdef:semimarkovrelci}) must hold in \(P(\*v)\). Since the C-LMP CIs of \(\G\)with respect to \(\*V^\prec\) are a subset of the semi-Markov relative CIs (Prop.~\ref{adxprop:equivalence:factorclmp}), the C-LMP CIs must necessarily hold in  \(P(\*v)\).
\end{proof}

\begin{figure*}[t]
    \centering
    \begin{subfigure}{0.48\textwidth}
        \centering
        \includegraphics[width=\textwidth]{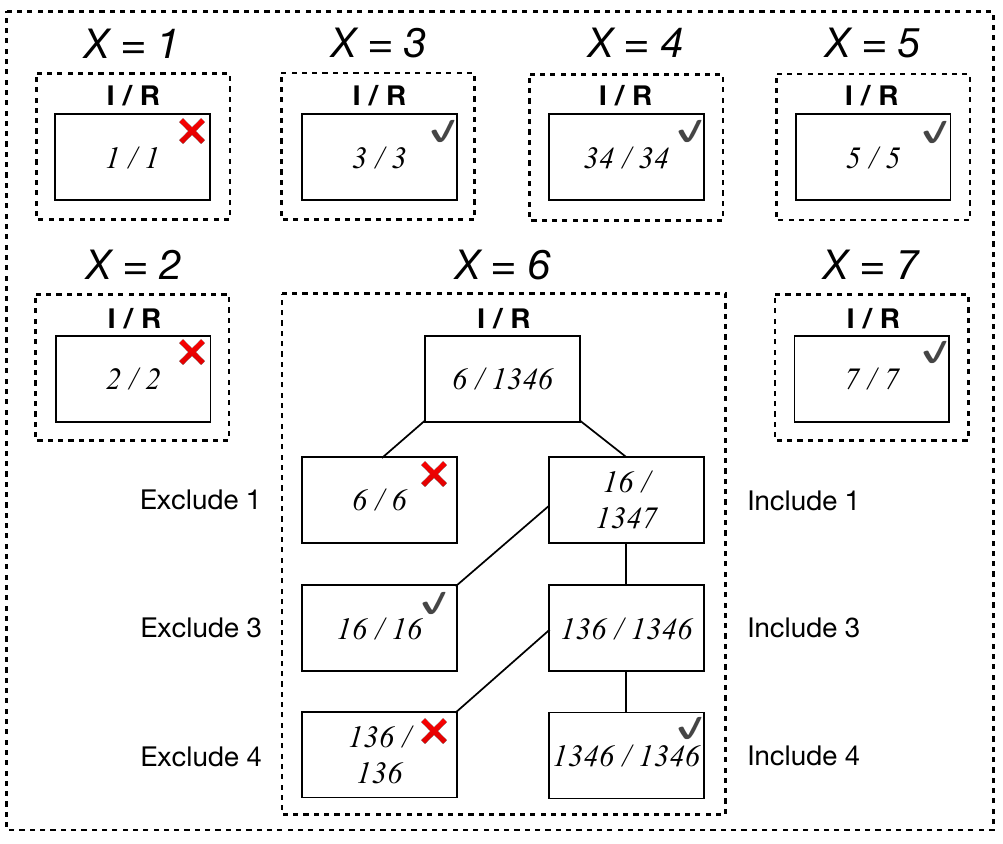}
        \caption{A set of search trees, one for each \(X \neq J\)}
        \label{fig:listci:tree_nonj}
    \end{subfigure}
    \hfill
    \begin{subfigure}{0.48\textwidth}
        \centering
        \includegraphics[width=\textwidth]{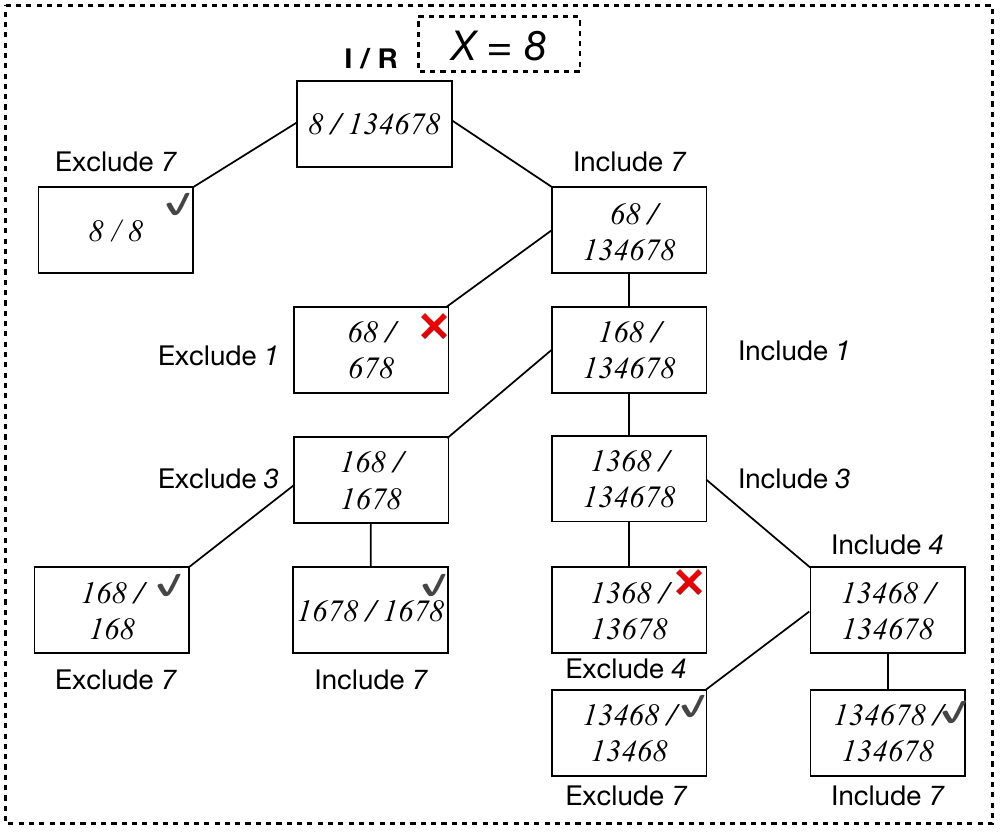}
        \caption{A search tree for \(X = J\)}
        \label{fig:listci:tree_j}
    \end{subfigure}
\caption{
    A set of search trees illustrating the running of \textsc{ListCI} in Ex.~\ref{ex:listci}.
}
\label{fig:listci:tree_full}
\end{figure*}

\subsection{Examples} 
\label{sec:appendix:examples}

The following example shows that total number of vacuous CIs invoked by C-LMP may be exponential with respect to the number of nodes in a graph.

\begin{adxexample}
\label{ex:exp_no_ci}
    Consider the three causal graphs in Fig.~\ref{fig:exp_no_ci}, each of which is a bidirected clique encoding no CIs.
    For \(\G^{b1}\) shown in Fig.~\ref{fig:exp_no_ci_1}, C-LMP invokes 7 vacuous CIs:
    \(
        A_1 \indep \emptyset,
        A_2 \indep \emptyset,
        A_2 \indep \emptyset \mid \{A_1\},
        A_3 \indep \emptyset,
        A_3 \indep \emptyset \mid \{A_1\},
        A_3 \indep \emptyset \mid \{A_2\},
        A_3 \indep \emptyset \mid \{A_1,A_2\}
    \).
     \(\G^{b2}\) shown in Fig.~\ref{fig:exp_no_ci_2} is constructed by adding one variable \(A_4\) to \(\G^{b1}\) with three bidirected edges \(A_4 \leftrightarrow A_i\) for \(1 \leq i \leq 3\).
     C-LMP invokes 15 vacuous CIs for \(\G^{b2}\). 
    If we add another variable \(A_5\) to \(\G^{b2}\) with four bidirected edges \(A_5 \leftrightarrow A_ii\) for \(1 \leq i \leq 4\), C-LMP invokes 31 vacuous CIs.
    As shown in Fig.~\ref{fig:exp_no_ci_3} which generalizes this pattern to variables \(\{A_1, \cdots, A_n\}\) with bidirected edges between every \(A_i\) and \(A_j\) with \(1 \leq i ,j \leq n, i \neq j\), C-LMP invokes \(2^n -1\) vacuous CIs.
\qed
\end{adxexample}

The following example expands on Ex.~\ref{ex:listcix:tree}.
We demonstrate the execution of \textsc{ListCI}(\(\G^3, \*V^\prec\)) with \(\G^3\) shown in Fig.~\ref{fig:listci} and \(\*V^\prec = \{A,B,C,D,E,F,H,J\}\).
The full search tree for the execution of \textsc{ListCI} (in Ex.~\ref{ex:listci}) is given in Fig.~\ref{fig:listci:tree_full}.

\begin{adxexample}
    Expanding Ex.~\ref{ex:listcix:tree}.
    Let \(\G^3\) be the causal graph shown in Fig.~\ref{fig:listci} and \(\*V^\prec = \{A,B,C,D,E,F,H,J\}\).
    We show a part of running \textsc{ListCI}(\(\G^3, \*V^\prec\)) with \(X = J\) starting from the root node \(\&N(\{J\},\{A,C,D,F,H,J\})\) to the leaf node \(\&L(\{A,F,J\},\{A,F,J\})\).

    Initially, the search starts from \(\&N\) which is constructed at line~\ref{alg:listci:calllistcix} of \textsc{ListCI} with \(X = J\), \(\*I = \{J\}\) and \(\*R = \{A,C,D,F,H,J\}\).
    At line~\ref{func:listcix:callfindadmissiblec} of \textsc{ListCIX}, \textsc{FindAAC} returns \(\{J\}\).
    With \(s = F\) and \(\*R' = \{J\}\), the recursive call \textsc{ListCIX}\((\G^3,J,\*V^\prec,\{J\},\{J\})\) is made at line~\ref{func:listcix:recursion:rprime}, spawning a child \(\&N_1(\{J\},\{J\})\).
    The search continues from \(\&N_1\).
    \textsc{FindAAC} returns \(\{J\}\).
    \(\&N_1\) is a leaf node, and \textsc{ListCIX} outputs a CI: \(J \indep \{A, B, C, D, E\}\) at line~\ref{func:listcix:outputci}.

\begin{figure}[ht]
    \centering
    \begin{subfigure}{0.15\textwidth}
        \includegraphics[width=\textwidth]{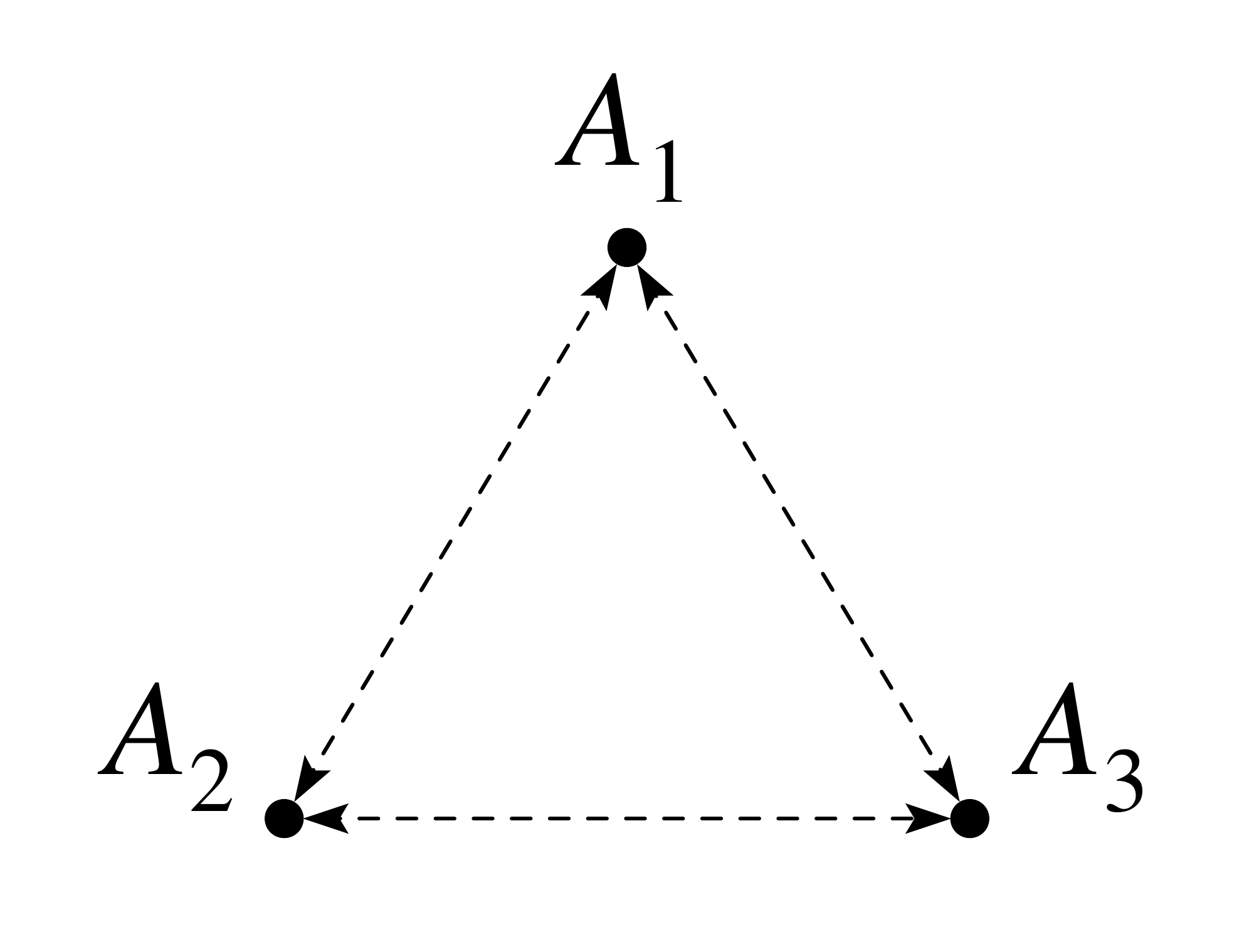}
        \caption{\(\G^{b1}\)}
        \label{fig:exp_no_ci_1}
    \end{subfigure}
    \hfill
    \begin{subfigure}{0.15\textwidth}
        \includegraphics[width=\textwidth]{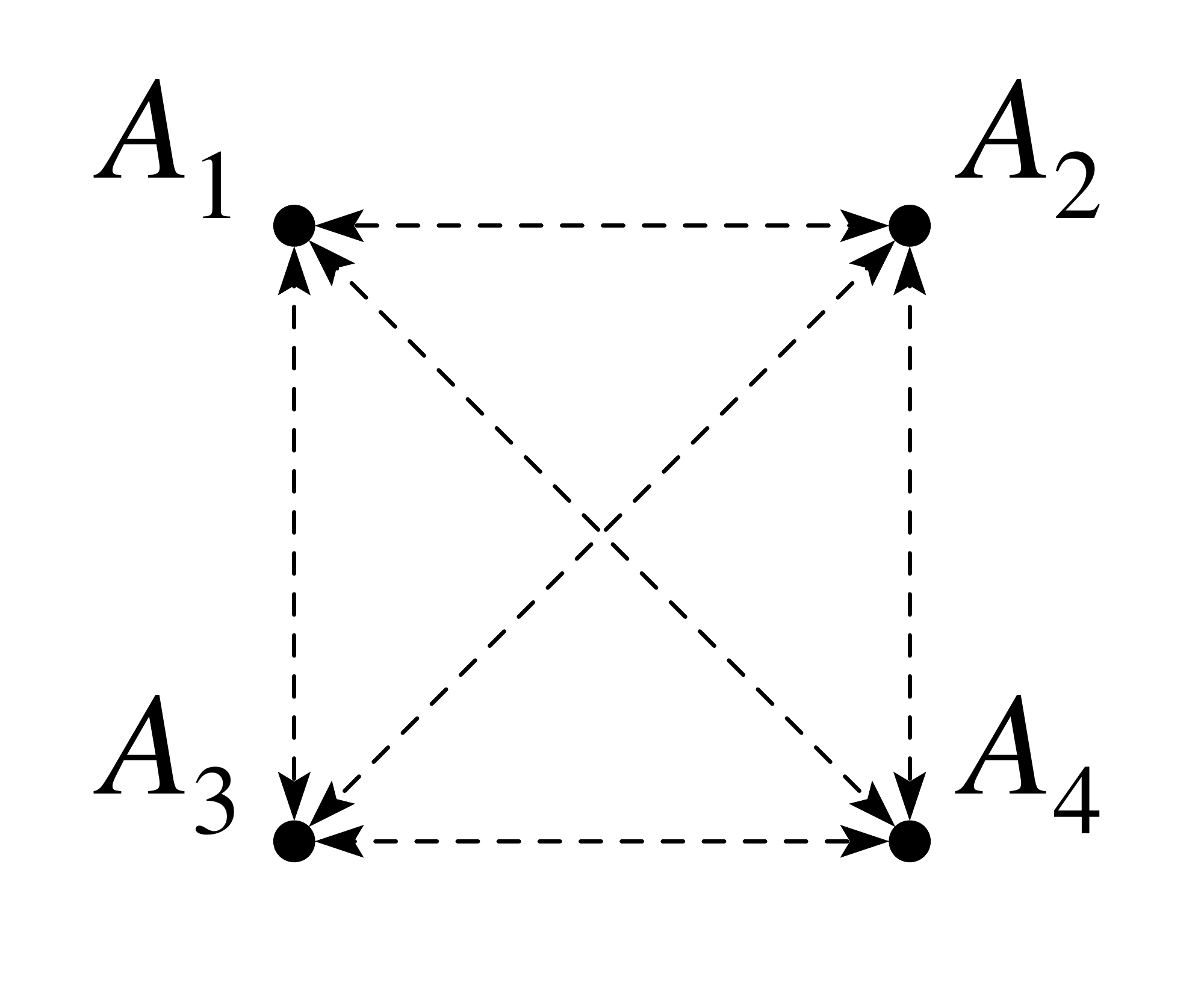}
        \caption{\(\G^{b2}\)}
        \label{fig:exp_no_ci_2}
    \end{subfigure}
    \hfill
    \begin{subfigure}{0.15\textwidth}
        \includegraphics[width=\textwidth]{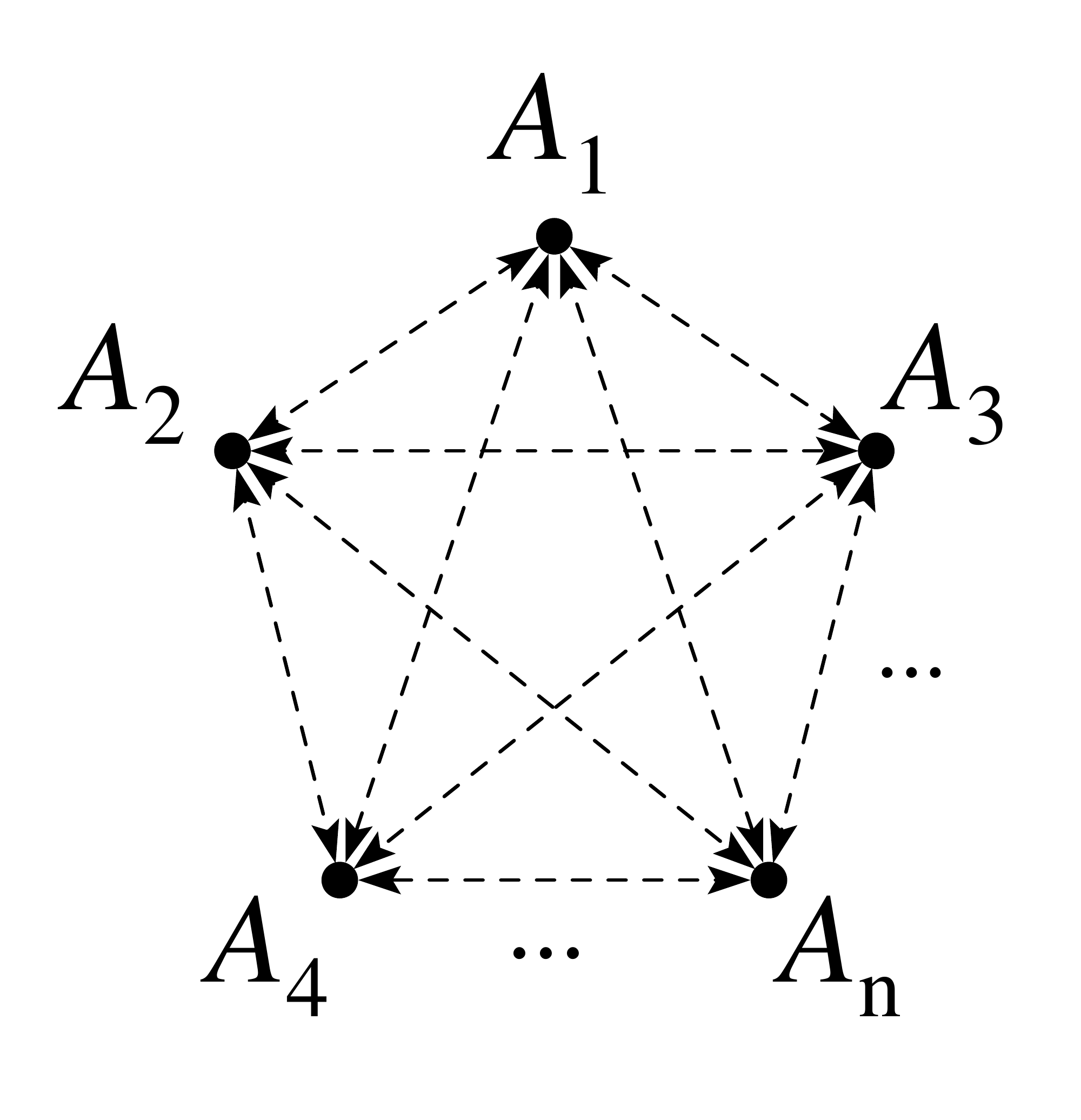}
        \caption{\(\G^{b3}\)}
        \label{fig:exp_no_ci_3}
    \end{subfigure}
\caption{
Three examples to demonstrate that total number of vacuous CIs invoked by C-LMP may be exponential with respect to the number of nodes in a graph.
}
\label{fig:exp_no_ci}
\end{figure}

    After, \textsc{ListCIX} backtracks to the parent \(\&N\).
    Then, with \(\*I' = \{F,J\}\) constructed at line~\ref{func:listcix:iprime}, a recursive call \textsc{ListCIX}\((\G^3,J,\*V^\prec,\{F,J\},\{A, C, D, F, H, J\})\) is made at line~\ref{func:listcix:recursion:rprime}, spawning a child \(\&N_2(\{F,J\},\{A, C, D, F, H, J\})\).
    At \(\&N_2\), \textsc{FindAAC} returns \(\{A, F, J\}\).
    With \(s = A\) and \(\*R' = \{F, H, J\}\), another recursive call \textsc{ListCIX}\((\G^3,J,\*V^\prec,\{F,J\},\{F, H, J\})\) is made at line~\ref{func:listcix:recursion:rprime}, spawning a child \(\&N_3(\{F,J\},\{F, H, J\})\).
    At \(\&N_3\), \textsc{FindAAC} returns \(\perp\), backtracking to \(\&N_2\).
    with \(\*I' = \{A,F,J\}\), a recursive call \textsc{ListCIX}\((\G^3,J,\*V^\prec,\{A,F,J\},\{A, C, D, F, H, J\})\) creates a child \(\&N_4(\{A,F,J\},\{A, C, D, F, H, J\})\).
    The recursion continues in the following order: \(\&N_4\) adds a child \(\&N_5(\{A,F,J\},\{A, F, H, J\})\) with \(s = C\) and \(\*R' = \{A,F,H,J\}\), and \(\&N_5\) adds a child \(\&N_6(\{A,F,J\},\{A, F, J\})\) with \(s = H\) and \(\*R' = \{A,F,J\}\).
    \(\&N_6 = \&L\) is a leaf node and \textsc{FindAAC} returns \(\{A, F, J\}\).
    Finally, \textsc{ListCIX} outputs a CI: \(J \indep \{B\} \mid \{A, F\}\) at line~\ref{func:listcix:outputci}.
\qed
\end{adxexample}

\section{Further Results}
\label{appendix:furtherresults}

We present a procedure \textsc{ListGMP} (Fig.~\ref{func:listgmp}) that lists all CIs invoked by GMP for a causal graph \(\G\) over a set of variables \(\*V\).
The following result states that \textsc{ListGMP} correctly lists all such CIs.

\begin{figure}[t]
\begin{algorithmic}[1]
\Function {ListGMP} {$\G, \*V$}
    \State {\bfseries Output:} Listing CIs invoked by GMP for \(\G\) over \(\*V\).

    \State \textbf{for} each \(\*X\) with \(\emptyset \subset \*X \subset \*V\) \textbf{do}
    \Indent
        \State \textbf{for} each \(\*Y\) with \(\emptyset \subset \*Y \subseteq \*V \setminus \*X\) \textbf{do}
        \Indent
            \State \textbf{for} \(\*Z\) with \(\emptyset \subseteq \*Z \subseteq \*V \setminus (\*X \cup \*Y)\) \textbf{do}
            \Indent
                \State Output \(\*X \indep \*Y \mid \*Z\)
            \EndIndent
        \EndIndent
    \EndIndent

\EndFunction
\end{algorithmic}
\caption{A function that lists all CIs invoked by GMP.}
\label{func:listgmp}
\end{figure}

\begin{adxlemma}[Correctness of \textsc{ListGMP}]
\label{adxlemma:listgmp}
    Given a causal graph \(\G\) over a set of variables \(\*V\), \textsc{ListGMP}\((\G,\*V)\) lists all and only all conditional independence relations invoked by the global Markov property for \(\G\).
\end{adxlemma}

\begin{proof}
    The proof follows by construction from Def.~\ref{def:gmp}.
\end{proof}

\section{Experimental Details}
\label{appendix:experiments}

All experiments were run on a machine with CPU: Apple M2 Chip, 16GB of RAM, and macOS operating system.
We used a single core for the experiments.
The algorithms are implemented in Python.

This section is organized as follows.
Section~\ref{subsection:experimenta} presents details of the runtime of \textsc{ListCI} and other algorithms.
Section~\ref{subsection:application} shows detailed result on testing a hypothesized model against a real-life dataset.
Section~\ref{subsection:experimentb} provides detailed analysis of the total number of non-vacuous CIs invoked by C-LMP.
We use \textsc{ListCI} for the analysis.

\subsection{Comparison of \textsc{ListCI} with Other Algorithms}
\label{subsection:experimenta}

We compare the runtime of \textsc{ListCI} with other two algorithms: \textsc{ListGMP} and \textsc{ListCIBF} over \texttt{bnlearn} instances.
The runtime of the algorithms across different levels of projection \(U \in \{0,20,40,60,80\}\) respectively, are shown in Tables~\ref{table:latent0} - \ref{table:latent80}.

Fig.~\ref{fig:experiments:results:detail} depicts further statistics on the results shown in Fig.~\ref{fig:experiments:results}. Each point corresponds to average runtime of the specified algorithm on a given graph and projection level \(U \in \{0,10,20,\dots,90\}\), from which 10 random projections were generated.
Only the cases where all 10 samples did not time out (\(>\) 1 hour) are shown.
Each dot is associated with an error bar where the 1st and 3rd quartiles across the 10 samples are given by the the lower and upper limits, respectively.

\begin{figure}[t]
    \centering
    \includegraphics[width=.48\textwidth]{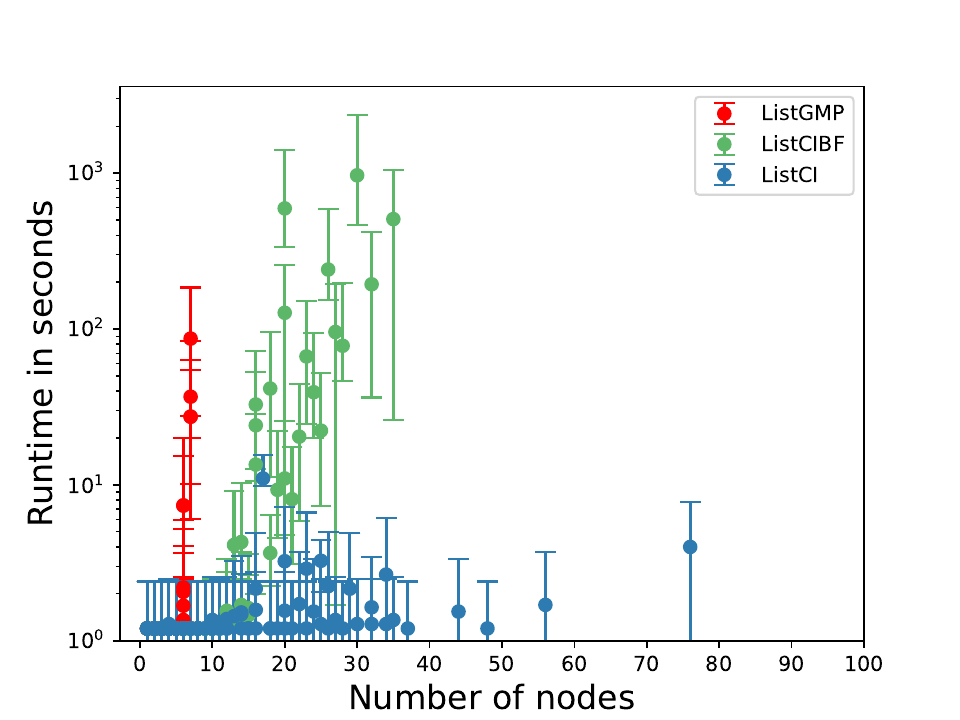}
    \caption{
    Plot of runtimes of the algorithms \textsc{ListGMP}, \textsc{ListCIBF}, and \textsc{ListCI} on graphs of various sizes.
    Each dot represents the average runtime of over 10 sample graphs for a given number of nodes and projection level \(U \in \{0,10,20,\dots,90\}\).
    Lower and upper limits of an error bar represent the 1st and 3rd quartiles respectively.
    The y-axis uses a logarithmic scale.
    }
    \label{fig:experiments:results:detail}
\end{figure}

\begin{table}[ht!]
    \footnotesize
    \centering
    \begin{tabular}{*{1}{l} *{5}{c}}
        \toprule
        \multicolumn{3}{c}{Graphs} & \multicolumn{3}{c}{Runtime (mm:ss)}\\
        \cmidrule(lr){1-3} \cmidrule(l){4-6}
        Name & n & m & \textsc{ListGMP} & \textsc{ListCIBF} & \textsc{ListCI} \\
        \midrule%
        asia    &   8   &   8   &   03:49   &   00:00   &   00:00 \\
        cancer  &   5   &   4   &   00:00   &   00:00   &   00:00 \\
        earthquake  &   5   &   4   &   00:00   &   00:00   &   00:00 \\
        sachs   &   11   &   17  &   -   &   00:00   &   00:00 \\
        survey  &   6   &   6   &   00:01   &   00:00   &   00:00 \\
        \midrule
        alarm   &   37   &   46   &   -   &   -   &   00:01 \\
        barley  &   48   &   84   &   -   &   -   &   00:01 \\
        child   &   20   &   25   &   -   &   00:46   &   00:00 \\
        insurance   &   27   &   52   &   -   &   00:54   &   00:00 \\
        mildew  &   35   &   46   &   -   &   04:10   &   00:00 \\
        water   &   32   &   66   &   -   &   -   &   00:00 \\
        \midrule
        hailfinder  &   56   &   66   &   -   &   -   &   00:01 \\
        win95pts  &   76   &   112   &   -   &   -   &   00:02 \\
        \bottomrule
    \end{tabular}
\caption{
No variables unobserved.
}
\label{table:latent0}

    \begin{tabular}{*{1}{l} *{5}{c}}
        \toprule
        \multicolumn{3}{c}{Graphs} & \multicolumn{3}{c}{Runtime (mm:ss)}\\
        \cmidrule(lr){1-3} \cmidrule(l){4-6}
        Name & n & m & \textsc{ListGMP} & \textsc{ListCIBF} & \textsc{ListCI} \\
        \midrule%
        asia    &   7   &   7   &   00:14   &   00:00   &   00:00 \\
        cancer  &   4   &   3   &   00:00   &   00:00   &   00:00 \\
        earthquake  &   4   &   3   &   00:00   &   00:00   &   00:00 \\
        sachs   &   9   &   14  &   -   &   00:00   &   00:00 \\
        survey  &   5   &   5   &   00:00   &   00:00   &   00:00 \\
        \midrule
        alarm   &   30   &   40   &   -   &   -   &   00:00 \\
        barley  &   39   &   88   &   -   &   -   &   00:01 \\
        child   &   16   &   24   &   -   &   00:05   &   00:00 \\
        insurance   &   22   &   57   &   -   &   00:06   &   00:00 \\
        mildew  &   28   &   45   &   -   &   00:29   &   00:00 \\
        water   &   26   &   78   &   -   &   13:22   &   00:01 \\
        \midrule
        hailfinder  &   45   &   84   &   -   &   -   &   06:06 \\
        win95pts  &   61   &   111   &   -   &   -   &   00:35 \\
        \bottomrule
    \end{tabular}
\caption{
20\% of variables unobserved.
}
\label{table:latent20}

    \begin{tabular}{*{1}{l} *{5}{c}}
        \toprule
        \multicolumn{3}{c}{Graphs} & \multicolumn{3}{c}{Runtime (mm:ss)}\\
        \cmidrule(lr){1-3} \cmidrule(l){4-6}
        Name & n & m & \textsc{ListGMP} & \textsc{ListCIBF} & \textsc{ListCI} \\
        \midrule%
        asia    &   5   &   6   &   00:00   &   00:00   &   00:00 \\
        cancer  &   3   &   2   &   00:00   &   00:00   &   00:00 \\
        earthquake  &   3   &   2   &   00:00   &   00:00   &   00:00 \\
        sachs   &   7   &   12  &   00:02   &   00:00   &   00:00 \\
        survey  &   4   &   4   &   00:00   &   00:00   &   00:00 \\
        \midrule
        alarm   &   23   &   35   &   -   &   -   &   00:00 \\
        barley  &   29   &   91   &   -   &   02:29   &   00:01 \\
        child   &   12   &   25   &   -   &   00:01   &   00:00 \\
        insurance   &   17   &   65   &   -   &   00:06   &   00:02 \\
        mildew  &   21   &   39   &   -   &   00:04   &   00:00 \\
        water   &   20   &   78   &   -   &   00:21   &   00:01 \\
        \midrule
        hailfinder  &   34   &   67   &   -   &   -   &   00:12 \\
        win95pts  &   46   &   98   &   -   &   -   &   02:07 \\
        \bottomrule
    \end{tabular}
\caption{
40\% of variables unobserved.
}
\label{table:latent40}
\end{table}

\begin{table}[ht!]
    \footnotesize
    \centering
    \begin{tabular}{*{1}{l} *{5}{c}}
        \toprule
        \multicolumn{3}{c}{Graphs} & \multicolumn{3}{c}{Runtime (mm:ss)}\\
        \cmidrule(lr){1-3} \cmidrule(l){4-6}
        Name & n & m & \textsc{ListGMP} & \textsc{ListCIBF} & \textsc{ListCI} \\
        \midrule%
        asia    &   4   &   3   &   00:00   &   00:00   &   00:00 \\
        cancer  &   2   &   1   &   00:00   &   00:00   &   00:00 \\
        earthquake  &   2   &   1   &   00:00   &   00:00   &   00:00 \\
        sachs   &   5   &   7  &   00:00   &   00:00   &   00:00 \\
        survey  &   3   &   2   &   00:00   &   00:00   &   00:00 \\
        \midrule
        alarm   &   15   &   27   &   -   &   00:10   &   00:00 \\
        barley  &   20   &   80   &   -   &   04:18   &   00:01 \\
        child   &   8   &   15   &   -   &   00:00   &   00:00 \\
        insurance   &   11   &   47   &   -   &   00:00   &   00:00 \\
        mildew  &   23   &   20   &   -   &   00:03   &   00:00 \\
        water   &   13   &   47   &   -   &   00:01   &   00:00 \\
        \midrule
        hailfinder  &   23   &   42   &   -   &   -   &   00:01 \\
        win95pts  &   31   &   53   &   -   &   -   &   00:14 \\
        \bottomrule
    \end{tabular}
\caption{
60\% of variables unobserved.
}
\label{table:latent60}

    \begin{tabular}{*{1}{l} *{5}{c}}
        \toprule
        \multicolumn{3}{c}{Graphs} & \multicolumn{3}{c}{Runtime (mm:ss)}\\
        \cmidrule(lr){1-3} \cmidrule(l){4-6}
        Name & n & m & \textsc{ListGMP} & \textsc{ListCIBF} & \textsc{ListCI} \\
        \midrule%
        asia    &   2   &   1   &   00:00   &   00:00   &   00:00 \\
        cancer  &   1   &   0   &   00:00   &   00:00   &   00:00 \\
        earthquake  &   1   &   0   &   00:00   &   00:00   &   00:00 \\
        sachs   &   3   &   2  &   00:00   &   00:00   &   00:00 \\
        survey  &   2   &   1   &   00:00   &   00:00   &   00:00 \\
        \midrule
        alarm   &   8   &   10   &   -   &   00:00   &   00:00 \\
        barley  &   10   &   23   &   -   &   00:01   &   00:01 \\
        child   &   4   &   6   &   00:01   &   00:00   &   00:00 \\
        insurance   &   6   &   20   &  00:02   &   00:00   &   00:00 \\
        mildew  &   7   &   14   &   00:19   &   00:00   &   00:00 \\
        water   &   7   &   12   &   00:49   &   00:00   &   00:00 \\
        \midrule
        hailfinder  &   12   &   24   &   -   &   00:05   &   00:00 \\
        win95pts  &   16   &   18   &   -   &   02:14   &   00:01 \\
        \bottomrule
    \end{tabular}
\caption{
80\% of variables unobserved.
}
\label{table:latent80}

\caption*{Table \ref{table:latent0} - \ref{table:latent80}:
Summary of runtime of algorithms over various graphs.
For each graph, the stated percent of variables were randomly chosen as unobserved, and the graph was projected onto the remaining observed variables.
Runtime is rounded to the nearest integer (second).
The symbol ``-'' indicates that the algorithm took over an hour on at least one sample graph.}

\end{table}

\subsection{Application to Model Testing}
\label{subsection:application}

In this section, we provide more details on our application of \textsc{ListCI} to the task of model testing in Section~\ref{section:experiments}.
Recall that we test an expert-provided ground-truth DAG (11 nodes and 16 edges, shown in Fig.~\ref{fig:ground_truth_graph}) against a real-world protein signaling dataset with 853 samples \cite{sachs2005causal}.
We present the details in the following example.

\begin{adxexample}
\label{ex:model_testing}
    Let \(\G\) be the ground-truth DAG shown in Fig.~\ref{fig:ground_truth_graph}, and fix the consistent ordering  
    \(\*V^\prec = \{PKA,PIP3,Plcg,Akt,PIP2,PKC,Raf,\)
    \(P38,Jnk,Mek,Erk\}\).
    GMP invokes 76580 CIs for \(\G\).
    A naive approach would be to test all these CIs against the data.
    In contrast, C-LMP invokes 10 CIs for \(\G\) with respect to \(\*V^\prec\), which together imply all CIs of the GMP.
    Therefore, C-LMP makes it possible to test the CIs encoded in \(\G\) against the data.
    The full list of CIs that C-LMP invokes is shown in Table~\ref{table:model_testing}.

    To generate and test these CIs, we call \textsc{ListCI}\((\G, \*V^\prec)\) and use a kernel-based CI test from the \texttt{causal-learn} package \cite{zheng2024causal} with p-value \(p = 0.05\) (for the null hypothesis of independence).

    As shown in Table~\ref{table:model_testing}, seven out of ten CIs invoked by C-LMP resulted in \(p > 0.05\).
\qed
\end{adxexample}

The test results show that \(\G\) may need to be revised, and the exact list of CIs that are violated may help experts in the revision process.
However, we note that significance testing (whether rejecting the null hypothesis or not) has its own limitations.
For example, selection of the level of significance impacts the probability of Type I error, and sample size of the dataset affects the likelihood of Type II error, especially for small datasets.

\begin{figure}[t]
    \centering
    \includegraphics[width=0.4\textwidth]{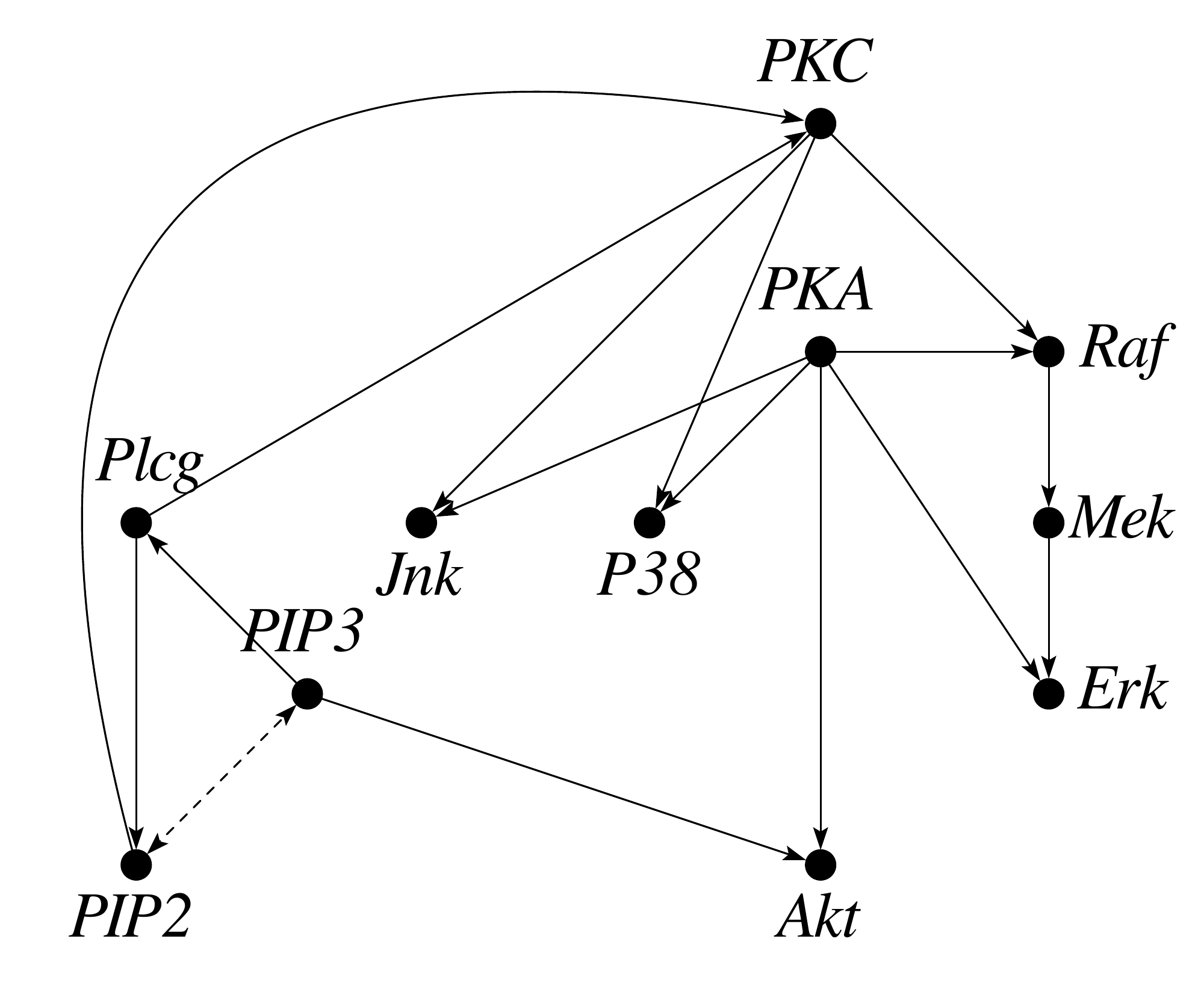}
    \caption{
    The ground-truth DAG (a protein-signaling network) sn \citep[Fig.~2]{sachs2005causal}.
    }
    \label{fig:ground_truth_graph}
\end{figure}

\begin{table}[t]
    \footnotesize
    \centering
    \begin{tabular}{*{1}{l} *{1}{c}}
        \toprule
        \multicolumn{1}{c}{CIs implied by \(\G\)} & \multicolumn{1}{c}{p-value}\\
        \midrule%
        \(PIP3 \indep PKA\) & 0.175 \\
        \midrule
        \(Plcg \indep PKA \mid PIP3\) & 0.081 \\
        \midrule
        \(Akt \indep Plcg \mid PIP3, PKA\) & 0.370 \\
        \midrule
        \(PIP2 \indep Akt, PKA \mid PIP3, Plcg\) & 0.648 \\
        \midrule
        \(PKC \indep Akt, PIP3, PKA \mid PIP2, Plcg\) & 0.318 \\
        \midrule
        \(Raf \indep Akt, PIP2, PIP3, Plcg\) & 0.036 \\
        \(\qquad \qquad \mid PKA, PKC\) & \\
        \midrule
        \(P38 \indep Akt, PIP2, PIP3, Plcg, Raf\) & 0.680 \\
        \(\qquad \qquad \mid PKA, PKC\) & \\
        \midrule
        \(Jnk \indep Akt, P38, PIP2, PIP3, Plcg, Raf\) & 0.002 \\
        \(\qquad \qquad \mid PKA, PKC\) & \\
        \midrule
        \(Mek \indep Akt, Jnk, P38, PIP2, PIP3, PKA,\) & 0.544 \\
        \(\qquad \qquad PKC, Plcg \mid Raf\) & \\
        \midrule
        \(Erk \indep Akt, Jnk, P38, PIP2, PIP3, PKC,\) & 0.000 \\
        \(\qquad \qquad Plcg, Raf \mid Mek, PKA\) & \\
        \bottomrule
    \end{tabular}
\caption{
Summary of results on testing a ground-truth DAG against a protein-signaling dataset.
A kernel-based CI test was used to test the set of CIs (invoked by C-LMP) on the dataset.
P-value is rounded to the nearest three digits after the decimal point.
}
\label{table:model_testing}
\end{table}

\subsection{Analysis of C-LMP}
\label{subsection:experimentb}

In this section, we use \textsc{ListCI} to understand the total number of non-vacuous CIs invoked by C-LMP.
Let \(\*{CI}\) denote this number.
We showed that \(\*{CI}\) is bounded by \(\Theta(n2^s)\) for a DAG with \(n\) nodes whose largest c-component has size \(s\).
While we give a concrete DAG to show this bound is tight, our hope is to empirically evaluate how often this worst-case arises, and \(\*{CI}\) in the `average case' using random graphs.
\textsc{ListCI} makes such empirical analysis of C-LMP possible by giving a way to efficiently compute \(\*{CI}\).

\subsubsection{Hypothesis.} We hypothesize that the following two parameters are key in determining \(\*{CI}\) for a DAG \(\G\):

\begin{enumerate}
    \item \(mu\): the number of bidirected edges in \(\G\), and
    \item \(s\): the size of the largest c-component in \(\G\).
\end{enumerate}

Intuitively, both the size of c-components in \(\G\) and their sparsity, which depends on \(mu\), are important indicators of \(\*{CI}\). Both parameters control the number of subsets of the c-component that result in admissible ancestral c-components.

\begin{figure}[t]
    \centering
    \begin{subfigure}{0.47\textwidth}
        \includegraphics[width=\textwidth]{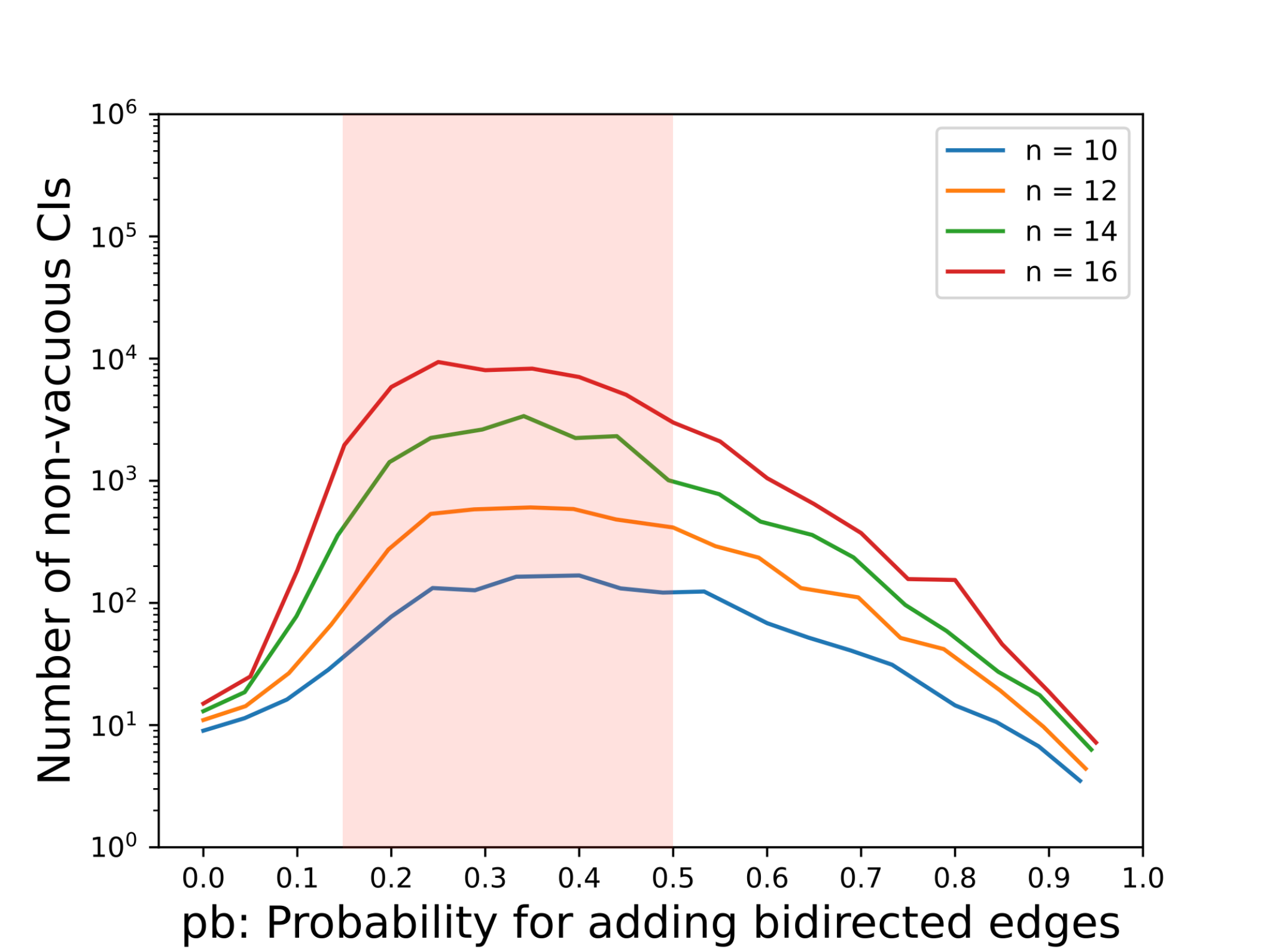}
        \caption{\(pb\) and \(\*{CI}\)}
        \label{fig:experiments:plot:1a:pbci}
    \end{subfigure}
    \hfill
    \begin{subfigure}{0.47\textwidth}
        \includegraphics[width=\textwidth]{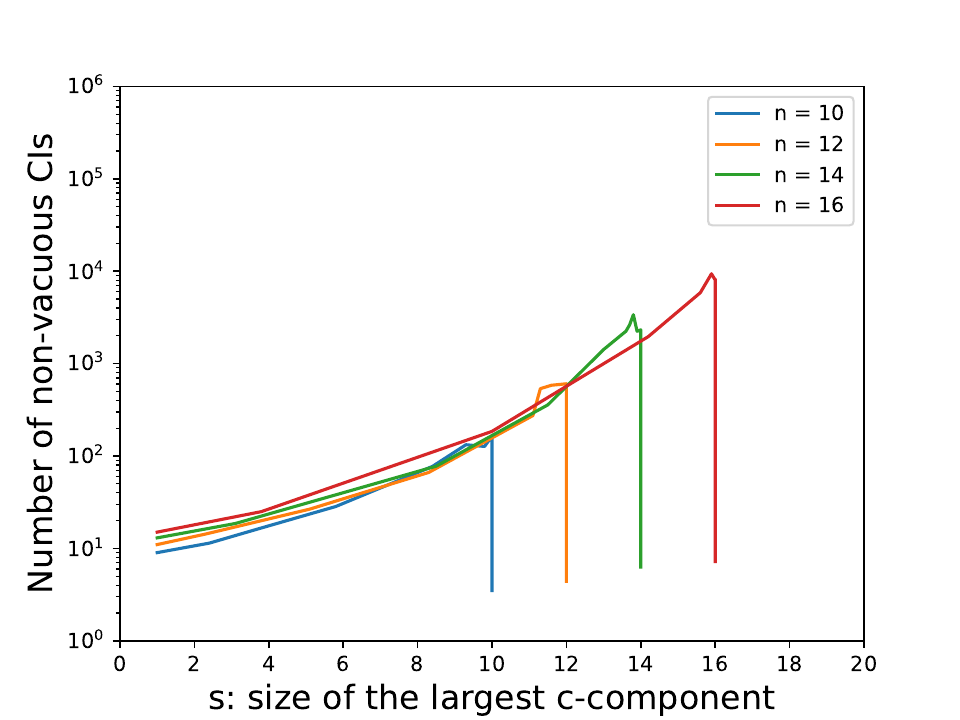}
        \caption{\(s\) and \(\*{CI}\)}
        \label{fig:experiments:plot:1a:sci}
    \end{subfigure}
    \caption{
    Illustration of results in Case 1 (\(md = 0\)).
    (a) displays two-phase transitions: as \(pb\) increases, \(\*{CI}\) grows rapidly up to a certain point (Phase 1) but shrinks after (Phase 2).
    (b) shows that \(s\) correlates with \(\*{CI}\) only within Phase 1.
    The red box indicates the `critical region.'
    }
    \label{fig:experiments:plot:1a}
\end{figure}

\begin{figure}[t]
    \centering
    \begin{subfigure}{0.47\textwidth}
        \includegraphics[width=\textwidth]{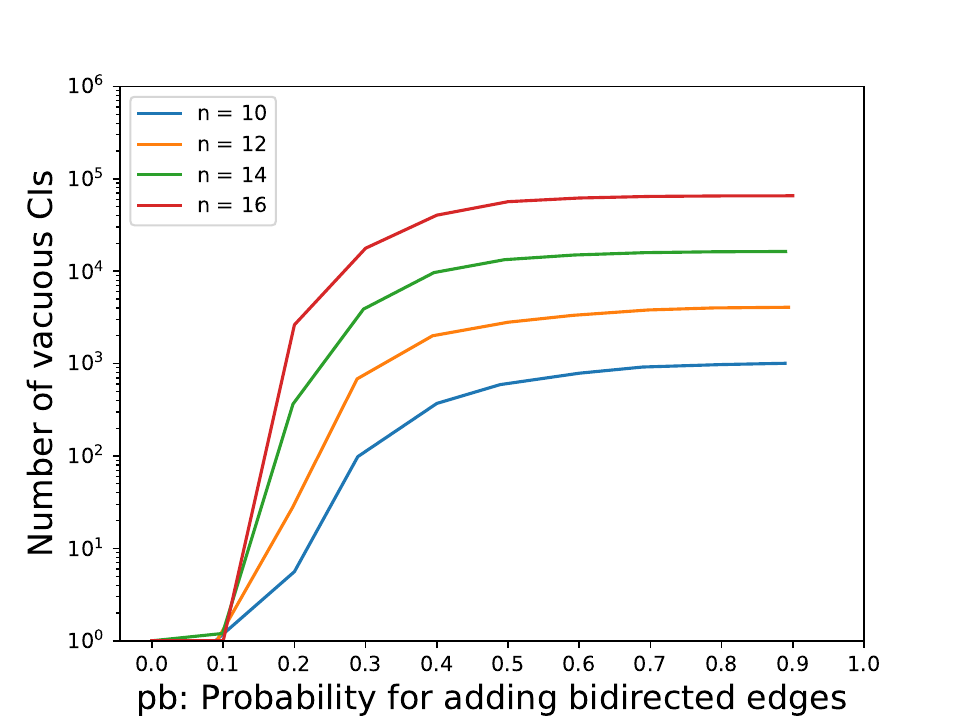}
        \caption{\(pb\) and number of vacuous CIs}
        \label{fig:experiments:plot:1v}
    \end{subfigure}
    \hfill
    \begin{subfigure}{0.47\textwidth}
        \includegraphics[width=\textwidth]{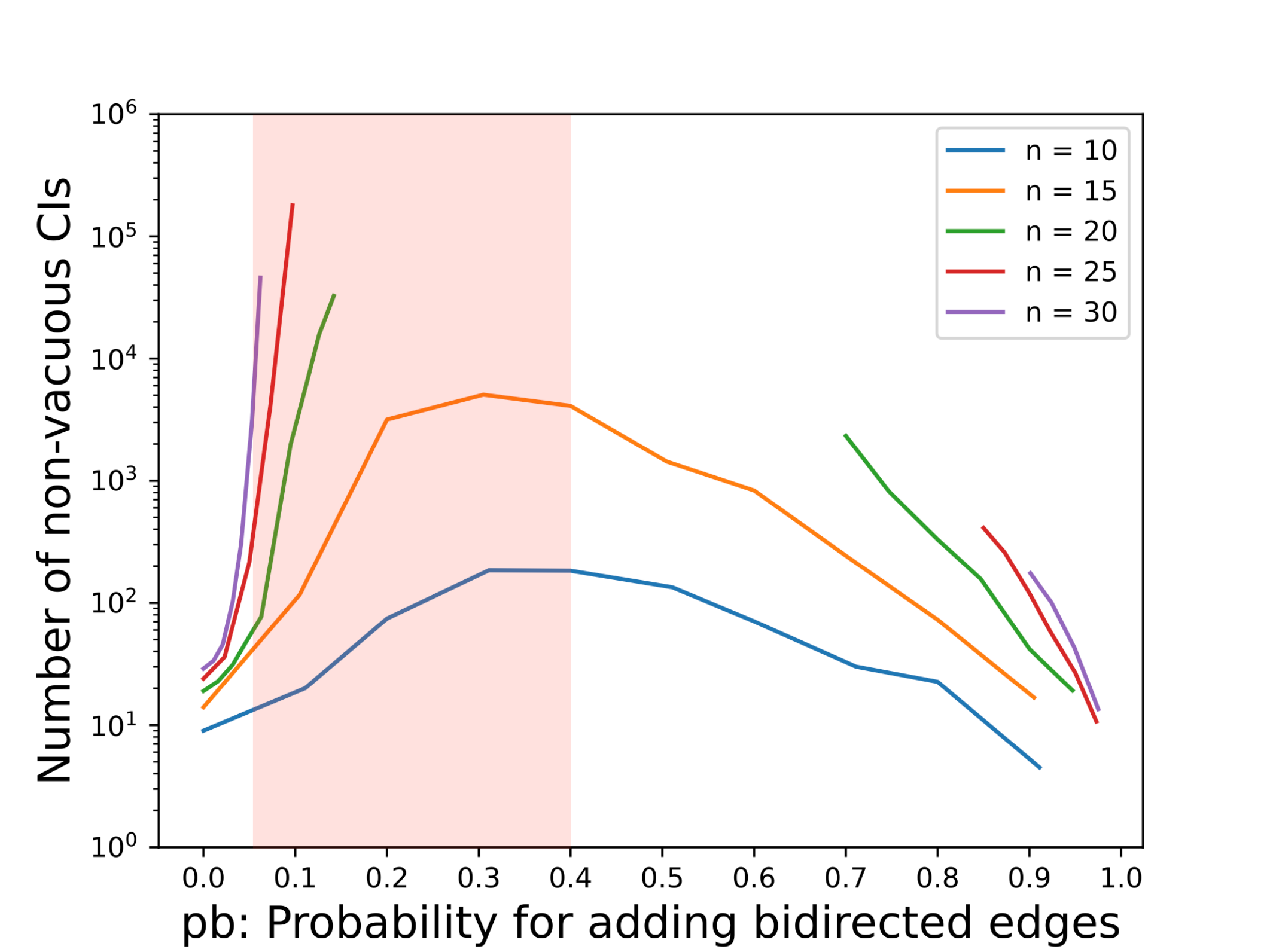}
        \caption{\(pb\) and \(\*{CI}\)}
        \label{fig:experiments:plot:1r}
    \end{subfigure}
    
    \caption{
    Illustration of results in Case 1.
    (a) Total number of vacuous CIs increases as \(\G\) becomes more dense with respect to bidirected edges.
    (b) Two-phase transitions are not fully observable for \(n \geq 20\).
    he red box indicates the `critical region.'
    }
    \label{fig:experiments:plot:1vr}
\end{figure}

\subsubsection{Random graphs.}
We run \textsc{ListCI} on random graphs to understand \(\*{CI}\) in the average case.
In particular, we use a minor variant of Erdős-Rényi random graphs  to include both directed and bidirected edges. 
We define a random causal DAG as \(\G(n,pd,pb)\) where \(pd\) (\(pb\)) represents the probability of a directed (bidirected) edge between a given pair of nodes.
Each possible edge is an independent Bernoulli.

\subsubsection{Experimental design.}  For each experiment, we fix our controls: \(n\), the number of nodes, and \(md\), the number of directed edges, in \(\G\).
We test \(md = 0,n,2n\) for select \(n \in [10,50]\).
Then, we change \(mu\), and observe \(s\) and \(\*{CI}\).
For each \(n\) and \(md\), we ran each experiment on 100 sample graphs.
Each run of \textsc{ListCI} was given an hour until timeout. 
Only complete runs are shown.

\begin{figure}[t]
    \centering
    \begin{subfigure}{0.47\textwidth}
        \includegraphics[width=\textwidth]{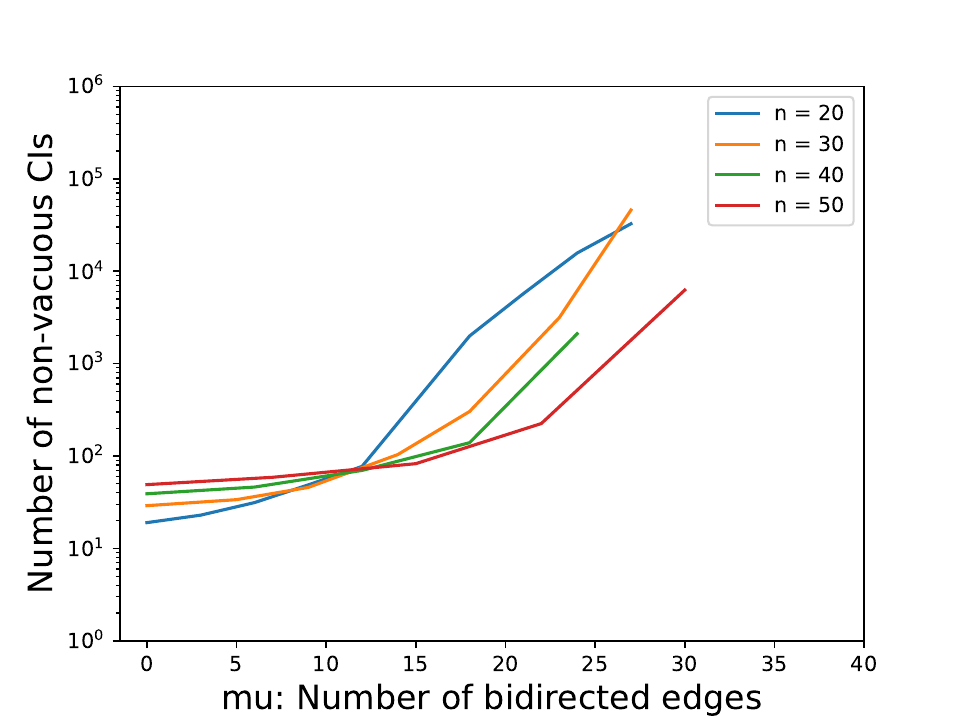}
        \caption{\(mu\) and \(\*{CI}\)}
        \label{fig:experiments:plot:1b:mbci}
    \end{subfigure}
    \hfill
    \begin{subfigure}{0.47\textwidth}
        \includegraphics[width=\textwidth]{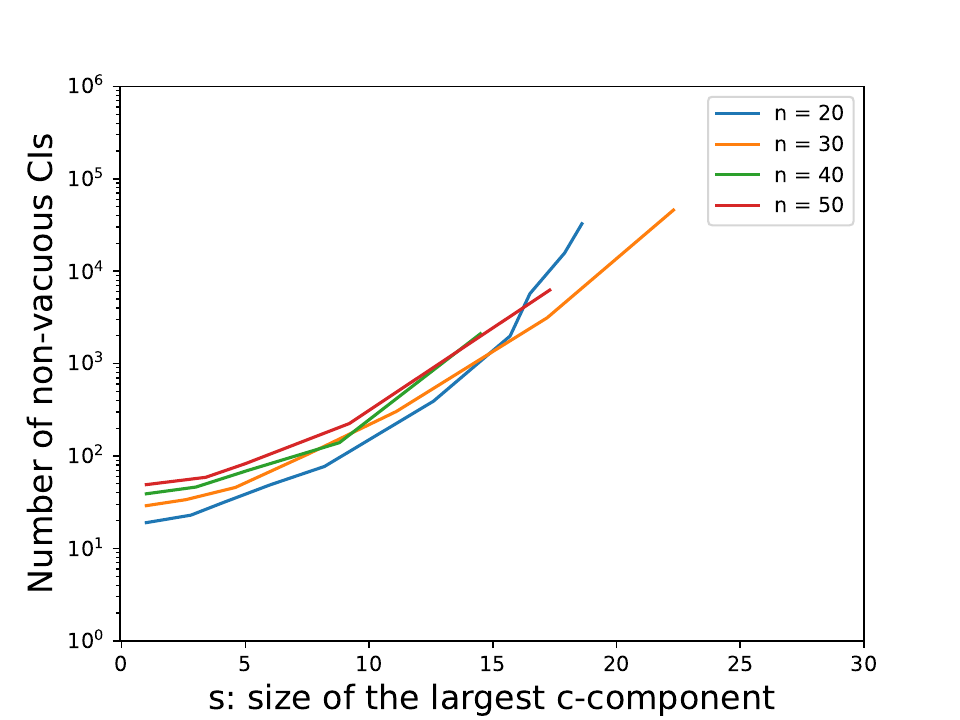}
        \caption{\(s\) and \(\*{CI}\)}
        \label{fig:experiments:plot:1b:sci}
    \end{subfigure}
    \caption{
    Illustration of results in Case 1 within Phase 1.
    \textsc{ListCI} starts timing out at approximately \(mu = 30\) or greater.
    }
    \label{fig:experiments:plot:1b}
\end{figure}

\subsubsection{Results and discussion.}

\begin{enumerate}
    \item \textbf{Case 1}: \(md = 0\). 
    First, for simplicity, we work with small graphs containing no directed edges.
    Starting from \(pb = 0\) (or \(mu = 0\)), we incrementally add bidirected edges until reaching full capacity, i.e., \(pb = 1\) or \(mu = \frac{n (n-1)}{2}\).
    Then, an interesting trend emerges as shown in Fig.~\ref{fig:experiments:plot:1a:pbci}.
    Roughly speaking, there are two phases seen on the curve.

    \begin{enumerate}
        \item Phase 1 on the left half of the curve.
        As more bidirected edges are added (i.e., \(pb\) increases), \(\*{CI}\) grows exponentially up to a certain peak region.
    
        \item Phase 2 on the right half of the curve.
        After reaching this peak, \(\*{CI}\) decreases exponentially as more bidirected edges are added.
    \end{enumerate}
    
    A possible explanation for the pattern shown in Phase 1 is that larger c-components tend to be constructed as \(pb\) increases.
    Then, \(s\) increases in general.
    As given by the bound \(O(n 2^s)\) (Prop~\ref{prop:lmpsize}), \(\*{CI}\) increases exponentially with a linear increase in \(s\).
    The curve in Fig~\ref{fig:experiments:plot:1a:sci} showing this relationship corresponds to Phase 1.
    Intuitively, a linear increase of the size of the largest c-component \(\*C\) (of size \(s\)) implies an exponential increase of total combination of subsets of \(\*C\) (i.e., MBs).
    As shown by Lemma~\ref{lemma:equivalence:cmb} and Thm.~\ref{thm:equivalence:clmpplus}, each MB maps uniquely to each CI invoked by C-LMP.
    Thus, the total number of MBs is the sum of the numbers of vacuous and non-vacuous CIs, which is represented by the ``sum'' of both curves: one curve in Fig.~\ref{fig:experiments:plot:1a:pbci} and the other curve in Fig.~\ref{fig:experiments:plot:1v}, respectively.

    On the other hand, in Phase 2, when even more bidirected edges are added to \(\G\), large c-components (or in fact, the largest and only c-component of size \(n\)) may become more dense in terms of bidirected edges.
    In the extreme case with \(pb = 1\), \(\G\) becomes a bidirected clique of size \(n\).
    With more bidirected paths between nodes, the number of \(d\)-separations in the graph decreases, leading to a decrease in \(\*{CI}\).
    Conceptually, the total number of MBs increases as c-components get more dense. 
    However, the ratio of MBs that result in non-vacuous CIs to total MBs decreases at a higher rate than the rate of increase of the number of MBs.
    We observe the difference by comparing Fig.~\ref{fig:experiments:plot:1a:pbci} and Fig.~\ref{fig:experiments:plot:1v}.
    This results in the decrease in \(\*{CI}\).

    \begin{figure}[t]
        \centering
        \begin{subfigure}{0.47\textwidth}
            \includegraphics[width=\textwidth]{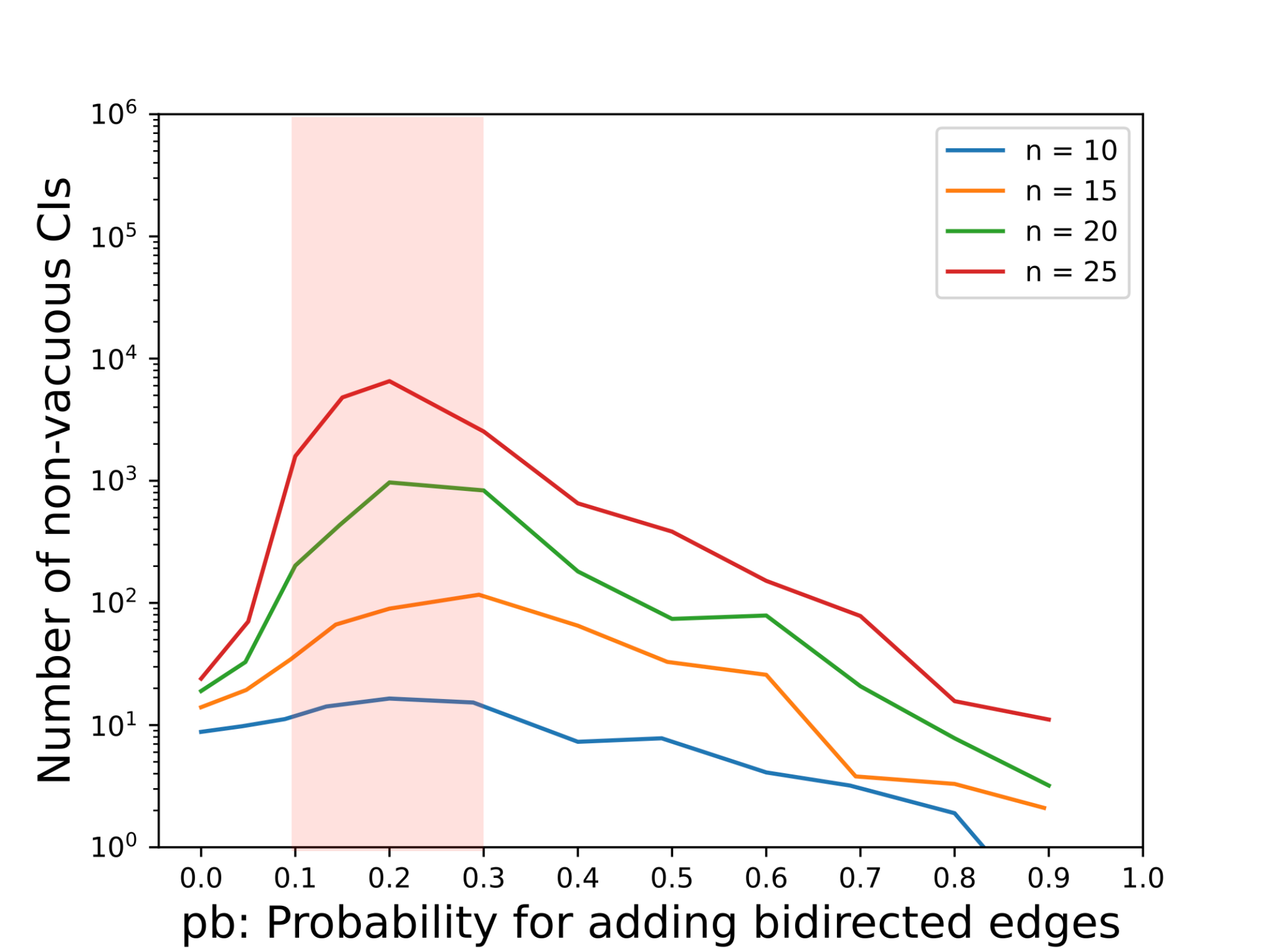}
            \caption{\(pb\) and \(\*{CI}\)}
            \label{fig:experiments:plot:2a:pbci}
        \end{subfigure}
        \hfill
        \begin{subfigure}{0.47\textwidth}
            \includegraphics[width=\textwidth]{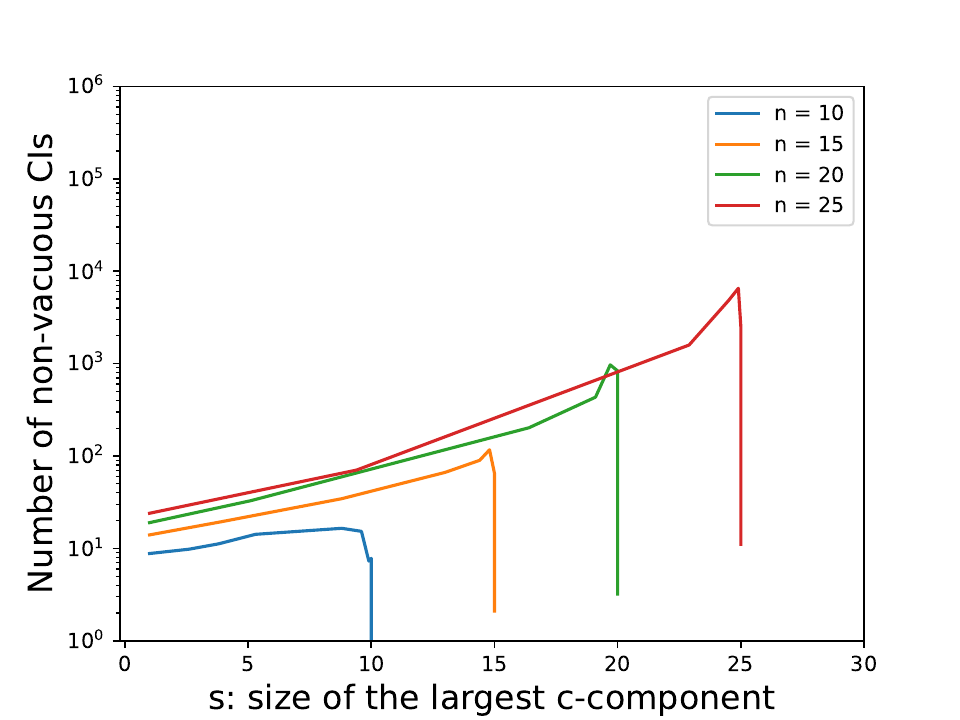}
            \caption{\(s\) and \(\*{CI}\)}
            \label{fig:experiments:plot:2a:sci}
        \end{subfigure}
        \caption{
        Illustration of results in Case 2 (\(md = n\)).
        Overall, similar patterns are shown as in Case 1 (Fig.~\ref{fig:experiments:plot:1a}).
        However, the rate of growth of \(\*{CI}\) with respect to \(n\) is lower than in Case 1.
        he red box indicates the `critical region.'
        }
        \label{fig:experiments:plot:2a}
    \end{figure}

    \begin{figure}[t]
        \centering
        \begin{subfigure}{0.47\textwidth}
            \includegraphics[width=\textwidth]{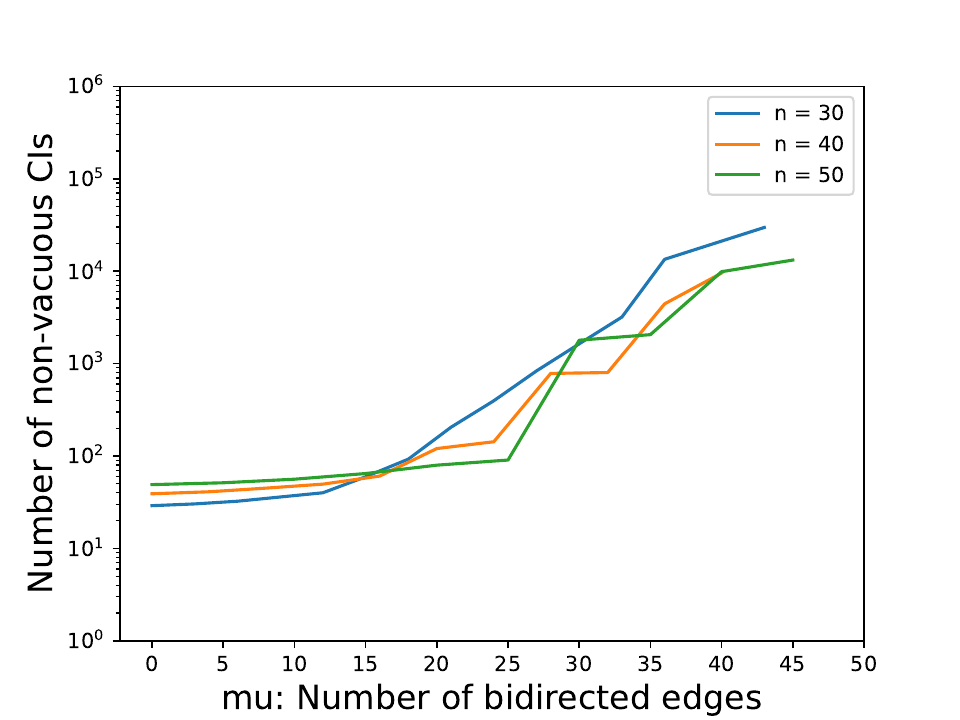}
            \caption{\(mu\) and \(\*{CI}\)}
            \label{fig:experiments:plot:2b:mbci}
        \end{subfigure}
        \hfill
        \begin{subfigure}{0.47\textwidth}
            \includegraphics[width=\textwidth]{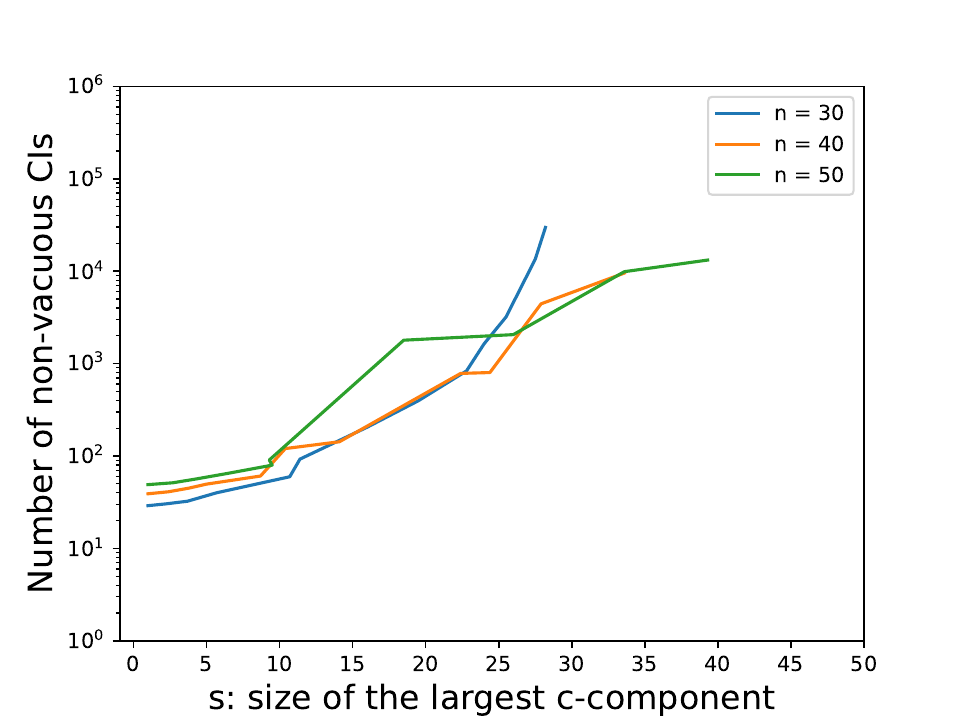}
            \caption{\(s\) and \(\*{CI}\)}
            \label{fig:experiments:plot:2b:sci}
        \end{subfigure}
        \caption{
        Illustration of results in Case 2 by adding bidirected edges to a graph \(\G\) across varying \(n\).
        \textsc{ListCI} starts timing out at approximately \(mu = 50\) or greater.
        }
        \label{fig:experiments:plot:2b}
    \end{figure}
    
    In Fig~\ref{fig:experiments:plot:1a:sci}, a vertical line where \(s\) stays constant, i.e., \(s = n\), corresponds to Phase 2.
    We note that \(s\) being a constant is a natural consequence of the experimental setup.
    When \(pb\) continues to increase from 0, all nodes in \(\G\) will eventually become connected to one another, and thus \(s\) converges to \(n\).
    Once the point with \(s = n\) is reached, \(s\) stays constant even with further addition of bidirected edges since the entire set of nodes in \(\G\) is the largest and the only c-component in \(\G\).
    When \(pb\) is further increased, the largest c-component in \(\G\)  becomes more dense, which explains the decrease in \(\*{CI}\).
    Therefore, \(s\) may be a good indicator of \(\*{CI}\) in Phase 1, but not necessarily in Phase 2.

    \begin{figure}[t]
        \centering
        \begin{subfigure}{0.47\textwidth}
            \includegraphics[width=\textwidth]{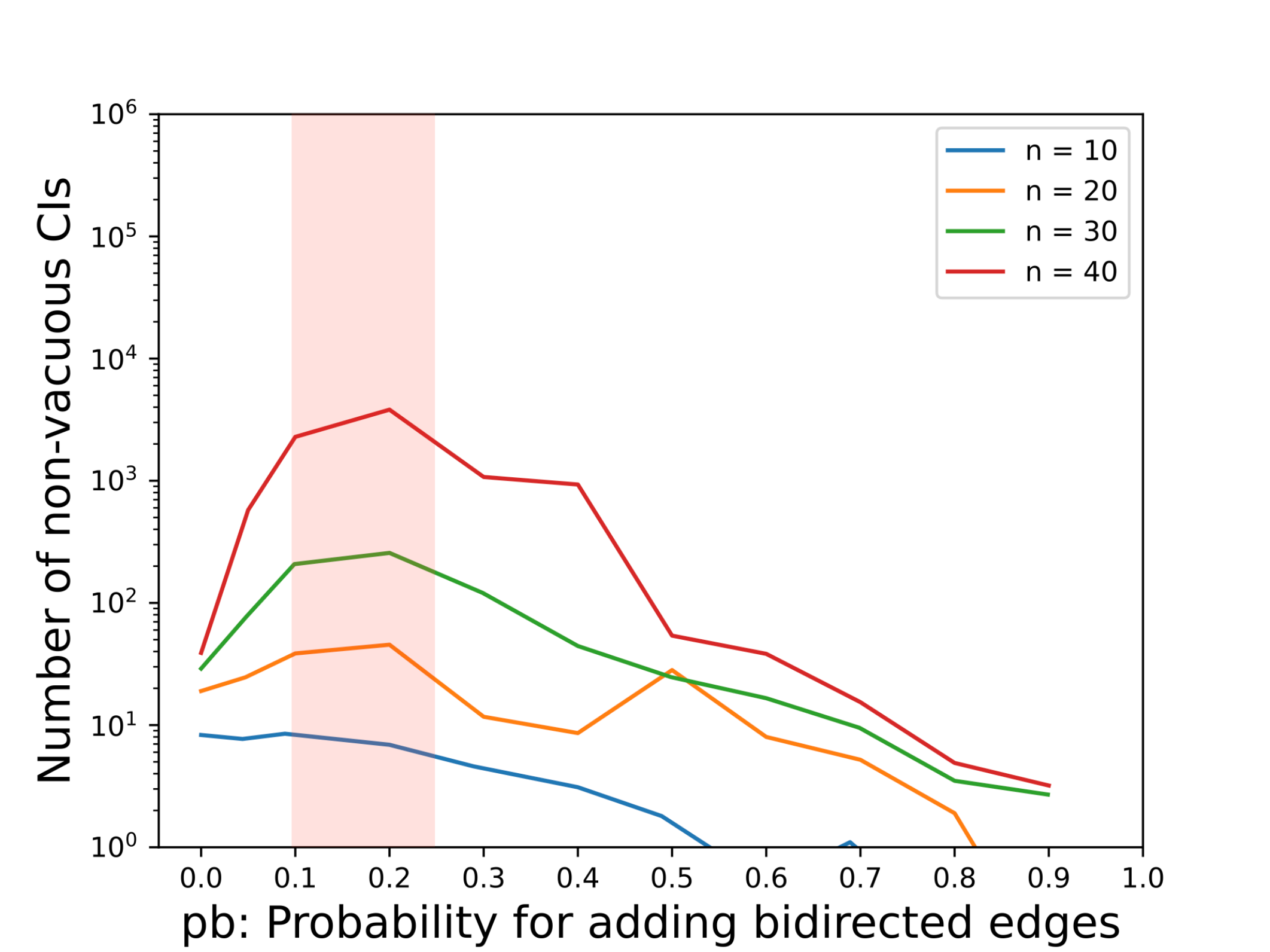}
            \caption{\(pb\) and \(\*{CI}\)}
            \label{fig:experiments:plot:3a:pbci}
        \end{subfigure}
        \hfill
        \begin{subfigure}{0.47\textwidth}
            \includegraphics[width=\textwidth]{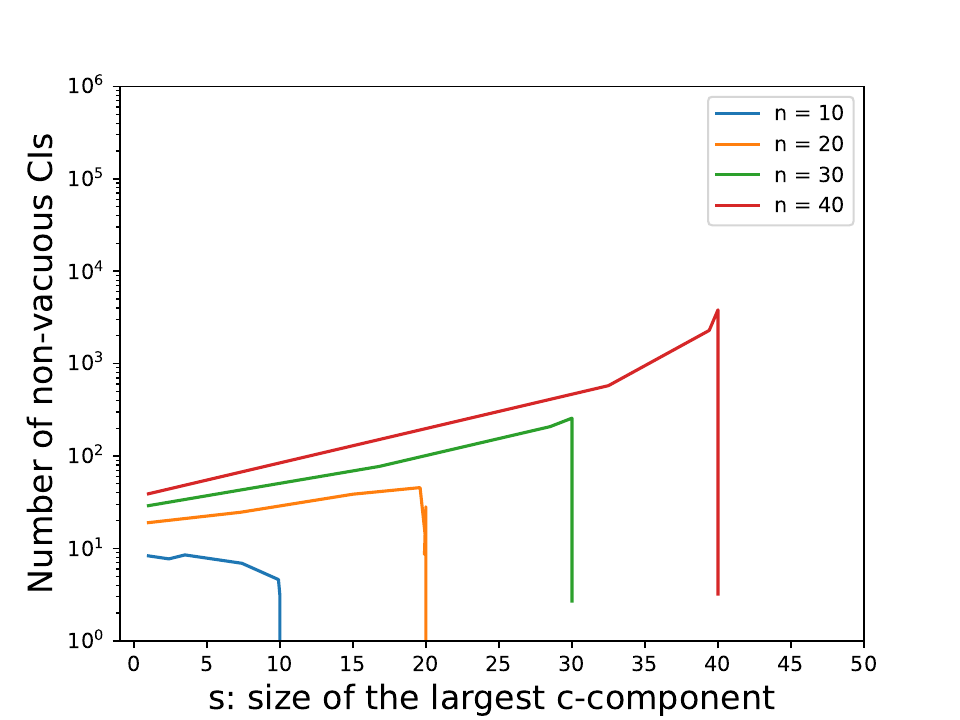}
            \caption{\(s\) and \(\*{CI}\)}
            \label{fig:experiments:plot:3a:sci}
        \end{subfigure}
        \caption{
        Illustration of results in Case 3 (\(md = 2n\)).
        Overall, the patters are similar as shown in Case 1 and Case 2 (Fig.~\ref{fig:experiments:plot:1a} and Fig.~\ref{fig:experiments:plot:2a}).
        The rate of growth of \(\*{CI}\) with respect to \(n\) is lower than those in Case 1 and Case 2.
        he red box indicates the `critical region.'
        }
        \label{fig:experiments:plot:3a}
    \end{figure}

    Another subtlety to note is the rate of growth of \(\*{CI}\) with respect to \(n\).
    Within a ``critical region'' shown in Fig~\ref{fig:experiments:plot:1a:pbci}, observe gaps between the curves for each \(n\).
    Even as \(n\) increases by two from \(n = 10\), \(\*{CI}\) in this middle region grows exponentially.
    This may not be immediate as the bound on \(\*{CI}\) is linear in  \(n\).
    However, in the peak region of the curve, we have \(s \approx n\), thus making \(\*{CI}\) exponential in \(n\).
    This makes it infeasible to observe phase transitions over larger \(n\) in graphs without directed edges.

    We verify the claim that phase transitions may not be fully observable for larger \(n\).
    Bidirected edges are added slowly until \textsc{ListCI} starts timing out.
    The results are shown in Fig.~\ref{fig:experiments:plot:1b:mbci}.
    We see that \(mu = 30\) is an approximate threshold after which \textsc{ListCI} may spend more than an hour.
    Given \(mu = 30\), the threshold values of \(pb\) that correspond to each \(n \in \{20,30,40,50\}\) are \(0.158, 0.069, 0.038\), and \(0.024\) respectively.
    Based on the curves shown in Fig.~\ref{fig:experiments:plot:1a:pbci}, it is possible \textsc{ListCI} times out before the peak.
    All curves for large \(n\) live within Phase 1.

    Additionally, we present Phase 2 for \(n \in \{10,15,20,25,30\}\) in Fig.~\ref{fig:experiments:plot:1r}.
    Starting from \(pb = 1\), we keep removing bidirected edges (i.e., decreasing \(pb\)) until \textsc{ListCI} starts timing out.
    For \(n = 20\), \textsc{ListCI} times out with \(pb < 0.7\).
    For \(n = 30\), \(pb < 0.85\) and for \(n = 40\), \(pb < 0.9\).
    All fraction of the curves represent some fraction of Phase 2.
   
    Returning to \(s\), we show the relationship between \(s\) and \(\*{CI}\) in Fig~\ref{fig:experiments:plot:1b:sci}.
    As in the case for small \(n\), \(\*{CI}\) is exponential in \(s\) during Phase 1.

    It may seem that \textsc{ListCI} is not feasible on larger graphs. However, Case 1 considers an edge case with no directed edges where all subsets of nodes are ancestral. The problem is highly unconstrained. Most real-world graphs are not this sparse, which makes \(\*{CI}\) less sensitive to changes in \(mu\), as we explain in the next part.

    \item \textbf{Case 2}: \(md = n\).
    We use a similar setup as in Case 1, except that we add \(n\) directed edges on a wider range of graph sizes.
    When we incrementally add bidirected edges to \(\G\) from \(pb = 0\) up to \(pb = 1\), a pattern identical to Case 1 (Fig.~\ref{fig:experiments:plot:1a:pbci}) arises in Case 2 (Fig.~\ref{fig:experiments:plot:2a:pbci}).
    A notable difference, however, is the rate of growth of \(\*{CI}\) with respect to \(n\).
    For example, let \(n = 20\).
    In Case 1 (Fig.~\ref{fig:experiments:plot:1b:mbci}), with increasing \(mu > 20\), \(\*{CI}\) increases to \(10^4\) and beyond until \textsc{ListCI} times out.
    On the other hand, in Case 2 (Fig.~\ref{fig:experiments:plot:2a:pbci}), \(\*{CI}\) does not exceed \(10^3\) for any \(mu\) (or \(pb\)).
    Still, the two-phase transition is not observable for larger \(n\), i.e., \(n \in \{30,40,50\}\).
    Similarly, as in Fig.~\ref{fig:experiments:plot:1a:sci}, we observe an exponential relationship between \(s\) and \(\*{CI}\) (Fig.~\ref{fig:experiments:plot:2a:sci}) in for \(s < n\).

    Next, we let \(n \in \{30,40,50\}\). %
    The results are shown in Fig.~\ref{fig:experiments:plot:2b}.
    We have a similar conclusion as in Case 1, except that an approximate threshold for \(mu\) until \textsc{ListCI} times out is increased to 50.
    Given \(mu = 50\), the values of \(pb\) that map to each \(n \in \{30,40,50\}\) are \(0.115, 0.064\), and \(0.041\) respectively.
    The larger threshold can be explained by the correspondence between ancestral sets and MBs.
    Since adding directed edges exponentially reduces the number of ancestral sets, this can only reduce the number of MBs, and hence \(\*{CI}\).
    Inspecting the curves shown in Fig.~\ref{fig:experiments:plot:2a:pbci}, it is likely that \textsc{ListCI} starts timing out before \(\*{CI}\) peaks.

    \item \textbf{Case 3}: \(md = 2n\).

    We continue the set up of Cases 1 and 2, now adding \(2n\) directed edges to small-to-medium sized graphs.
    As shown in Fig.~\ref{fig:experiments:plot:3a:pbci}, we see phase transitions for \(n\) up to \(40\).
    Comparing to Case 2 where \(md = n\), the rate of growth of \(\*{CI}\) with \(mu\) is lower in general.
    For example, let \(n = 20\).
    In Case 2 (Fig.~\ref{fig:experiments:plot:2a:pbci}), \(\*{CI}\) reaches approximately  \(10^3\).
    However, in Case 3 (Fig.~\ref{fig:experiments:plot:3a:pbci}), \(\*{CI}\) does not reach \(10^2\), even in the peak.
    The relationship between \(s\) and \(\*{CI}\) seen in Cases 1 and 2 (Figs.~\ref{fig:experiments:plot:1a:sci} and ~\ref{fig:experiments:plot:2a:sci}) – with two patterns corresponding to the two phases – is reproduced in Case 3 (Fig.~\ref{fig:experiments:plot:3a:sci}).
\end{enumerate}

Summarizing experimental findings from Cases 1, 2, and 3, we conclude that both the size \(s\) of the largest c-component \(\*C\) in \(\G\) and the sparsity of \(\*C\) determined by the number of bidirected edges play a key role in \(\*{CI}\). The reproducibility of the phase transitions and relationships between \(s, mu\), and \(\*{CI}\) across different combinations of \(md\) and \(n\) lends credence to this conclusion.

\section{Frequently Asked Questions}

\renewcommand{\labelenumi}{Q\arabic{enumi}.}

\begin{enumerate}

    \item Is it reasonable to expect that the causal graph is available? How do you get the graph?

    \textbf{Answer.}
    The assumption of the causal diagram is made out of necessity; without causal assumptions, causal inferences are almost never possible (e.g., see the Causal Hierarchy Theorem in \citep[Section~1.3]{bareinboim:etal20}).
    
    In the real world, data scientists engage in causal modeling and leverage their background knowledge about the problem to construct a causal model (e.g., graph). 
    Celebrated results in the literature, such as Pearl's do-calculus, were designed to take advantage of this knowledge in order to generate quantitative  understanding of the system that was previously unknown to the data scientist. 
    Part of the main theme in the field is about how to infer new facts given a collection of causal assumptions. 

    Against this context, the main goal of our work is to provide a set of tools to evaluate whether the assumptions encoded in a causal model are plausible, or formally compatible with the observed data. 
    It is not easy to characterize or to list all of such assumptions, as discussed formally in Section~\ref{sec:reformulation} and empirically in  Appendix~\ref{appendix:experiments}. We provide the first algorithm for listing a small set of CI assumptions in poly-delay, using which a model can be tested in settings with non-parametric distributions and arbitrary unobserved variables. 

    Finally, the task known as causal discovery aims to a coarser representation of the causal model from data, including from observational \cite{verma:pea92,Spirtes2001,pearl:2k} and interventional data \cite{kocaoglu:etal17,kocaoglu2019characterization,jaber2020cd,li2023causal}.

    \item Can this result be used to evaluate the quality of a learned model, e.g., a partial ancestral graph (PAG) \cite{zhang:08}?
    \textbf{Answer.} Yes. If the learned model is a Markov equivalence class  (MEC) of DAGs, for e.g., a PAG, all DAGs in the MEC imply exactly the same set of CIs. Therefore, an observational dataset is consistent with the MEC if and only if it is consistent with some (or every) DAG in the MEC. To test the learned MEC, one can choose any DAG in the MEC, and apply our result to this DAG.

    \item What's the difference between (LMP,\(\prec\)) and C-LMP?
    Since they output an identical list of CIs, aren't they the same?

    \textbf{Answer.} 
    It is true that (LMP,\(\prec\)) and C-LMP invoke the same set of CIs.
    Since \textsc{ListCI} lists CIs invoked by C-LMP, it thus equivalently lists CIs invoked by (LMP,\(\prec\)). 
    There is nothing inherent in the definition of (LMP,\(\prec\)) (Def.~\ref{def:lmp}) that makes it impossible to list the CIs it invokes in poly-delay.
    However, in Def.~\ref{def:maxanc} and Def.~\ref{def:lmp}, MASs are defined non-constructively.
    Def.~\ref{def:maxanc} leaves it open whether there is exactly one MAS relative to an MB, and how to construct such an MAS.
    Therefore, each CI in (LMP,\(\prec\)) is also defined non-constructively.
    The only object with a constructive definition is the ancestral set, which is used to define MASs using universal quantifiers.
    This considerable degree of indeterminacy leads to the brute-force approach we develop in Section~\ref{subsection:lmpbruteforce}.
    In contrast, the definition of C-LMP (Def.~\ref{def:lmpplus}) is entirely constructive, and abstracts away the complexities of MASs and MBs. 
    We give an explicit one-to-one mapping between ACs and CIs (Thm.~\ref{thm:equivalence:clmpplus}) that does not need any universal quantifiers, except over the space of ACs.
    Therefore, the definition of C-LMP provides a natural path to enumerating the invoked CIs by enumerating ACs.
    Moreover, the explicit one-to-one mapping between ACs and CIs allows us to derive tight bounds on the number of CIs invoked by C-LMP (and equivalently, (LMP,\(\prec\))) by reasoning about connected components in the graph, an approach that would not be clear from a non-constructive definition of CIs.
    
    \item What happens if the total number of CIs invoked by C-LMP is exponential?
    Do we have to wait until \textsc{ListCI} outputs the full list of CIs?

    \textbf{Answer.}
    First, we note that C-LMP is an exponential improvement on the global Markov property (GMP) with respect to number of CIs invoked.
    In contrast with the \(\Theta(4^n)\) many CIs invoked by GMP (Prop.~\ref{prop:gmpsize}), C-LMP invokes \(O(n 2^s)\) number of CIs given a DAG on \(n\) variables whose largest c-component has size \(s\) (Prop.~\ref{prop:lmpsize}).
    The upshot is largest for a DAG with large \(n\) but small c-components.
    For instance, for the DAG \(\G^2\) in Fig.~\ref{fig:intro_local_semimarkov}, GMP invokes 753 CIs but C-LMP invokes only 5. For the real-world protein-signaling network in Fig.~\ref{fig:ground_truth_graph}, GMP invokes 76580 CIs but C-LMP invokes only 10.

    Even when C-LMP invokes exponentially many CIs, \textsc{ListCI} outputs all such CIs in poly-delay (Thm.~\ref{thm:listci}).
    This is the first known algorithm that runs in poly-delay where the associated Markov property is applicable to arbitrary data distributions and DAGs with latent variables.
    The poly-delay property allows researchers to test the subset of CIs listed in the available time, which enables partial testing of the model.
    This is not possible with an algorithm that takes exponential time to output one CI, or even all CIs at once.
    Please refer to Appendix~\ref{sec:related_work} for more details on related literature.

    \item How well does \textsc{ListCI} scale?
    
    \textbf{Answer.}
    \textsc{ListCI} scales well and is currently the most efficient algorithm that enumerates all CIs invoked by a Markov property which is applicable to arbitrary data distributions and DAGs with latent variables.
    The plot in Fig.~\ref{fig:experiments:results} shows that \textsc{ListCI} takes more than an hour over some graphs with \(n >= 70\) nodes.
    Here, we note that \(n\) is not the only factor in the running time of \textsc{ListCI}.
    In fact, as shown in Appendix~\ref{subsection:experimentb}, the graph topology associated with c-components plays a major role in the number of CIs invoked by C-LMP.
    Two factors related to c-components are of major interest:

    \begin{enumerate}
        \item \(s \leq n\): the size of the largest c-component, and
        \item Sparsity of c-components with respect to the number of bidirected edges
    \end{enumerate}

    Let \(\*{CI}\) be the total number of non-vacuous CIs invoked by C-LMP.
    In summary, when c-components are sparse, \(\*{CI}\) increases exponentially in term \(s\), given by the bound \(O(n 2^s)\).
    However, as c-components become denser, \(\*{CI}\) decays in exponential term.
    For illustration, please refer to the discussion on Case 1 in Appendix~\ref{subsection:experimentb} (Fig.~\ref{fig:experiments:plot:1a}).

    Graph topology may vary across different graphs with different sizes.
    For example, large graphs can have many, small c-components.
    In this case, \(s\) may be small.
    Then, an exponent \(s\) in the bound \(O(n 2^s)\) is small, and thus total number of CIs invoked by C-LMP may not be large.
    Even when large graphs have large c-components, if such c-components are dense, then total number of CIs invoked by C-LMP could be smaller in an order of magnitude, as oppose to the case where the c-components are sparse.

    Next, there may exist exponentially many CIs invoked by C-LMP (with respect to \(n\)), requiring exponential time to list them all.
    In such cases, one guarantee we can provide is the poly-delay property, which holds for \textsc{ListCI} (Thm.~\ref{thm:listci}).
\end{enumerate}

\end{document}